\documentclass[10pt]{article}
\usepackage{slca_style}
\pdfmapline{+optimistic < assets/Optimistic.ttf <T1-WGL4.enc}
\DeclareFontFamily{T1}{optimistic}{}
\DeclareFontShape{T1}{optimistic}{m}{n}{<-> s * [0.88] assets/optimistic}{}
\DeclareFontShape{T1}{optimistic}{b}{n}{<-> s * [0.88] assets/optimistic}{}

\DisableLigatures{family=optimistic}

\theoremstyle{plain}
\newtheorem{theorem}{Theorem}[section]
\newtheorem{proposition}[theorem]{Proposition}

\theoremstyle{definition}

\newtheorem{assumption}[theorem]{Assumption}
\theoremstyle{remark}

\lstdefinestyle{promptstyle}{basicstyle=\ttfamily\small,backgroundcolor=\color{backcolour},breaklines=true,breakatwhitespace=true,numbers=none,numbersep=5pt,numberstyle=\tiny\color{codegray},frame=tb,rulecolor=\color{black},keepspaces=true,columns=fullflexible,showstringspaces=false,moredelim=[s][\bfseries]{**}{**},moredelim=[l][\bfseries]{\#\#}}
\definecolor{codegray}{rgb}{0.5,0.5,0.5}
\definecolor{backcolour}{rgb}{0.98,0.98,0.98}

\begin{document}
\sloppy
\makearxivtitle
\begingroup
\renewcommand{\thefootnote}{}
\footnotetext{\fontsize{8}{9.5}\selectfont\textsuperscript{*}Equal contribution. Work done during Yan Zhan and Shaobo Liu's internships at Tencent.\enspace\textsuperscript{+}Correspondence to: Weizhou Pan (\texttt{josephpan@tencent.com}).}
\endgroup

\section{Introduction}
\label{sec:intro}

The evolution of Large Language Models (LLMs) from passive chatbots to active agents hinges on their ability to plan and execute actions via external tools \citep{qin2024toolllm, schick2023toolformer, yao2023react, patil2024gorilla}. Although Supervised Fine-Tuning (SFT) establishes a strong initialization, it depends on imitating fixed patterns and is fragile against the large, changing namespaces of ecosystems such as the Model Context Protocol (MCP) \citep{hou2025mcp}. To transcend mimicry, post-training via on-policy Reinforcement Learning (RL) is essential \citep{ouyang2022instructgpt,schulman2017ppo,yu2025dapo}. Tool-calling resembles reasoning tasks in that trajectories can contain intermediate work before a final answer, but it adds externally verifiable boundaries: structured API calls, environment-injected observations, and schema constraints. These boundaries naturally decompose a rollout into $y = \big[ \toolsub{y} \oplus \summarysub{y} \big]$, where $\toolsub{y}$ encodes the operational trajectory (reasoning traces interleaved with structured tool calls), and $\summarysub{y}$ constitutes the user-facing natural language response.

Applying standard RL algorithms to this two-segment structure reveals a fundamental mismatch. Dominant approaches such as Group Relative Policy Optimization (GRPO) \citep{guo2025deepseekr1} broadcast a unified trajectory-level advantage to all tokens; even Process Reward Models (PRMs) aggregate signals ($R_{\text{total}} = \toolsub{R} + \summarysub{R}$) before advantage estimation. We argue that this \textbf{Global Signal Conflation} is a structural failure mode: summary-reward variation can enter the tool-token advantage, producing \textbf{Cross-Segment Credit Misattribution}. In sign-conflict regimes, a failed tool call can be reinforced when followed by a correct summary, while a correct tool trajectory can be penalized when followed by an incorrect summary. Existing mitigations do not close this pathway: temporal methods (VinePPO~\citep{kazemnejad2025vineppo}, GiGPO~\citep{feng2025gigpo}, SPO~\citep{guo2025spo}) address step-wise credit while retaining $R_{\text{total}}$; ToolPO~\citep{deepagent2025} adds local tool rewards but still lets summary-dependent noise reach tool tokens; and RLTR~\citep{li2025rltr} separates planner and summarizer, abandoning a unified backbone.

We verify this failure through diagnostic cases: standard GRPO assigns a single aggregated advantage to both tool and summary tokens, so a correct summary can mask an inefficient or wrong tool trajectory, and an incorrect summary can suppress otherwise correct tool use. Gradient diagnostics show instability spikes while tool/summary gradient directions remain near-orthogonal, consistent with advantage contamination as a source of the observed instability (App.~\ref{app:gradient_analysis}). Two questions therefore arise: \emph{(1)} can cross-segment advantage contamination be structurally eliminated within a single unified policy, and \emph{(2)} does eliminating it translate into measurable gains beyond what additive augmentation delivers?

\begin{figure}[H]
    \vspace{-0.4em}
    \centering
    \includegraphics[width=\linewidth]{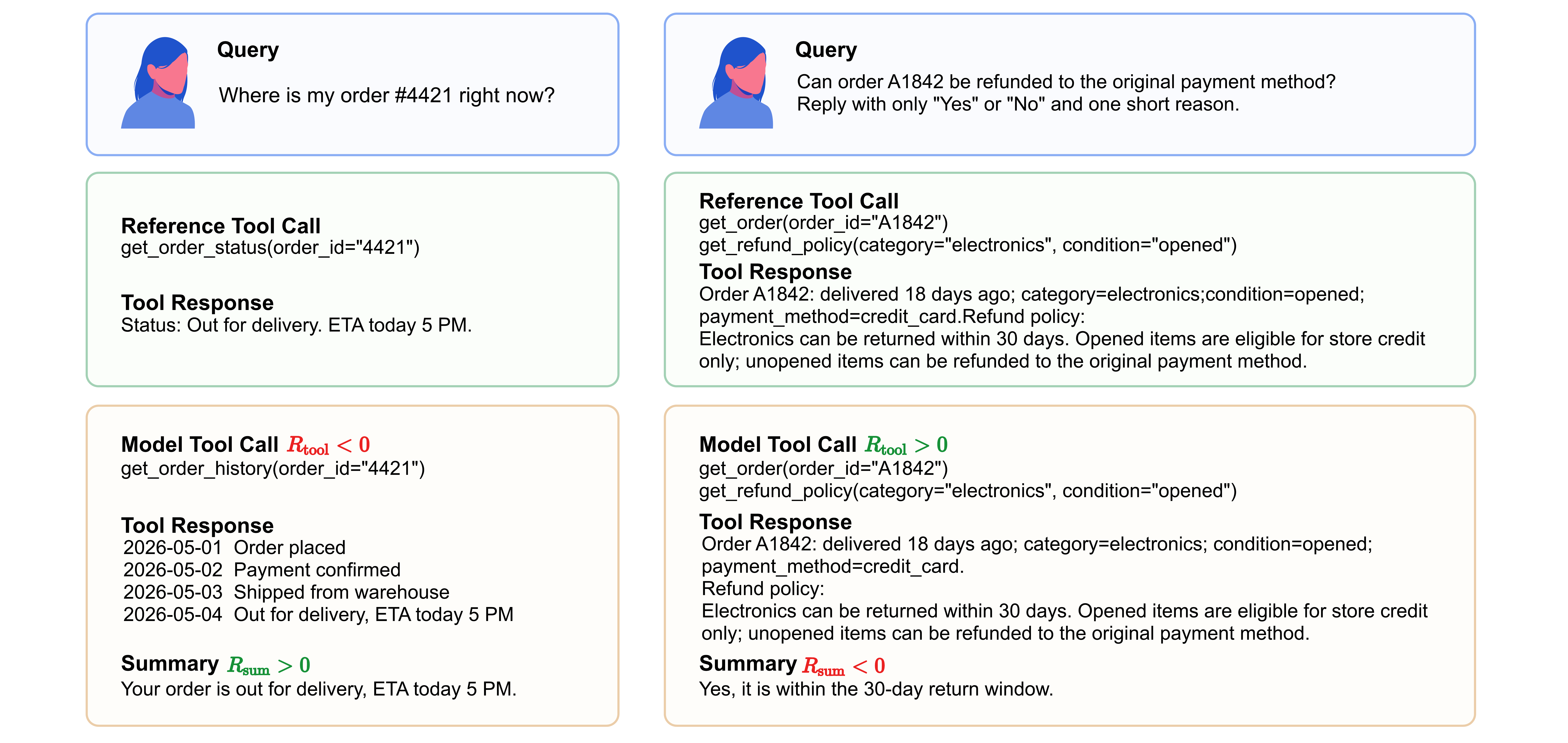}
    \caption*{\footnotesize \textbf{Diagnostic cases.} Unified reward aggregation can misassign credit across tool and summary segments. Left: an unnecessary tool call is paired with a correct delivery summary. Right: correct tool evidence is paired with an incorrect refund summary. Under standard GRPO, the mixed advantage can reward bad tool use or penalize good tool use.}
    \vspace{-1.0em}
\end{figure}

To resolve this, we propose \textbf{\slcaname} (\Cref{fig:framework}). Its core, Segment-Locked Credit Assignment (SLCA), routes tool-side advantages ($\hat{A}_{\toolsegment}$) exclusively to $\toolsub{y}$ tokens and summary-side advantages ($\hat{A}_{\summarysegment}$) exclusively to $\summarysub{y}$ tokens, blocking the defined summary-to-tool advantage path \emph{before} the backward pass, at zero additional rollout cost and within a single unified policy (\S\ref{subsec:slca}). This routing is a deliberate bias--variance trade-off: SLCA does not claim that tool calls are causally irrelevant to final answers; rather, it assumes that a dense execution reward is the lower-variance and sufficiently informative signal for updating tool-decision tokens, while summary rewards train the articulation segment. Crucially, SLCA operates on a \emph{structural} axis (execution vs.\ articulation) that is orthogonal to the \emph{temporal} axis addressed by VinePPO / SPO / GiGPO: it therefore \emph{composes with} rather than replaces these methods, and can be combined with temporal credit-assignment within each segment. A Schema-Guided LLM Simulator (SGLS, \S\ref{subsec:sgls}) and Hierarchical Rewards (HierR, \S\ref{subsec:hierr}) supply the infrastructure that makes this routing practical at scale.

Across three backbones (Qwen2.5-3B/7B-Instruct, Qwen3-8B-Base) and three benchmarks, \slcaname has higher reported means than matched GRPO on the main comparisons. On the 7B backbone, the matched mean gaps are +2.53\,pp on in-domain Toucan, +1.36\,pp on BFCL, and +9.15\,pp on $\tau^2$-Bench. The corresponding Toucan gaps are +2.35\,pp (3B) and +2.05\,pp (8B); complete multi-run results and method-specific baseline protocols are reported in the experiments and appendix. Our contributions are threefold:
\begin{itemize}[leftmargin=*,itemsep=0pt,topsep=2pt]
    \item We identify \textbf{Cross-Segment Credit Misattribution} as a structural failure mode of on-policy RL for tool-calling agents, arising from advantage contamination rather than gradient-direction conflict.
    \item We propose \textbf{\slcaname}, an intra-trajectory estimator that separately normalizes and routes segment advantages without additional rollouts, complemented by the SGLS simulator and HierR rewards.
    \item We evaluate \slcaname across three scales and three benchmarks, with ablations (w/o SLCA / SGLS / HierR) and a reward-protocol sensitivity experiment.
\end{itemize}

\section{Related Work}
\label{sec:related}

\subsection{Tool-Calling Post-Training}
Tool-calling agents are commonly initialized with SFT on annotated tool-use trajectories \citep{schick2023toolformer,qin2024toolllm}, often followed by preference- or verification-driven optimization \citep{ouyang2022instructgpt,christiano2017preferences,stiennon2020summarize,rafailov2023dpo}. However, SFT can be brittle under rule variants and OOD shifts due to memorization and exposure bias \citep{chu2025sftmemorizes,bengio2015scheduled,ross2011dagger,wei2025webagent,lightman2024letsverify,guo2025deepseekr1}. As tool namespaces scale, executable benchmarks such as ToolBench, API-Bank, and BFCL make post-training robustness increasingly central \citep{qin2024toolllm,li2023apibank,pmlr-v267-patil25a}. We target this post-training setting, where tool trajectories and final summaries form heterogeneous segments but standard end-to-end objectives still broadcast a single advantage to all tokens.
Scalable tool-agent training also relies on simulated or emulated tool environments, since live API interaction can be costly, unstable, or unavailable at RL scale \citep{guo-etal-2024-stabletoolbench,ruan2024toolemu,deepagent2025}.
Our SGLS follows this line but uses schemas as a control plane so that the simulator supplies segment-specific feedback for credit-assignment analysis.

\subsection{Credit Assignment for Agentic RL}
On-policy RL for tool use inherits credit-assignment issues from sparse trajectory-level optimization \citep{williams1992reinforce,schulman2015trpo,schulman2017ppo,guo2025deepseekr1}. Process supervision, PRMs, and reward-centric tool RL provide denser feedback \citep{lightman2024letsverify,uesato2022solving,qian2025toolrl,lin2025awpo}, but dense rewards do not by themselves prevent aggregation before advantage estimation. Existing credit-assignment methods mainly operate along a \emph{temporal} axis \citep{kazemnejad2025vineppo,feng2025gigpo,guo2025spo} or use token-level refinements, while ToolPO~\citep{deepagent2025} adds local tool rewards and RLTR~\citep{li2025rltr} separates planner and summarizer into a pipeline. SLCA instead decouples advantage estimation along a \emph{structural} axis within the same sampled rollout group: it routes independently normalized tool and summary advantages to their corresponding token segments without intermediate-state rollouts. This structural decomposition targets a different axis from temporal credit assignment; we do not evaluate a combined method here.

\section{Methodology}
\label{sec:method}

\paragraph{Core Design}
\slcaname turns the diagnosis above into three design requirements: expose structural token segments, obtain segment-specific feedback, and block cross-segment advantages before optimization. As summarized in \Cref{fig:framework}, the framework implements these requirements through mask-based segmentation, HierR segment returns, SLCA routing, SGLS rollouts, and a unified PPO-style objective, in contrast to post-hoc gradient-projection methods \citep{yu2020pcgrad}; a schematic comparison is provided in App.~\ref{fig:credit_misattribution}.

\begin{figure}[t]
    \centering
    \includegraphics[width=\linewidth]{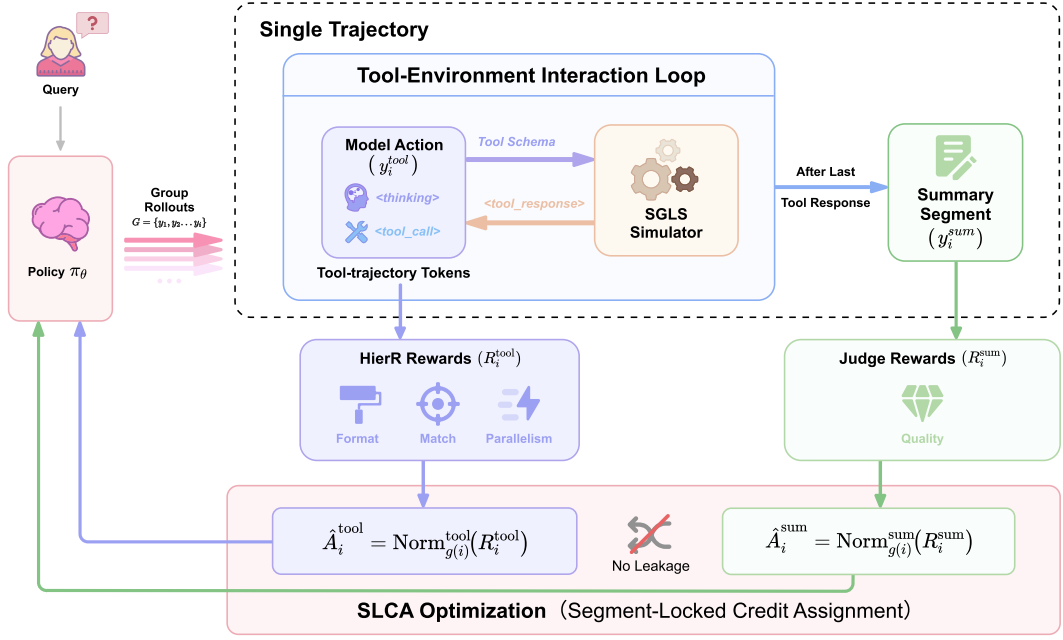}
    \caption{Overview of the \slcaname Framework.
(Top) Infrastructure: SGLS enables scalable exploration via schema-constrained simulation, generating interleaved trajectories.
(Middle) Signal: HierR decouples feedback into dense execution rewards ($\toolsup{R}$) and terminal outcome rewards ($\summarysup{R}$).
(Bottom) Optimization: SLCA decouples segment-wise advantages. By normalizing and routing advantages independently ($\toolsup{\hat{A}}$ vs. $\summarysup{\hat{A}}$), it blocks the defined summary-to-tool support path within each policy update.}
    \label{fig:framework}
\end{figure}

\subsection{Problem Formulation \& Segment Decomposition}
\label{subsec:formulation}

\paragraph{Heterogeneous Rollouts}
To lock credit to the right semantic target, the structural boundary must first be made explicit. Given an instruction $x \sim \mathcal{D}$, the policy $\pi_\theta$ generates a trajectory $y = (y_1, \dots, y_T)$ that decomposes as $y = [\toolsub{y} \oplus \summarysub{y}]$, where $\toolsub{y}$ comprises reasoning traces and structured tool calls, and $\summarysub{y}$ is the final user-facing response.

\paragraph{Gradient Masking}
To align optimization with this structure, a binary mask $m_{i,t} \in \{0, 1\}$ identifies \emph{learnable policy tokens} ($m_{i,t}=1$) versus \emph{frozen environment contexts} ($m_{i,t}=0$).
This mask restricts gradient propagation to agent actions and provides the structural signal for our automatic segment decomposition.

\paragraph{Automatic Segment Decomposition}
This mask partitions learnable tokens into disjoint semantic sets without training auxiliary segmenters.
For a rollout $i$, the \textbf{Summary Segment} $\textcolor{ArxivGreen}{\mathcal{T}_{i,\mathrm{sum}}}$ is the \emph{final contiguous run} of learnable tokens.
The \textbf{Tool Segment} $\textcolor{ArxivBlue}{\mathcal{T}_{i,\mathrm{tool}}}$ is the union of \emph{all} preceding learnable runs, capturing the entire reasoning-action chain including any intermediate thinking blocks between tool calls.
This deterministic mapping ensures that every gradient-bearing token is uniquely assigned to a semantic role, as long as the trajectory culminates in a single final-answer block (see App.~\ref{app:segmentation} for edge cases). \textbf{Scope.} SLCA targets the \emph{structural} axis (execution vs.\ articulation) rather than the \emph{temporal} axis within the tool segment (e.g., Action~1 vs.\ Action~2 in a ReAct chain). Thus, SLCA is not a variant of VinePPO/GiGPO/SPO, which address temporal credit assignment under different grouping assumptions (\S\ref{sec:related}).

\subsection{HierR and Segment Returns}
\label{subsec:hierr}

To make segment-locked routing meaningful, the reward must also separate execution quality from response articulation. We decompose the return into two semantically distinct components: $R(y) = \toolsub{R}(\toolsub{y}) + \summarysub{R}(\summarysub{y}\mid \toolsub{y},o)$.
\textbf{HierR} instantiates these signals as:
\begin{enumerate*}[label=(\roman*)] 
    \item Process Reward: A dense, structure-aware score for tool correctness and efficiency.
    \item Summary Preference: A terminal score evaluating the final answer quality.
\end{enumerate*}

For the $i$-th sampled rollout $y_i$, the segment rewards are computed by applying these functions directly:
\begin{equation}
\label{eq:reward_instantiation}
    \toolsup{R_i} = S_{\text{process}}(\textcolor{ArxivBlue}{y_{i,\mathrm{tool}}}), \quad
    \summarysup{R_i} = S_{\summaryword}(\textcolor{ArxivGreen}{y_{i,\mathrm{sum}}} \mid \textcolor{ArxivBlue}{y_{i,\mathrm{tool}}}, o_i).
\end{equation}
Full reward definitions and matching logic are detailed in App.~\ref{app:hierr_reward_details}.
Together, SGLS and HierR build on prior tool-use simulation and process-supervision work, but serve a specific role here: providing scalable, segment-specific feedback for testing structural credit routing.

\subsection{\textsc{\slcaname}}
\label{subsec:slca}

Given segment boundaries and segment returns, SLCA directly blocks the contamination channel identified in the introduction. Standard GRPO broadcasts a single trajectory-level advantage derived from $R_{\text{total}}$ to all learnable tokens, coupling tool-token gradients with summary rewards.

\paragraph{Decoupled Advantage Estimation}
For each rollout $i$, the two segment rewards $(\toolsup{R_i}, \summarysup{R_i})$ from Eq.~\ref{eq:reward_instantiation} are normalized \emph{separately} within the group:
\begin{equation}
\label{eq:slca_normalization}
    \toolsup{\hat{A}_i} = \frac{\toolsup{R_i} - \toolsup{\mu_g}}{\toolsup{\sigma_g} + \epsilon_{\text{norm}}}, \quad
    \summarysup{\hat{A}_i}  = \frac{\summarysup{R_i} - \summarysup{\mu_g}}{\summarysup{\sigma_g} + \epsilon_{\text{norm}}},
\end{equation}
where $\mu_g$ and $\sigma_g$ denote the mean and standard deviation within the group, and $\epsilon_{\text{norm}}$ is a numerical floor. The token-level advantage $\hat{A}^{\text{SLCA}}_{i,t}$ is then constructed by routing these signals exclusively to their semantic counterparts:
\begin{equation}
\hat A^{\text{SLCA}}_{i,t}=
\begin{cases}
\toolsub{\lambda}\,\toolsup{\hat A_i}, & t\in \textcolor{ArxivBlue}{\mathcal{T}_{i,\mathrm{tool}}},\\
\summarysub{\lambda}\,\summarysup{\hat A_i}, & t\in \textcolor{ArxivGreen}{\mathcal{T}_{i,\mathrm{sum}}},\\
0, & m_{i,t}=0,
\end{cases}
\label{eq:slca_routing}
\end{equation}
This construction changes only the scalar advantage attached to each token; it does not split the model or require additional rollouts.
Tool tokens are updated only by execution quality, while summary tokens are updated only by response quality.
Thus, summary rewards can still train final-answer articulation, but no longer supply gradients to tool-decision tokens.

\paragraph{Robust Optimization Mechanisms}
For sparse-tool stability, the implementation uses presence filtering, post-normalization weighting, and omission-penalty routing; details are in App.~\ref{app:slca_details}.

\paragraph{Theoretical Properties}
The key per-update guarantee of SLCA is advantage isolation: the tool-token gradient component is functionally independent of the summary reward,
\begin{equation}
\frac{\partial \toolsub{g}^{\text{SLCA}}}{\partial \summarysup{R}}=\mathbf{0}.
\end{equation}
Beyond this isolation property, local score-function analysis shows that SLCA removes a positive summary-noise variance term from tool updates and preserves the tool-execution direction in sign-conflict regimes. These are conditional, per-update guarantees rather than claims of global bias-free optimization; full assumptions and proofs are in App.~\ref{app:slca_theory}.


\subsection{Scalable Exploration via SGLS}
\label{subsec:sgls}

To test structural credit routing at scale, SGLS supplies schema-consistent observations without relying on live APIs. It combines deterministic schema validation with frozen cross-family LLM mocking, so invalid calls receive immediate structured feedback and valid calls receive plausible tool responses; deployment details and alignment examples are in App.~\ref{app:sgls_impl}. The cross-family design mitigates implicit leakage while preserving the semantic topology needed for transfer.

\subsection{Objective and Training Algorithm}
\label{subsec:training}

Finally, a single unified policy is trained by replacing GRPO's unified scalar advantage with the routed segment advantage. The objective is PPO-clip \citep{schulman2017ppo} with segment-locked advantages. Let $\rho_{i,t}(\theta) = \frac{\pi_\theta(y_{i,t} \mid h_{i,t})}{\pi_{\text{old}}(y_{i,t} \mid h_{i,t})}$; a per-token KL penalty $\beta \mathbb{D}_{\text{KL}}(\pi_\theta \| \pi_{\text{ref}})$ is enforced but omitted below for brevity:
\begin{equation}
\label{eq:slca_grpo_obj}
    \mathcal{J}(\theta) =
    \mathbb{E}_{x\sim\mathcal{D},\,\{y_i\}_{i=1}^{G}\sim\pi_{\theta_{\mathrm{old}}}(\cdot\mid x)}
    \Bigg[
    \frac{1}{N}\sum_{i,t} m_{i,t}\cdot
    \min\Big(
    \rho_{i,t}\hat{A}^{\text{SLCA}}_{i,t},
    \text{clip}(\rho_{i,t},1-\epsilon_{\text{clip}},1+\epsilon_{\text{clip}})\hat{A}^{\text{SLCA}}_{i,t}
    \Big)
    \Bigg],
\end{equation}
where $G$ is the rollout group size and $N = \sum_{i,t} m_{i,t}$ is the total valid token count.
The complete training procedure is summarized in Algorithm~\ref{alg:slca_grpo}.

\section{Experiments}
\label{sec:exp}

Our experiments proceed in three parts: we examine the diagnosed failure under unified or additive credit signals, test transfer across schemas and long-horizon interaction, and ablate SLCA, SGLS, and HierR.

\subsection{Experimental Setup}
\label{sec:exp_setup}

\paragraph{Datasets and Benchmarks}
\textsc{\slcaname} is evaluated across three dimensions: in-domain mastery, cross-distribution generalization, and collaborative robustness.
\begin{enumerate*}[label=(\roman*)] 
    \item \textbf{Toucan-1.5M (In-Domain):} The Toucan dataset \citep{xu2025toucan} is used for both SFT initialization and RL post-training. A multi-stage filtering pipeline (App.~\ref{app:data_pipeline}) yields 42,423 SFT samples and 31,818 RL training samples (spanning single-turn and decomposed multi-turn trajectories). Evaluation uses a held-out Toucan-Test set of 4,000 samples.
    \item \textbf{BFCL V3 (Generalization):} Generalization is evaluated on the Berkeley Function-Calling Leaderboard \citep{pmlr-v267-patil25a}. \footnote{
    For the main BFCL V3 comparison, relevance detection is excluded and evaluation is restricted to \textbf{single-turn} samples. This isolates \textit{atomic} generalization to unseen schemas (e.g., AST); the appendix additionally reports BFCL Multi-Turn accuracy, while $\tau^2$-Bench covers longer-horizon collaboration.
}
    \item \textbf{$\tau^2$-Bench (Robustness):} Robustness is measured in dynamic dual-control environments (Airline, Retail, Telecom) where agents must collaborate with users to manipulate state, reporting Pass$^1$ \citep{barres2025tau2bench}.
\end{enumerate*}
Full protocols are in App.~\ref{app:eval_protocols}.

\paragraph{Baselines and Training}
The methods are implemented on Qwen2.5-7B-Instruct \citep{qwen2_5_techreport} (default), Qwen2.5-3B-Instruct, and Qwen3-8B-Base \citep{qwen3_techreport} to verify scalability, comparing six paradigms (details in App.~\ref{app:training_details}):
\begin{enumerate*}[label=(\roman*)] 
    \item Original Backbones: The original model weights evaluated directly without any exposure to the Toucan training set.
    
    \item SFT (Toucan): A behavioral cloning baseline fine-tuned on the union of the standard SFT partition (42k) and the raw source trajectories of the RL partition (31k).

    \item SFT+GRPO (Baseline): The standard GRPO algorithm using the same HierR rewards and SGLS environment as \slcaname, with a unified advantage computed from $R_{\text{total}}=\toolsub{R}+\summarysub{R}$.
    
    \item ToolPO \citep{deepagent2025}: An additive credit-assignment baseline evaluated with its global outcome and local tool rewards. Tool tokens receive their sum, while summary tokens inherit the global term. The outcome-reward protocol and its sensitivity diagnostic are documented in App.~\ref{app:baseline_comparison}.

    \item RLTR (adapted) \citep{li2025rltr}: A planner-summarizer baseline with a dedicated planner-only SFT initialization and a frozen summarizer, following a two-stage planning pipeline.

    \item SFT+\textsc{\slcaname} (Ours): Our proposed framework employing Segment-Locked Credit Assignment.

\end{enumerate*}
The controlled isolation comparison is between SFT+GRPO and SFT+\textsc{\slcaname}: these two rows share the same data partitions, SFT initialization, SGLS endpoint, mocker configuration, decoding settings, evaluation protocol, rollout group size $G=16$, one RL epoch, and HierR rewards, differing in unified versus segment-locked advantage computation. The matched ablations use the same per-backbone run set; w/o HierR removes HierR, and w/o SGLS removes schema conditioning and deterministic validation. ToolPO uses the LLM-judge outcome reward in the main tables, while RLTR retains its method-specific protocol; the reference-rule ToolPO sensitivity diagnostic is reported in App.~\ref{app:baseline_comparison}.

\subsection{Main Results}
\label{sec:main_results}

\paragraph{Toucan-Test}
The in-domain evaluation first asks whether structural routing improves execution under matched data, simulator, and reward conditions. Table~\ref{tab:toucan_main} reports the matched GRPO comparison together with the method-specific baseline rows (full results in App.~\ref{app:full_results}). On Qwen2.5-7B-Instruct, \slcaname has a 79.13\% three-run mean for strict Success@0.9, 2.53 pp above matched GRPO. The corresponding mean gaps are +2.35 pp on Qwen2.5-3B-Instruct and +2.05 pp on Qwen3-8B-Base.

\begin{table}[H]
\centering
\footnotesize
\renewcommand{\arraystretch}{1.0} 
\setlength{\tabcolsep}{2pt} 

\caption{
Toucan-Test Results. Metrics: 
Name F1: Multiset F1 score measuring precision and recall of predicted tool-name occurrences against gold names;
ArgMatch: Arithmetic mean of argument key and value matching scores;
Process: The dense tool-segment reward ($S_{\text{process}}$) defined in Eq.~\ref{eq:hierr_process} (App.~\ref{app:hierr_reward_details});
Success: Strict binary indicator ($\mathbb{I}[S_{\text{process}} \geq 0.9]$) obtained by thresholding the process score.
All trained entries report mean$\pm$std over three runs; Original rows are point evaluations. The matched rows share the per-backbone SFT initialization, data split, SGLS endpoint, mocker configuration, decoding settings, evaluation protocol, $G=16$, one RL epoch, and run set. GRPO, \slcaname, w/o SLCA, and w/o SGLS also share the HierR definition and subweights; w/o HierR removes HierR. ToolPO and RLTR retain their method-specific protocols. The ToolPO rows use the LLM-judge outcome reward; App.~\ref{app:baseline_comparison} reports how this baseline responds to a rule-based outcome reward under the same estimator.
}

\label{tab:toucan_main}

\begin{tabularx}{\columnwidth}{l*{4}{>{\centering\arraybackslash}X}}
\toprule
\rowcolor{ArxivTableHead}
\textbf{Method} & \textbf{Name F1} & \textbf{ArgMatch} & \textbf{Process} & \textbf{Success} \\
\midrule

\rowcolor{scaleThreeTint}\textit{\textbf{Qwen2.5-3B-Instruct}} & & & & \\
\midrule
Original       & 0.5207 & 0.5073 & 0.5174  & 0.3816 \\
SFT                 & 0.7685{\scriptsize$\pm$.0069} & 0.7132{\scriptsize$\pm$.0239} & 0.7401{\scriptsize$\pm$.0050}  & 0.6500{\scriptsize$\pm$.0089} \\
SFT+GRPO            & 0.8910{\scriptsize$\pm$.0087} & 0.8072{\scriptsize$\pm$.0154} & 0.8577{\scriptsize$\pm$.0057}  & 0.7412{\scriptsize$\pm$.0115} \\
RLTR           & 0.7598{\scriptsize$\pm$.0149} & 0.6353{\scriptsize$\pm$.0229} & 0.6924{\scriptsize$\pm$.0162}  & 0.6627{\scriptsize$\pm$.0151} \\
ToolPO         & 0.8050{\scriptsize$\pm$.0142} & 0.7769{\scriptsize$\pm$.0160} & 0.8368{\scriptsize$\pm$.0141}  & 0.5128{\scriptsize$\pm$.0145} \\
\rowcolor{oursTint}\textbf{\textsc{Ours}} & \textbf{0.9006}{\scriptsize$\pm$.0087} & \textbf{0.8092}{\scriptsize$\pm$.0112} & \textbf{0.8625}{\scriptsize$\pm$.0049} & \textbf{0.7647}{\scriptsize$\pm$.0093} \\

\midrule

\rowcolor{scaleSevenTint}\textit{\textbf{Qwen2.5-7B-Instruct}} &  &  &  &  \\
\midrule
Original                & 0.8008 & 0.7388 & 0.7687  & 0.6224 \\
SFT                     & 0.8390{\scriptsize$\pm$.0075} & 0.7531{\scriptsize$\pm$.0224} & 0.8094{\scriptsize$\pm$.0082}  & 0.7214{\scriptsize$\pm$.0087} \\
SFT+GRPO                & 0.8971{\scriptsize$\pm$.0093} & 0.8350{\scriptsize$\pm$.0148} & 0.8667{\scriptsize$\pm$.0061}  & 0.7660{\scriptsize$\pm$.0127} \\
RLTR           & 0.7673{\scriptsize$\pm$.0132} & 0.6346{\scriptsize$\pm$.0205} & 0.7022{\scriptsize$\pm$.0143}  & 0.6545{\scriptsize$\pm$.0168} \\
ToolPO         & 0.3042{\scriptsize$\pm$.0168} & 0.3552{\scriptsize$\pm$.0192} & 0.3408{\scriptsize$\pm$.0157}  & 0.2000{\scriptsize$\pm$.0136} \\
\rowcolor{oursTint}\textbf{\textsc{Ours}} & \textbf{0.9164}{\scriptsize$\pm$.0075} & \textbf{0.8353}{\scriptsize$\pm$.0132} & \textbf{0.8766}{\scriptsize$\pm$.0045} & \textbf{0.7913}{\scriptsize$\pm$.0105} \\

\midrule

\rowcolor{scaleEightTint}\textit{\textbf{Qwen3-8B-Base}} &  &  &  &  \\
\midrule
Original                & 0.0887 & 0.1158 & 0.1089  & 0.0471 \\
SFT                     & 0.8058{\scriptsize$\pm$.0082} & 0.7437{\scriptsize$\pm$.0260} & 0.7754{\scriptsize$\pm$.0060}  & 0.6925{\scriptsize$\pm$.0098} \\
SFT+GRPO                & 0.9110{\scriptsize$\pm$.0092} & 0.8334{\scriptsize$\pm$.0135} & 0.8767{\scriptsize$\pm$.0050}   & 0.7697{\scriptsize$\pm$.0141} \\
RLTR           & 0.8780{\scriptsize$\pm$.0125} & 0.7468{\scriptsize$\pm$.0219} & 0.7941{\scriptsize$\pm$.0147}  & 0.7454{\scriptsize$\pm$.0167} \\
ToolPO         & 0.6854{\scriptsize$\pm$.0168} & 0.8180{\scriptsize$\pm$.0177} & 0.7510{\scriptsize$\pm$.0165}  & 0.3199{\scriptsize$\pm$.0156} \\
\rowcolor{oursTint}\textbf{\textsc{Ours}} & \textbf{0.9177}{\scriptsize$\pm$.0075} & 0.8320{\scriptsize$\pm$.0140} & \textbf{0.8786}{\scriptsize$\pm$.0046} & \textbf{0.7902}{\scriptsize$\pm$.0105} \\

\bottomrule
\end{tabularx}
\end{table}

\paragraph{Comparison with Prior Credit-Assignment Baselines}
The next comparison evaluates prior paradigms that either augment or separate the tool-planning signal rather than route segment advantages. RLTR (pipeline separation) and ToolPO (additive augmentation) provide complementary reference points for the segment-locked estimator. Their training traces show capacity-dependent format and invocation patterns (App.~\ref{app:baseline_dynamics}); the main-table ToolPO rows use the LLM-judge outcome reward described in App.~\ref{app:baseline_comparison}.

\paragraph{Training Dynamics and Efficiency}
Training traces reveal whether the final gains arise from stable correction of the misattribution channel rather than late-stage overfitting.
\Cref{fig:train_dynamics} contrasts training behavior from two views.
In the cross-scale training curves (\Cref{fig:baseline_dynamics_7b}), the plotted traces are representative single runs. \textsc{\slcaname} rises across the three backbones. The 7B ToolPO curve is a separate diagnostic: it declines after step~40, with the structural divergence appearing after step~80 (App.~\ref{app:baseline_dynamics}); RLTR changes more slowly under its coarse completeness reward.
Across the representative traces (3B/7B/8B; \Cref{fig:pareto}), \textsc{\slcaname} ends with \emph{fewer} average tool turns and \emph{higher} success rates than standard GRPO. The gap in trajectory length is narrow, and the ordering is consistent across the plotted traces. Standard GRPO stabilizes at longer trajectories without a corresponding success gain, compatible with ``performative execution'': padding trajectories to exploit summary rewards rather than improving tool-call precision.
These dynamics provide evidence for the failure mode diagnosed in the introduction: additive or unified signals can improve proxy rewards while degrading executable tool behavior.

\begin{figure}[!tbp]
    \centering
    \begin{subfigure}[t]{0.34\linewidth}
        \centering
        \includegraphics[height=3.8cm]{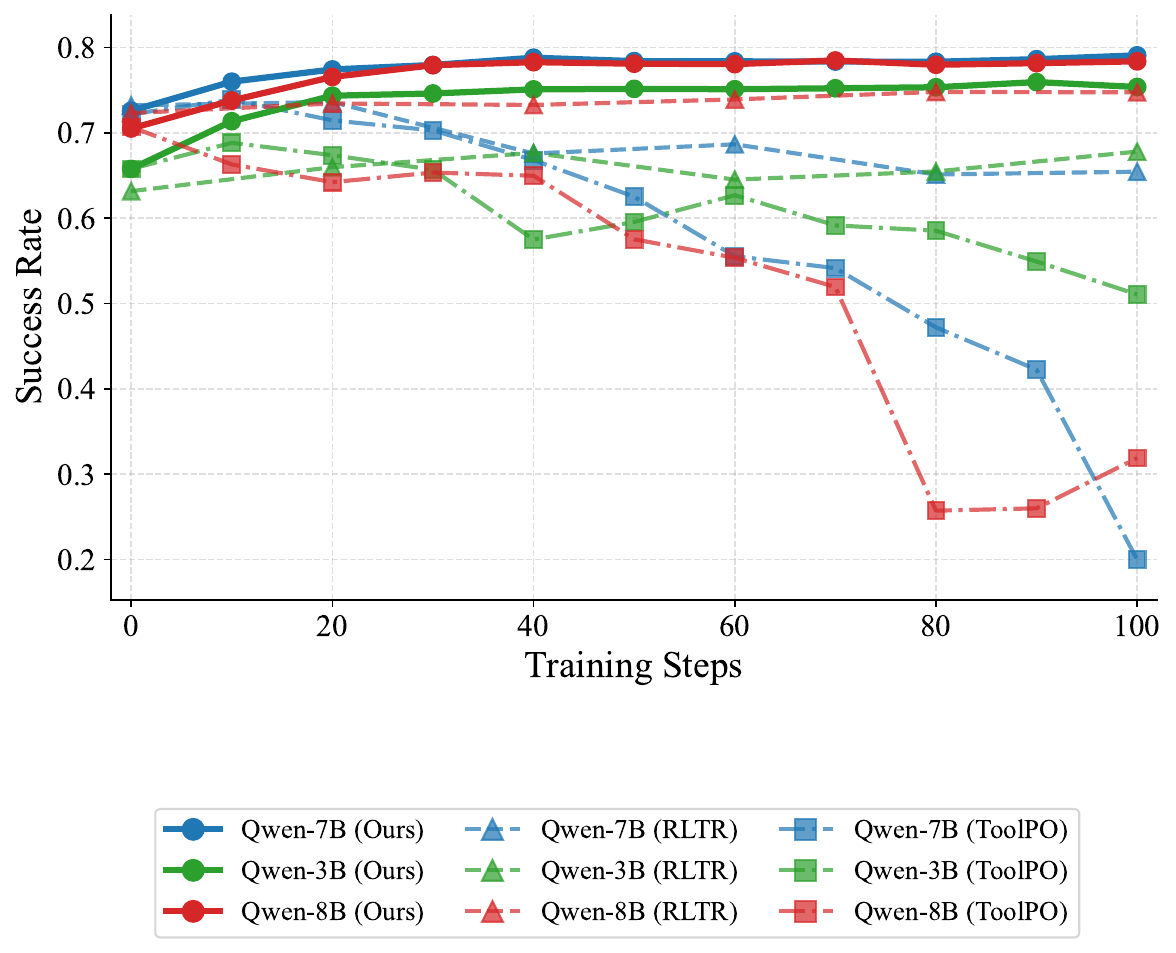}
        \caption{\textbf{Training curves across 3B/7B/8B.}
        Representative single-run traces.}
        \label{fig:baseline_dynamics_7b}
    \end{subfigure}\hfill
    \begin{subfigure}[t]{0.64\linewidth}
        \centering
        \includegraphics[height=3.8cm]{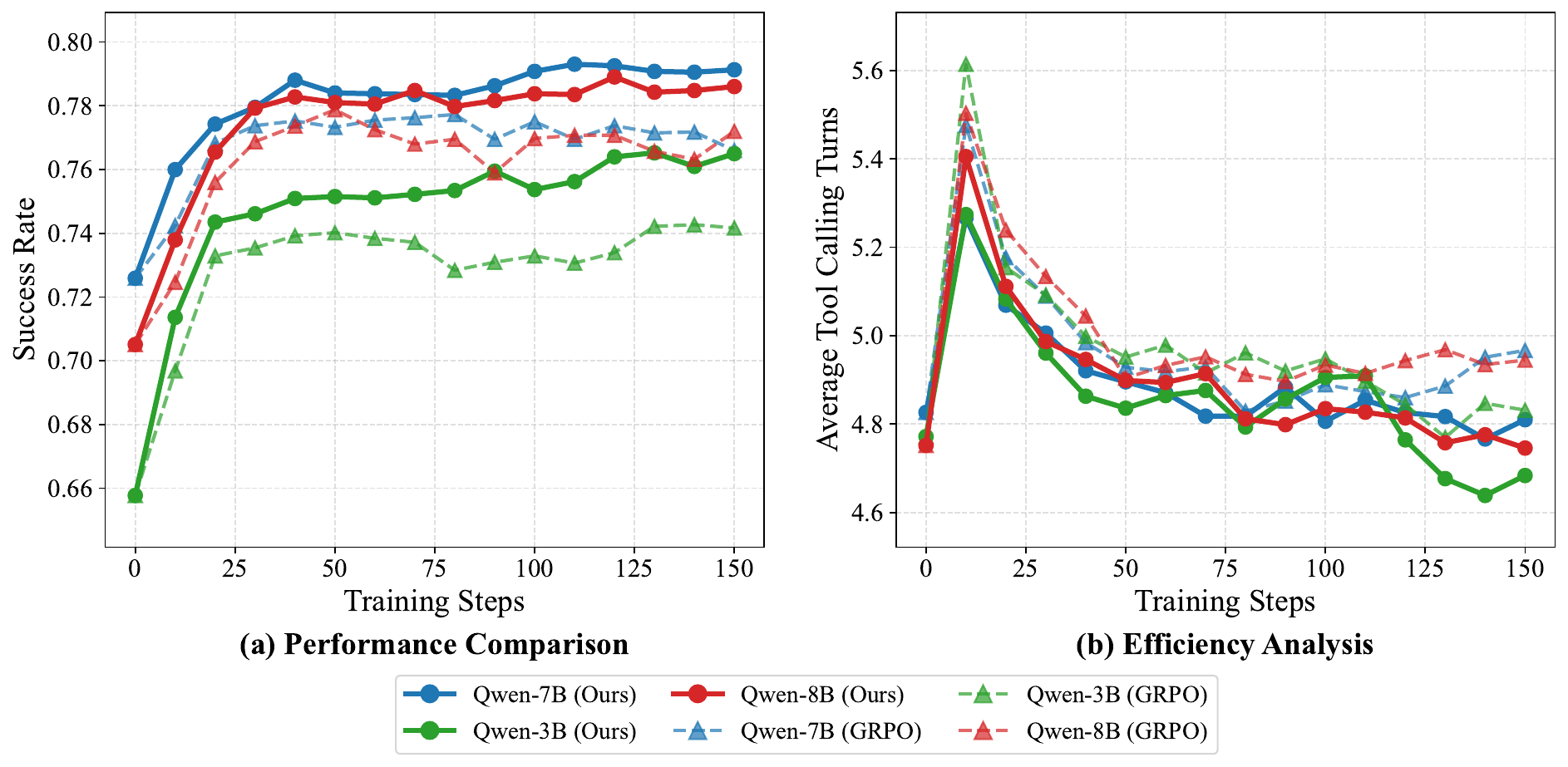}
        \caption{\textbf{Across 3B/7B/8B: success vs.\ tool turns.}
        The plotted traces show \slcaname (solid) ending with shorter trajectories than GRPO (dashed).}
        \label{fig:pareto}
    \end{subfigure}
    \caption{\textbf{Training Dynamics.}
    \textbf{(a) Baseline comparison across 3B/7B/8B} for SLCA / ToolPO / RLTR (the ToolPO diagnostic is described in App.~\ref{app:baseline_dynamics}), and
    \textbf{(b) cross-scale dynamics} of success rate and average tool turns over training steps.}
    \label{fig:train_dynamics}
\end{figure}


\paragraph{Generalization on BFCL and $\tau^2$-Bench}
The OOD evaluation then tests whether correcting structural credit assignment transfers beyond the in-domain simulator setting.
\Cref{fig:ood_scales} compares \slcaname with matched GRPO and the method-specific baselines on two OOD benchmarks across scales (protocols in App.~\ref{app:eval_protocols}).
On BFCL (\Cref{fig:bfcl_scales}), SLCA reaches \textbf{70.31$\pm$0.50\%} Overall Accuracy on Qwen3-8B-Base, compared with 66.96$\pm$0.32\% for matched GRPO and 68.50$\pm$0.35\% for SFT. The 7B BFCL gap over standard GRPO is +1.36\,pp; the corresponding gaps are +0.40\,pp on 3B and +3.35\,pp on 8B. On $\tau^2$-Bench, the matched gaps are +1.09\,pp (3B), +9.15\,pp (7B), and +10.03\,pp (8B), all based on three-run means. ToolPO outcome-reward details and the LLM-judge protocol used in the scale rows are documented in App.~\ref{app:baseline_comparison}. Multi-Turn BFCL results and per-domain $\tau^2$-Bench breakdowns are in App.~\ref{app:bfcl_multiturn} and App.~\ref{app:tau2_details}.
Robustness under larger candidate tool spaces ($K$ up to 150) with hard-negative distractors is evaluated in App.~\ref{app:full_toolspace}.

\begin{figure}[!tbp]
    \centering
    \begin{subfigure}[t]{0.49\linewidth}
        \centering
        \includegraphics[width=\linewidth]{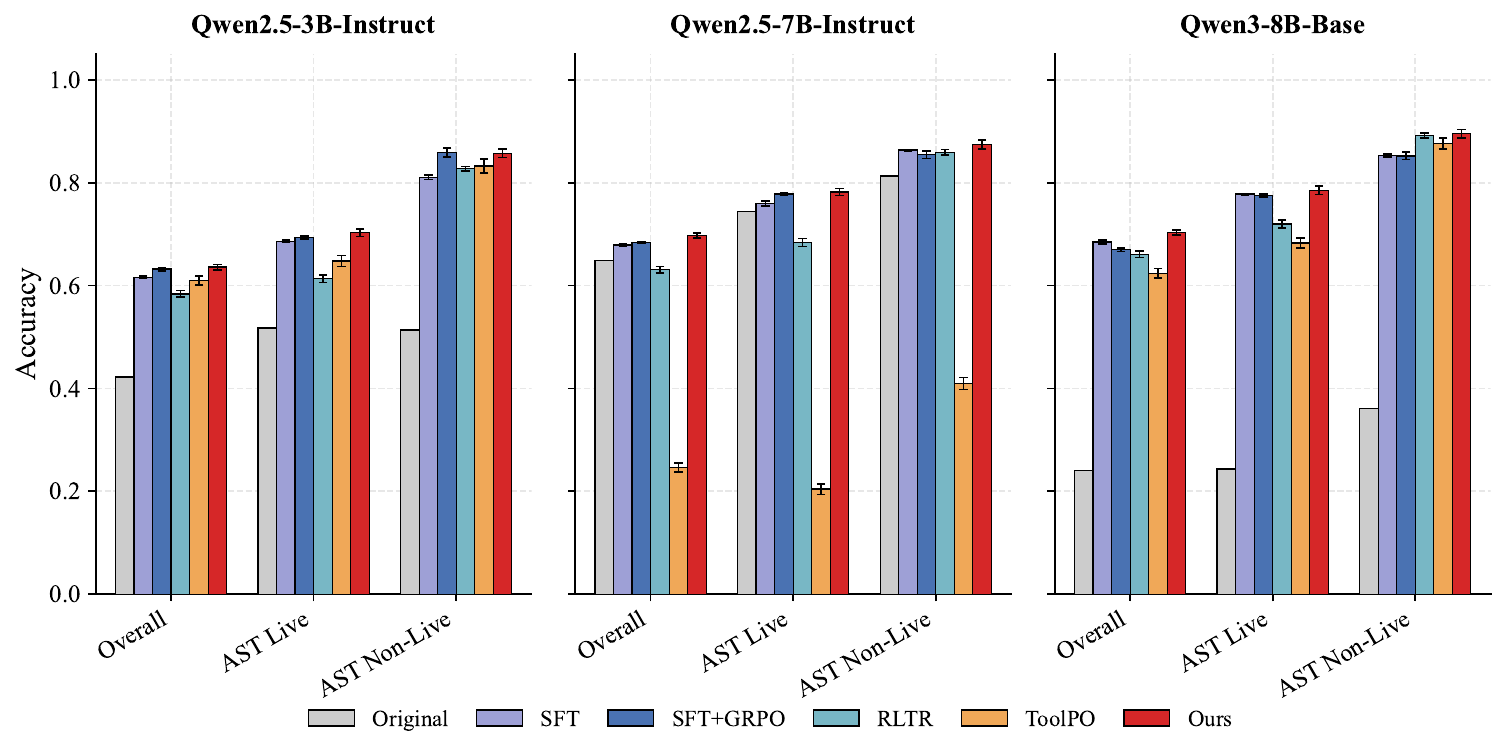}
        \caption{\textbf{BFCL V3 (Single-Turn; mean$\pm$std).}
        Trained bars show three-run means with one-standard-deviation error bars; Original is a point evaluation.}
        \label{fig:bfcl_scales}
    \end{subfigure}\hfill
    \begin{subfigure}[t]{0.49\linewidth}
        \centering
        \includegraphics[width=\linewidth]{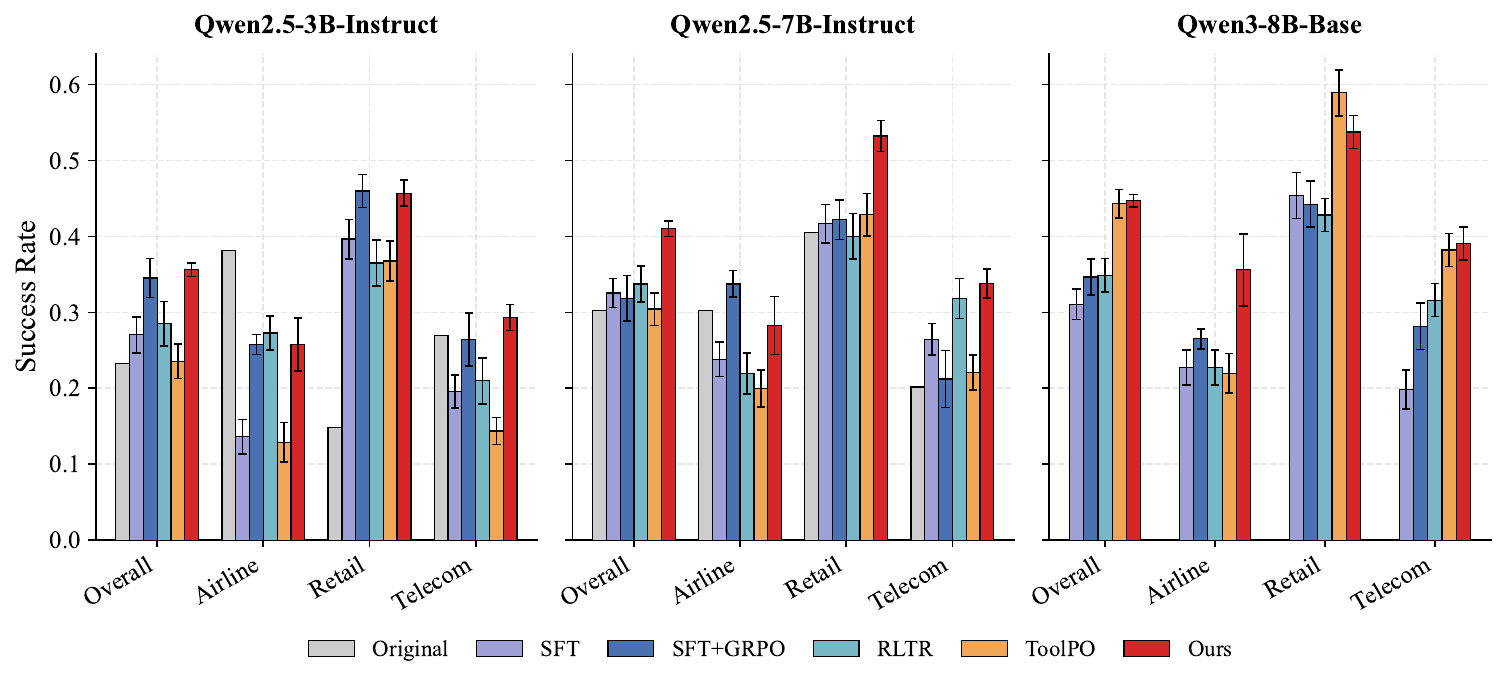}
        \caption{\textbf{$\tau^2$-Bench (Overall Pass$^1$; mean$\pm$std).}
        Trained rows are three-run means with one-standard-deviation error bars; Original is a point evaluation. The 8B Original point is omitted because its four domain scores are zero.}
        \label{fig:tau2_scales}
    \end{subfigure}
    \caption{\textbf{OOD Generalization across Scales.}
    \textbf{(a) Atomic generalization on BFCL V3} and
    \textbf{(b) Collaboration robustness on $\tau^2$-Bench}
    for the three backbones (3B, 7B, 8B), comparing \slcaname against
    standard GRPO, RLTR, and ToolPO. Per-domain $\tau^2$-Bench details are in
    App.~\ref{app:tau2_details}.}
    \label{fig:ood_scales}
\end{figure}

\subsection{Ablation Studies}
\label{sec:ablation}
Finally, the ablations target the method's three design requirements: routing, scalable simulation, and dense segment feedback. Qwen2.5-7B-Instruct is used as the primary testbed (Table~\ref{tab:ablation}), with detailed ablation across scales in App.~\ref{app:full_ablation}.

\begin{enumerate*}[label=(\roman*)] 
    \item \textbf{w/o SLCA (Unified Advantage)}: Reverting to the standard unified advantage ($\toolsub{R} + \summarysub{R}$) changes both normalization and token support; it lowers Toucan Success by 2.53 pp on 7B under the matched setting. The support-control experiment in App.~\ref{app:support_control_7b} separately measures the contribution of removing the summary-to-tool support path.
    \item \textbf{w/o SGLS:} Removing deterministic schema validation and schema-conditioned tool-response prompting lowers the reported OOD transfer metrics (Table~\ref{tab:ablation}); its Multi-Turn effect is reported in App.~\ref{app:bfcl_multiturn}. Schema-constrained simulation supplies the controlled exploration interface.
    \item \textbf{w/o HierR:} Relying solely on sparse summary rewards leads to slow convergence and poor argument alignment on complex queries.
\end{enumerate*}
The matched ``w/o SLCA'' condition compares the unified and segment-locked estimators under fixed reward definitions, so it measures the combined estimator change rather than either operation in isolation; the unified reward ratio sweep in App.~\ref{app:unified_ratio_sweep} examines this distinction directly. Together, these ablations test the three design requirements in Section~\ref{sec:method}: routing addresses the cross-segment path, SGLS stabilizes schema-grounded exploration, and HierR supplies dense execution feedback. The summary advantage ablation and the 7B support control are reported in Apps.~\ref{app:summary_advantage_ablation} and~\ref{app:support_control_7b}.

\begin{table}[ht]
\centering
\small
\caption{Ablation Studies on Qwen2.5-7B-Instruct. Entries report mean$\pm$std over three independent runs. Full results are provided in App.~\ref{app:full_ablation}. Process is an evaluation metric computed post hoc for all settings, including w/o HierR.}
\label{tab:ablation}
\setlength{\tabcolsep}{3.5pt} 
\begin{tabular}{lcccc}
\toprule
\rowcolor{ArxivTableHead}
 & \multicolumn{2}{c}{\textbf{Toucan-Test}} & \textbf{BFCL} & \textbf{$\tau^2$-Bench} \\
\cmidrule{2-5}
\textbf{Setting} & \textbf{Process} & \textbf{Success} & \textbf{Acc} & \textbf{Pass$^1$} \\
\midrule
\rowcolor{oursTint}\textbf{\textsc{\slcaname}} & 0.8766{\scriptsize$\pm$.0045}  & \textbf{0.7913}{\scriptsize$\pm$.0105} & \textbf{0.6977}{\scriptsize$\pm$.0047} & \textbf{0.4102}{\scriptsize$\pm$.0101} \\
w/o SLCA & 0.8667{\scriptsize$\pm$.0061} & 0.7660{\scriptsize$\pm$.0127} & 0.6841{\scriptsize$\pm$.0011} & 0.3187{\scriptsize$\pm$.0297} \\
w/o SGLS & 0.8787{\scriptsize$\pm$.0051} & 0.7802{\scriptsize$\pm$.0112} & 0.6875{\scriptsize$\pm$.0048} & 0.3619{\scriptsize$\pm$.0195} \\
w/o HierR & 0.8521{\scriptsize$\pm$.0083} & 0.7344{\scriptsize$\pm$.0121} & 0.6907{\scriptsize$\pm$.0046} & 0.3450{\scriptsize$\pm$.0144} \\
\bottomrule
\end{tabular}
\end{table}

\paragraph{Reward and response-mocker checks}
We also test SLCA with a binary $R_{\text{succ}}$ judge without per-example gold-call matching and with a GPT-OSS-120B response mocker; the matched results are reported in Apps.~\ref{app:exec_reward} and \ref{app:sgls_impl}.


\section{Conclusion}
\label{sec:conclusion}

We identified Global Signal Conflation as a structural pathology in tool-calling RL and proposed \slcaname, which decouples advantage estimation at the segment level. Across three backbones, \slcaname improves success by \textbf{+2.53\,pp} on 7B in-domain Toucan, \textbf{+1.36\,pp} on BFCL, and \textbf{+9.15\,pp} on $\tau^2$-Bench. The corresponding matched gaps are +2.35/+0.40/+1.09\,pp on 3B and +2.05/+3.35/+10.03\,pp on 8B for Toucan, BFCL, and $\tau^2$-Bench, respectively. The Toucan breakdown shows higher mean Process and Success for \slcaname across the reported backbones, while higher Summary scores for matched GRPO on 3B and 8B do not translate into higher task Success. Controlled comparisons and ablations support the role of segment-level routing, and the reward-swap check shows that the improvement is not tied to a single reward formulation. SLCA is complementary to temporal credit-assignment methods.

\section{Limitations}
\label{sec:limitations}

\textbf{Segment decomposition.} SLCA assumes a clear boundary between tool calls and free-form text. Settings without one, such as inline code generation, would require learned segmentation.
\textbf{Reward and bias.} $S_{\text{process}}$ relies on matching to gold trajectories and may under-reward alternative valid tool sequences when multiple API plans solve the same request. The execution-based judge check in App.~\ref{app:exec_reward} provides a complementary reward formulation. Toucan strict Success@0.9 thresholds this score instead of using an independent task outcome label. $\summarysup{\hat{A}}$ inherits a grouping bias in summary tokens (App.~\ref{par:asum_bias}); explicit subgrouping would shrink group size exponentially.
\textbf{Intra-tool temporal credit.} SLCA routes one scalar $\toolsup{\hat{A}}$ to every tool token, so it does not distinguish a correct Action~1 from a failing Action~2. This temporal issue is separate from cross-segment credit. A combination with VinePPO/GiGPO/SPO remains untested.
\textbf{Scale and simulation.} Experiments use models up to 8B with SGLS. Larger-scale training and direct real-API comparisons remain future work.

%

{\small\bibliographystyle{acl_natbib}\bibliography{example_paper}}

\begin{thebibliography}{42}
\providecommand{\natexlab}[1]{#1}

\bibitem[{Barres et~al.(2025)Barres, Dong, Ray, Si, and
  Narasimhan}]{barres2025tau2bench}
Victor Barres, Honghua Dong, Soham Ray, Xujie Si, and Karthik Narasimhan. 2025.
\newblock $\tau^2$-{Bench}: Evaluating conversational agents in a dual-control
  environment.
\newblock \emph{arXiv preprint arXiv:2506.07982}.

\bibitem[{Bengio et~al.(2015)Bengio, Vinyals, Jaitly, and
  Shazeer}]{bengio2015scheduled}
Samy Bengio, Oriol Vinyals, Navdeep Jaitly, and Noam Shazeer. 2015.
\newblock Scheduled sampling for sequence prediction with recurrent neural
  networks.
\newblock In \emph{Advances in Neural Information Processing Systems},
  volume~28, pages 1171--1179. Curran Associates, Inc.

\bibitem[{Christiano et~al.(2017)Christiano, Leike, Brown, Martic, Legg, and
  Amodei}]{christiano2017preferences}
Paul~F Christiano, Jan Leike, Tom Brown, Miljan Martic, Shane Legg, and Dario
  Amodei. 2017.
\newblock Deep reinforcement learning from human preferences.
\newblock In \emph{Advances in Neural Information Processing Systems},
  volume~30, pages 4299--4307. Curran Associates, Inc.

\bibitem[{Chu et~al.(2025)Chu, Zhai, Yang, Tong, Xie, Schuurmans, Le, Levine,
  and Ma}]{chu2025sftmemorizes}
Tianzhe Chu, Yuexiang Zhai, Jihan Yang, Shengbang Tong, Saining Xie, Dale
  Schuurmans, Quoc~V Le, Sergey Levine, and Yi~Ma. 2025.
\newblock {SFT} memorizes, {RL} generalizes: A comparative study of foundation
  model post-training.
\newblock In \emph{Proceedings of the 42nd International Conference on Machine
  Learning}, volume 267, pages 10818--10838.

\bibitem[{Feng et~al.(2025)Feng, Xue, Liu, and An}]{feng2025gigpo}
Lang Feng, Zhenghai Xue, Tingcong Liu, and Bo~An. 2025.
\newblock \href {https://doi.org/10.52202/085713-1544} {Group-in-group policy
  optimization for {LLM} agent training}.
\newblock In \emph{The Thirty-ninth Annual Conference on Neural Information
  Processing Systems}, volume~38, pages 51797--51830.

\bibitem[{Guo et~al.(2025{\natexlab{a}})Guo, Yang, Zhang, Song, Wang, Zhu
  et~al.}]{guo2025deepseekr1}
Daya Guo, Dejian Yang, Haowei Zhang, Junxiao Song, Peiyi Wang, Qihao Zhu,
  et~al. 2025{\natexlab{a}}.
\newblock \href {https://doi.org/10.1038/s41586-025-09422-z} {{DeepSeek-R1}
  incentivizes reasoning in {LLMs} through reinforcement learning}.
\newblock \emph{Nature}, 645:633--638.

\bibitem[{Guo et~al.(2025{\natexlab{b}})Guo, Xu, Liu, Ye, and Qiu}]{guo2025spo}
Yiran Guo, Lijie Xu, Jie Liu, Dan Ye, and Shuang Qiu. 2025{\natexlab{b}}.
\newblock \href {https://doi.org/10.52202/085713-3815} {Segment policy
  optimization: Effective segment-level credit assignment in {RL} for large
  language models}.
\newblock In \emph{The Thirty-ninth Annual Conference on Neural Information
  Processing Systems}, volume~38, pages 126866--126898.

\bibitem[{Guo et~al.(2024)Guo, Cheng, Wang, Liang, Qin, Li, Liu, Sun, and
  Liu}]{guo-etal-2024-stabletoolbench}
Zhicheng Guo, Sijie Cheng, Hao Wang, Shihao Liang, Yujia Qin, Peng Li, Zhiyuan
  Liu, Maosong Sun, and Yang Liu. 2024.
\newblock \href {https://doi.org/10.18653/v1/2024.findings-acl.664}
  {{S}table{T}ool{B}ench: Towards stable large-scale benchmarking on tool
  learning of large language models}.
\newblock In \emph{Findings of the Association for Computational Linguistics:
  ACL 2024}, pages 11143--11156. Association for Computational Linguistics.

\bibitem[{Hou et~al.(2026)Hou, Zhao, Wang, and Wang}]{hou2025mcp}
Xinyi Hou, Yanjie Zhao, Shenao Wang, and Haoyu Wang. 2026.
\newblock \href {https://doi.org/10.1145/3796519} {{Model Context Protocol}
  ({MCP}): Landscape, security threats, and future research directions}.
\newblock \emph{ACM Transactions on Software Engineering and Methodology}.

\bibitem[{Kazemnejad et~al.(2025)Kazemnejad, Aghajohari, Portelance, Sordoni,
  Reddy, Courville, and Roux}]{kazemnejad2025vineppo}
Amirhossein Kazemnejad, Milad Aghajohari, Eva Portelance, Alessandro Sordoni,
  Siva Reddy, Aaron Courville, and Nicolas~Le Roux. 2025.
\newblock Vine{PPO}: Refining credit assignment in {RL} training of {LLM}s.
\newblock In \emph{Proceedings of the 42nd International Conference on Machine
  Learning}, volume 267, pages 29557--29590.

\bibitem[{Kingma and Ba(2015)}]{kingma2015adam}
Diederik~P. Kingma and Jimmy Ba. 2015.
\newblock Adam: {A} method for stochastic optimization.
\newblock In \emph{3rd International Conference on Learning Representations,
  {ICLR} 2015, San Diego, CA, USA, May 7-9, 2015, Conference Track
  Proceedings}.

\bibitem[{Li et~al.(2026{\natexlab{a}})Li, Sun, Huang, Zhong, Jiang, Han,
  Zhang, Wang, and Liu}]{li2025preferenceleakage}
Dawei Li, Renliang Sun, Yue Huang, Ming Zhong, Bohan Jiang, Jiawei Han,
  Xiangliang Zhang, Wei Wang, and Huan Liu. 2026{\natexlab{a}}.
\newblock \href {https://openreview.net/forum?id=grIvSXVJ65} {Preference
  leakage: A contamination problem in {LLM}-as-a-judge}.
\newblock In \emph{International Conference on Learning Representations}.

\bibitem[{Li et~al.(2023)Li, Zhao, Yu, Song, Li, Yu, Li, Huang, and
  Li}]{li2023apibank}
Minghao Li, Yingxiu Zhao, Bowen Yu, Feifan Song, Hangyu Li, Haiyang Yu, Zhoujun
  Li, Fei Huang, and Yongbin Li. 2023.
\newblock {API-Bank}: A comprehensive benchmark for tool-augmented {LLM}s.
\newblock In \emph{Proceedings of the 2023 Conference on Empirical Methods in
  Natural Language Processing}, pages 3102--3116. Association for Computational
  Linguistics.

\bibitem[{Li et~al.(2026{\natexlab{b}})Li, Jiao, Jin, Dong, Jin, Wang, Wang,
  Zhu, Wen, Lu, and Dou}]{deepagent2025}
Xiaoxi Li, Wenxiang Jiao, Jiarui Jin, Guanting Dong, Jiajie Jin, Yinuo Wang,
  Hao Wang, Yutao Zhu, Ji-Rong Wen, Yuan Lu, and Zhicheng Dou.
  2026{\natexlab{b}}.
\newblock \href {https://doi.org/10.1145/3774904.3792460} {{DeepAgent}: A
  general reasoning agent with scalable toolsets}.
\newblock In \emph{Proceedings of the ACM Web Conference 2026}, pages
  2219--2230. Association for Computing Machinery.

\bibitem[{Li et~al.(2025)Li, Hu, and Wang}]{li2025rltr}
Zhiwei Li, Yong Hu, and Wenqing Wang. 2025.
\newblock \href {https://doi.org/10.18653/v1/2025.emnlp-industry.116}
  {Encouraging good processes without the need for good answers: Reinforcement
  learning for {LLM} agent planning}.
\newblock In \emph{Proceedings of the 2025 Conference on Empirical Methods in
  Natural Language Processing: Industry Track}, pages 1654--1666. Association
  for Computational Linguistics.

\bibitem[{Lightman et~al.(2024)Lightman, Kosaraju, Burda, Edwards, Baker, Lee,
  Leike, Schulman, Sutskever, and Cobbe}]{lightman2024letsverify}
Hunter Lightman, Vineet Kosaraju, Yuri Burda, Harrison Edwards, Bowen Baker,
  Teddy Lee, Jan Leike, John Schulman, Ilya Sutskever, and Karl Cobbe. 2024.
\newblock Let's verify step by step.
\newblock In \emph{The Twelfth International Conference on Learning
  Representations}.

\bibitem[{Lin et~al.(2025)Lin, Wang, Yang, Chai, Cao, Yin, Lin, and
  He}]{lin2025awpo}
Zihan Lin, Xiaohan Wang, Hexiong Yang, Jiajun Chai, Jie Cao, Guojun Yin, Wei
  Lin, and Ran He. 2025.
\newblock {AWPO}: Enhancing tool-use of large language models through adaptive
  integration of reasoning rewards.
\newblock \emph{arXiv preprint arXiv:2512.19126}.

\bibitem[{Loshchilov and Hutter(2019)}]{loshchilov2019adamw}
Ilya Loshchilov and Frank Hutter. 2019.
\newblock Decoupled weight decay regularization.
\newblock In \emph{International Conference on Learning Representations}.

\bibitem[{Ouyang et~al.(2022)Ouyang, Wu, Jiang, Almeida, Wainwright, Mishkin,
  Zhang, Agarwal, Slama, Ray, Schulman, Hilton, Kelton, Miller, Simens, Askell,
  Welinder, Christiano, Leike, and Lowe}]{ouyang2022instructgpt}
Long Ouyang, Jeffrey Wu, Xu~Jiang, Diogo Almeida, Carroll Wainwright, Pamela
  Mishkin, Chong Zhang, Sandhini Agarwal, Katarina Slama, Alex Ray, John
  Schulman, Jacob Hilton, Fraser Kelton, Luke Miller, Maddie Simens, Amanda
  Askell, Peter Welinder, Paul~F Christiano, Jan Leike, and Ryan Lowe. 2022.
\newblock Training language models to follow instructions with human feedback.
\newblock In \emph{Advances in Neural Information Processing Systems},
  volume~35, pages 27730--27744. Curran Associates, Inc.

\bibitem[{Patil et~al.(2025)Patil, Mao, Yan, Ji, Suresh, Stoica, and
  Gonzalez}]{pmlr-v267-patil25a}
Shishir~G Patil, Huanzhi Mao, Fanjia Yan, Charlie Cheng-Jie Ji, Vishnu Suresh,
  Ion Stoica, and Joseph~E. Gonzalez. 2025.
\newblock \href {https://openreview.net/forum?id=2GmDdhBdDk} {The {Berkeley}
  function calling leaderboard ({BFCL}): From tool use to agentic evaluation of
  large language models}.
\newblock In \emph{Proceedings of the 42nd International Conference on Machine
  Learning}, volume 267, pages 48371--48392.

\bibitem[{Patil et~al.(2024)Patil, Zhang, Wang, and
  Gonzalez}]{patil2024gorilla}
Shishir~G. Patil, Tianjun Zhang, Xin Wang, and Joseph~E. Gonzalez. 2024.
\newblock Gorilla: Large language model connected with massive {APIs}.
\newblock In \emph{Advances in Neural Information Processing Systems},
  volume~37, pages 126544--126565. Curran Associates, Inc.

\bibitem[{Qian et~al.(2025)Qian, Acikgoz, He, Wang, Chen, Hakkani-T{\"u}r, Tur,
  and Ji}]{qian2025toolrl}
Cheng Qian, Emre~Can Acikgoz, Qi~He, Hongru Wang, Xiusi Chen, Dilek
  Hakkani-T{\"u}r, Gokhan Tur, and Heng Ji. 2025.
\newblock \href {https://doi.org/10.52202/085713-3524} {Tool{RL}: Reward is all
  tool learning needs}.
\newblock In \emph{The Thirty-ninth Annual Conference on Neural Information
  Processing Systems}, volume~38, pages 116896--116926.

\bibitem[{Qin et~al.(2024)Qin, Liang, Ye, Zhu, Yan, Lu, Lin, Cong, Tang, Qian,
  Zhao, Hong, Tian, Xie, Zhou, Gerstein, Li, Liu, and Sun}]{qin2024toolllm}
Yujia Qin, Shihao Liang, Yining Ye, Kunlun Zhu, Lan Yan, Yaxi Lu, Yankai Lin,
  Xin Cong, Xiangru Tang, Bill Qian, Sihan Zhao, Lauren Hong, Runchu Tian,
  Ruobing Xie, Jie Zhou, Mark Gerstein, Dahai Li, Zhiyuan Liu, and Maosong Sun.
  2024.
\newblock Tool{LLM}: Facilitating large language models to master 16000+
  real-world {API}s.
\newblock In \emph{The Twelfth International Conference on Learning
  Representations}.

\bibitem[{{Qwen Team}(2025{\natexlab{a}})}]{qwen2_5_techreport}
{Qwen Team}. 2025{\natexlab{a}}.
\newblock Qwen2.5 technical report.
\newblock \emph{arXiv preprint arXiv:2412.15115}.

\bibitem[{{Qwen Team}(2025{\natexlab{b}})}]{qwen3_techreport}
{Qwen Team}. 2025{\natexlab{b}}.
\newblock Qwen3 technical report.
\newblock \emph{arXiv preprint arXiv:2505.09388}.

\bibitem[{Rafailov et~al.(2023)Rafailov, Sharma, Mitchell, Manning, Ermon, and
  Finn}]{rafailov2023dpo}
Rafael Rafailov, Archit Sharma, Eric Mitchell, Christopher~D Manning, Stefano
  Ermon, and Chelsea Finn. 2023.
\newblock Direct preference optimization: Your language model is secretly a
  reward model.
\newblock In \emph{Advances in Neural Information Processing Systems},
  volume~36, pages 53728--53741. Curran Associates, Inc.

\bibitem[{Ross et~al.(2011)Ross, Gordon, and Bagnell}]{ross2011dagger}
Stephane Ross, Geoffrey Gordon, and Drew Bagnell. 2011.
\newblock A reduction of imitation learning and structured prediction to
  no-regret online learning.
\newblock In \emph{Proceedings of the Fourteenth International Conference on
  Artificial Intelligence and Statistics}, volume~15, pages 627--635. PMLR.

\bibitem[{Ruan et~al.(2024)Ruan, Dong, Wang, Pitis, Zhou, Ba, Dubois, Maddison,
  and Hashimoto}]{ruan2024toolemu}
Yangjun Ruan, Honghua Dong, Andrew Wang, Silviu Pitis, Yongchao Zhou, Jimmy Ba,
  Yann Dubois, Chris Maddison, and Tatsunori Hashimoto. 2024.
\newblock Identifying the risks of {LM} agents with an {LM}-emulated sandbox.
\newblock In \emph{International Conference on Learning Representations},
  volume 2024, pages 27031--27098.

\bibitem[{Schick et~al.(2023)Schick, Dwivedi-Yu, Dessi, Raileanu, Lomeli,
  Hambro, Zettlemoyer, Cancedda, and Scialom}]{schick2023toolformer}
Timo Schick, Jane Dwivedi-Yu, Roberto Dessi, Roberta Raileanu, Maria Lomeli,
  Eric Hambro, Luke Zettlemoyer, Nicola Cancedda, and Thomas Scialom. 2023.
\newblock Toolformer: Language models can teach themselves to use tools.
\newblock In \emph{Advances in Neural Information Processing Systems},
  volume~36, pages 68539--68551. Curran Associates, Inc.

\bibitem[{Schulman et~al.(2015)Schulman, Levine, Abbeel, Jordan, and
  Moritz}]{schulman2015trpo}
John Schulman, Sergey Levine, Pieter Abbeel, Michael Jordan, and Philipp
  Moritz. 2015.
\newblock Trust region policy optimization.
\newblock In \emph{Proceedings of the 32nd International Conference on Machine
  Learning}, volume~37, pages 1889--1897. PMLR.

\bibitem[{Schulman et~al.(2017)Schulman, Wolski, Dhariwal, Radford, and
  Klimov}]{schulman2017ppo}
John Schulman, Filip Wolski, Prafulla Dhariwal, Alec Radford, and Oleg Klimov.
  2017.
\newblock Proximal policy optimization algorithms.
\newblock In \emph{arXiv preprint arXiv:1707.06347}.

\bibitem[{Stiennon et~al.(2020)Stiennon, Ouyang, Wu, Ziegler, Lowe, Voss,
  Radford, Amodei, and Christiano}]{stiennon2020summarize}
Nisan Stiennon, Long Ouyang, Jeffrey Wu, Daniel Ziegler, Ryan Lowe, Chelsea
  Voss, Alec Radford, Dario Amodei, and Paul~F Christiano. 2020.
\newblock Learning to summarize with human feedback.
\newblock In \emph{Advances in Neural Information Processing Systems},
  volume~33, pages 3008--3021. Curran Associates, Inc.

\bibitem[{Uesato et~al.(2022)Uesato, Kushman, Kumar, Song, Siegel, Wang,
  Creswell, Irving, and Higgins}]{uesato2022solving}
Jonathan Uesato, Nate Kushman, Ramana Kumar, Francis Song, Noah Siegel, Lisa
  Wang, Antonia Creswell, Geoffrey Irving, and Irina Higgins. 2022.
\newblock Solving math word problems with process- and outcome-based feedback.
\newblock \emph{arXiv preprint arXiv:2211.14275}.

\bibitem[{Wang et~al.(2025)Wang, Liu, Jiang, Liu, Qi, Chen, and
  He}]{wang2025grpoverif}
Xiaoxuan Wang, Bo~Liu, Song Jiang, Jingzhou Liu, Jingyuan Qi, Xia Chen, and
  Baosheng He. 2025.
\newblock From solving to verifying: A unified objective for robust reasoning
  in {LLMs}.
\newblock \emph{arXiv preprint arXiv:2511.15137}.

\bibitem[{Wataoka et~al.(2025)Wataoka, Takahashi, and
  Ri}]{wataoka2025selfpreference}
Koki Wataoka, Tsubasa Takahashi, and Ryokan Ri. 2025.
\newblock Self-preference bias in {LLM}-as-a-judge.
\newblock \emph{arXiv preprint arXiv:2410.21819}.

\bibitem[{Wei et~al.(2025{\natexlab{a}})Wei, Zeng, Li, Wang, Brown, Frunza,
  Deng, Schneider, Nevmyvaka, Zhao, Garcia, and Hong}]{wei2025mtgrpo}
Quan Wei, Siliang Zeng, Chenliang Li, Zhongruo Wang, William Brown, Oana
  Frunza, Wei Deng, Anderson Schneider, Yuriy Nevmyvaka, Yang~Katie Zhao,
  Alfredo Garcia, and Mingyi Hong. 2025{\natexlab{a}}.
\newblock Reinforcing multi-turn reasoning in {LLM} agents via fine-grained
  reward structure and credit assignment.
\newblock \emph{arXiv preprint arXiv:2505.11821}.

\bibitem[{Wei et~al.(2025{\natexlab{b}})Wei, Yao, Liu, Zhang, Lu, Qiu, Yu, Xu,
  Zhang, Yin, Yun, and Li}]{wei2025webagent}
Zhepei Wei, Wenlin Yao, Yao Liu, Weizhi Zhang, Qin Lu, Liang Qiu, Changlong Yu,
  Puyang Xu, Chao Zhang, Bing Yin, Hyokun Yun, and Lihong Li.
  2025{\natexlab{b}}.
\newblock \href {https://doi.org/10.18653/v1/2025.emnlp-main.401}
  {{W}eb{A}gent-{R1}: Training web agents via end-to-end multi-turn
  reinforcement learning}.
\newblock In \emph{Proceedings of the 2025 Conference on Empirical Methods in
  Natural Language Processing}, pages 7909--7928. Association for Computational
  Linguistics.

\bibitem[{Williams(1992)}]{williams1992reinforce}
Ronald~J. Williams. 1992.
\newblock Simple statistical gradient-following algorithms for connectionist
  reinforcement learning.
\newblock \emph{Machine Learning}, 8(3):229--256.

\bibitem[{Xu et~al.(2025)Xu, Soria, Tan, Roy, Agrawal, Poovendran, and
  Panda}]{xu2025toucan}
Zhangchen Xu, Adriana~Meza Soria, Shawn Tan, Anurag Roy, Ashish~Sunil Agrawal,
  Radha Poovendran, and Rameswar Panda. 2025.
\newblock {TOUCAN}: Synthesizing 1.5m tool-agentic data from real-world {MCP}
  environments.
\newblock \emph{arXiv preprint arXiv:2510.01179}.

\bibitem[{Yao et~al.(2023)Yao, Zhao, Yu, Du, Shafran, Narasimhan, and
  Cao}]{yao2023react}
Shunyu Yao, Jeffrey Zhao, Dian Yu, Nan Du, Izhak Shafran, Karthik~R Narasimhan,
  and Yuan Cao. 2023.
\newblock {ReAct}: Synergizing reasoning and acting in language models.
\newblock In \emph{The Eleventh International Conference on Learning
  Representations}.

\bibitem[{Yu et~al.(2025)Yu, Zhang, Zhu, Yuan, Zuo, Yue, Dai, Fan, Liu, Liu,
  Liu, Liu, Lin, Lin, Ma, Sheng, Tong, Zhang, Zhang, Zhang, Zhang, Zhu, Zhu,
  Chen, Chen, Wang, Yu, Song, Wei, Zhou, Liu, Ma, Zhang, Yan, Wu, and
  Wang}]{yu2025dapo}
Qiying Yu, Zheng Zhang, Ruofei Zhu, Yufeng Yuan, Xiaochen Zuo, Yu~Yue, Weinan
  Dai, Tiantian Fan, Gaohong Liu, Juncai Liu, LingJun Liu, Xin Liu, Haibin Lin,
  Zhiqi Lin, Bole Ma, Guangming Sheng, Yuxuan Tong, Chi Zhang, Mofan Zhang, and
  17 others. 2025.
\newblock \href {https://doi.org/10.52202/085713-3775} {{DAPO}: An open-source
  {LLM} reinforcement learning system at scale}.
\newblock In \emph{The Thirty-ninth Annual Conference on Neural Information
  Processing Systems}, volume~38, pages 125532--125554.

\bibitem[{Yu et~al.(2020)Yu, Kumar, Gupta, Levine, Hausman, and
  Finn}]{yu2020pcgrad}
Tianhe Yu, Saurabh Kumar, Abhishek Gupta, Sergey Levine, Karol Hausman, and
  Chelsea Finn. 2020.
\newblock Gradient surgery for multi-task learning.
\newblock In \emph{Advances in Neural Information Processing Systems},
  volume~33, pages 5824--5836. Curran Associates, Inc.

\end{thebibliography}
\newpage
\appendix

\section{Additional Formalism and Proofs}
\label{app:proofs}

\begin{figure}[H]
    \centering
    \includegraphics[width=0.95\linewidth,height=5.0cm,keepaspectratio]{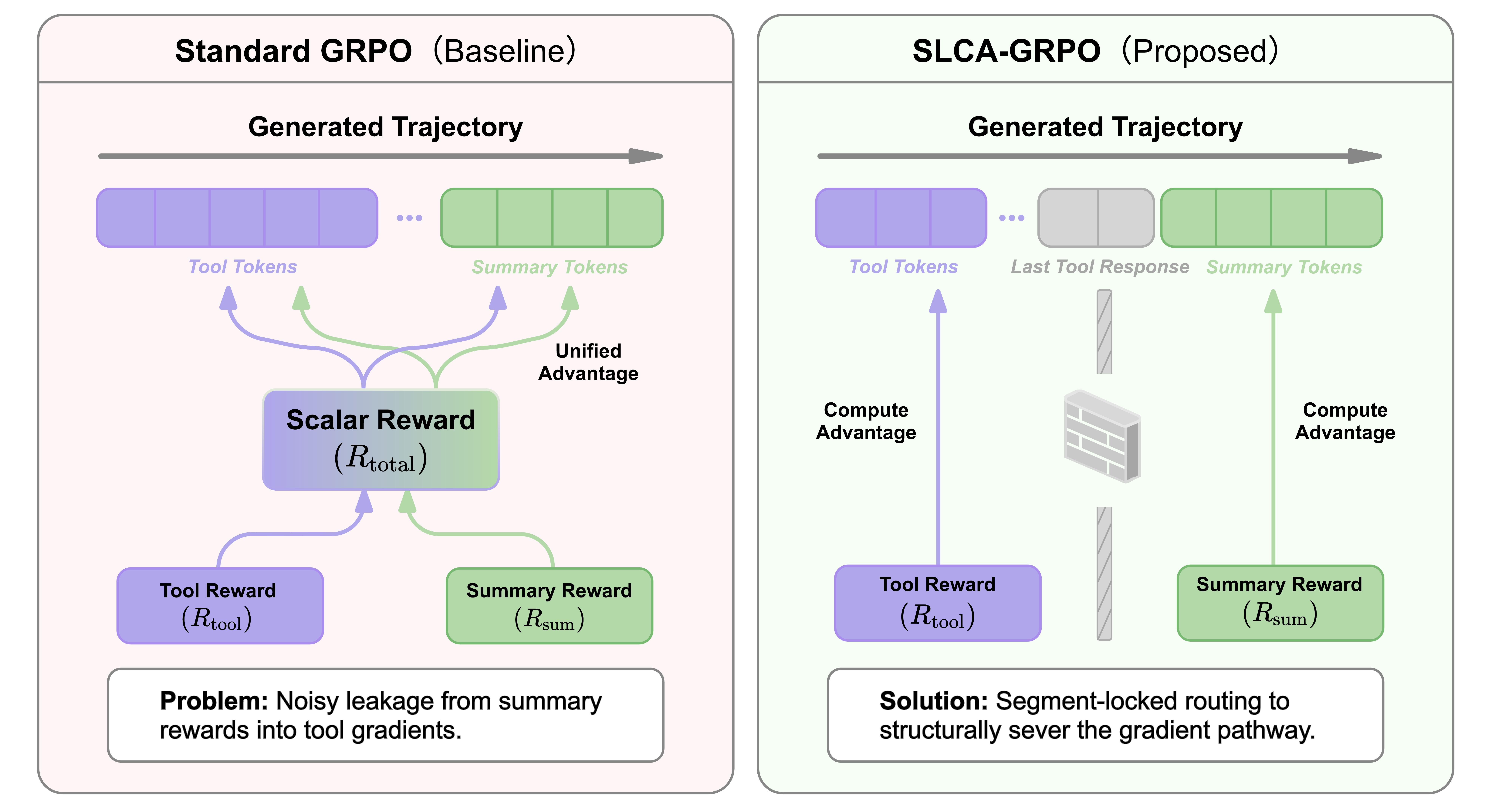}
    \caption{Mechanism of Cross-Segment Credit Misattribution and segment-locked routing. Unified advantage broadcast leaks summary rewards into tool tokens, whereas SLCA blocks this path by routing segment-wise advantages only to their corresponding tokens.}
    \label{fig:credit_misattribution}
\end{figure}

\subsection{Mask-Based Segment Decomposition}
\label{app:segmentation}

We derive segment boundaries using only the standard token-level mask $m_{i,t}\in\{0,1\}$ from the agent loop.
Here $m_{i,t}=1$ denotes a policy-generated (learnable) token and $m_{i,t}=0$ denotes an environment-injected token.
Let $T_i$ be the unpadded length of rollout $i$.

\paragraph{Runs of ones}
Consider the maximal contiguous runs where $m_{i,t}=1$.
Let $\{(s_{i,k},e_{i,k})\}_{k=1}^{K_i}$ denote the start and end indices (left-closed, right-open) of these runs so that
\begin{equation}
m_{i,t}=1 \iff t \in \bigcup_{k=1}^{K_i}[s_{i,k},e_{i,k}),
\end{equation}
with disjoint intervals ordered by time.

\paragraph{Segment definition}
We define the \textbf{Summary} segment as the last run of learnable tokens:
\begin{equation}
\textcolor{ArxivGreen}{\mathcal{T}_{i,\mathrm{sum}}} \triangleq [s_{i,K_i}, e_{i,K_i}),
\end{equation}
and the \textbf{Tool} segment as the union of all earlier learnable runs:
\begin{equation}
\textcolor{ArxivBlue}{\mathcal{T}_{i,\mathrm{tool}}} \triangleq \bigcup_{k=1}^{K_i-1}[s_{i,k}, e_{i,k}).
\end{equation}

\paragraph{Edge cases}
If $K_i=1$ (no environment injection), then $\textcolor{ArxivBlue}{\mathcal{T}_{i,\mathrm{tool}}}=\emptyset$ and the entire generation is treated as Summary.
If tool usage is required but the rollout contains no parseable tool call, we trigger the omission-penalty guard
(Appendix~\ref{app:slca_details}).

\subsection{Theoretical Analysis: Proofs of Main Properties}
\label{app:slca_theory}

This section provides the proofs and modeling assumptions behind Proposition~\ref{prop:slca_theory} (Per-Step Advantage Isolation), Theorem~\ref{thm:variance_main} (Nuisance-Variance Removal), and Corollary~\ref{thm:fidelity_main} (Directional Fidelity), which are summarized in \S\ref{subsec:slca}. The variance and direction results analyze a local score-function estimator that isolates the noise channel induced by unified reward broadcasting; exact GRPO group normalization introduces cross-sample dependencies, which we discuss explicitly below.

\paragraph{Setup and Definitions}

Let a trajectory $y$ be partitioned into a tool segment $\toolsub{y}$ and a summary segment $\summarysub{y}$.
The total reward is $R = \toolsup{R} + \summarysup{R}$.
Let $u_t = 
\nabla_\theta \log \pi_\theta(y_t \mid h_t)$ be the score function.
For the tool-token component, the accumulated score for trajectory $i$ is $\textcolor{ArxivBlue}{U_{i,\mathrm{tool}}} = \sum_{t \in \textcolor{ArxivBlue}{\mathcal{T}_{i,\mathrm{tool}}}} u_{i,t}$.
Let $w_i$ be the importance weight (e.g., length normalization).
Let $\toolsub{\mathcal{F}}$ be the $\sigma$-algebra generated by $\toolsub{y}$.

\begin{proposition}[Per-Step Advantage Isolation]
\label{prop:slca_theory}
Under SLCA routing (Eq.~\ref{eq:slca_routing}), the tool-token component of the gradient estimator has no functional dependence on $\summarysup{R}$.
\end{proposition}

\begin{proof}
The SLCA-routed tool-token component can be written as
\begin{equation}
    \toolsub{g}^{\text{SLCA}}
    =
    \sum_i w_i \toolsub{\lambda} \toolsup{\hat A_i}
    \sum_{t\in\textcolor{ArxivBlue}{\mathcal{T}_{i,\mathrm{tool}}}} u_{i,t}.
\end{equation}
The routed advantage $\toolsup{\hat A_i}$ is computed only from the tool rewards $\{\toolsup{R_j}\}$ within the rollout group, while summary tokens are outside the summation. The tool-token gradient component $\toolsub{g}^{\text{SLCA}}$ has no functional dependence on $\summarysup{R}$, and $\partial \toolsub{g}^{\text{SLCA}}/\partial \summarysup{R}=\mathbf{0}$. Shared parameters may still transfer information across future optimization steps; the proposition is a per-update statement about the routed gradient estimator.
\end{proof}

\paragraph{Part I: Conditional Nuisance-Variance Removal}

We model the summary reward $\summarysup{R}$ as containing a stochastic noise term $\xi$ (e.g., phrasing randomness) that is independent of the tool execution logic.

\begin{assumption}[Nuisance Noise Decomposition]
\label{assump:noise_main}
Conditioned on the tool trajectory $\toolsub{y}$, the summary reward decomposes into a deterministic expectation and a zero-mean noise term:
\begin{equation}
    \summarysup{R} = \summarysup{\bar{R}}(\toolsub{y}) + \xi, \quad \text{where } \mathbb{E}[\xi \mid \toolsub{y}] = 0, \ \mathbb{V}[\xi \mid \toolsub{y}] = \sigma_\xi^2 > 0.
\end{equation}
\end{assumption}

\begin{theorem}[Nuisance-Variance Removal: restated from Theorem~\ref{thm:variance_main}]
\label{thm:variance_reduction}
\label{thm:variance_main}
Let $g_{\text{STD}}$ and $g_{\text{SLCA}}$ be local score-function gradient estimators for tool tokens under unified reward broadcasting and SLCA, respectively.
In this local analysis, SLCA removes the irreducible conditional variance term induced by summary-reward noise:
\begin{equation}
    \mathbb{E}\!\left[\mathbb{V}(g_{\text{STD}}\mid \toolsub{\mathcal{F}})\right]
    -
    \mathbb{E}\!\left[\mathbb{V}(g_{\text{SLCA}}\mid \toolsub{\mathcal{F}})\right]
    =
    \mathbb{E}\left[ w_i^2 \| \toolsub{U} \|^2 \cdot \sigma_\xi^2 \right] > 0.
\end{equation}
\end{theorem}

\begin{proof}
We analyze the estimators for a single trajectory $i$ after fixing the local normalization scale of the update:
\begin{align}
    g_{\text{STD}} &= w_i (\toolsup{R} + \summarysup{\bar{R}} + \xi) \toolsub{U}, \\
    g_{\text{SLCA}} &= w_i \toolsup{R} \toolsub{U}.
\end{align}
This is the local score-function form of the GRPO update after suppressing group mean and standard-deviation terms. Exact GRPO uses group-normalized rewards, so $\mu_g$ and $\sigma_g$ are functions of all samples in the group and create cross-sample dependencies. This theorem does not give an exact closed form for the fully normalized GRPO covariance; it isolates the summary-noise channel that unified broadcasting necessarily exposes to tool-token gradients and SLCA removes.
Conditioned on $\toolsub{\mathcal{F}}$, the terms $w_i, \toolsup{R}, \summarysup{\bar{R}}, \toolsub{U}$ are deterministic. The only random variable is $\xi$.
For vector-valued gradients, $\mathbb{V}[\cdot]$ denotes the trace covariance, equivalently the expected squared $L_2$ deviation from the conditional mean. We apply the Law of Total Variance: $\mathbb{V}[X] = \mathbb{E}[\mathbb{V}[X|\mathcal{F}]] + \mathbb{V}[\mathbb{E}[X|\mathcal{F}]]$.

\textbf{1. Conditional variance of the standard estimator:}
\begin{itemize}
    \item \textit{Conditional Variance:} $\mathbb{V}[g_{\text{STD}} \mid \toolsub{\mathcal{F}}] = \mathbb{V}[w_i \xi \toolsub{U} \mid \toolsub{\mathcal{F}}] = w_i^2 \|\toolsub{U}\|^2 \sigma_\xi^2$.
    \item \textit{Conditional Expectation:} $\mathbb{E}[g_{\text{STD}} \mid \toolsub{\mathcal{F}}] = w_i (\toolsup{R} + \summarysup{\bar{R}}) \toolsub{U}$.
\end{itemize}
Thus, by the law of total variance,
\begin{equation}
    \mathbb{V}[g_{\text{STD}}] = \mathbb{E}[w_i^2 \|\toolsub{U}\|^2 \sigma_\xi^2] + \mathbb{V}[w_i (\toolsup{R} + \summarysup{\bar{R}}) \toolsub{U}].
\end{equation}

\textbf{2. Conditional variance of the SLCA estimator:}
\begin{itemize}
\item \textit{Since $g_{\text{SLCA}}$ is $\toolsub{\mathcal{F}}$-measurable (independent of $\xi$):}
\end{itemize}
\begin{equation}
    \mathbb{V}[g_{\text{SLCA}} \mid \toolsub{\mathcal{F}}] = 0.
\end{equation}

\textbf{3. Comparison:}
Taking expectations over $\toolsub{\mathcal{F}}$ gives the stated identity:
\begin{equation}
    \mathbb{E}\!\left[\mathbb{V}(g_{\text{STD}}\mid \toolsub{\mathcal{F}})\right]
    -
    \mathbb{E}\!\left[\mathbb{V}(g_{\text{SLCA}}\mid \toolsub{\mathcal{F}})\right]
    =
    \mathbb{E}[w_i^2 \|\toolsub{U}\|^2 \sigma_\xi^2].
\end{equation}
Since generation noise $\sigma_\xi^2 > 0$ and gradient norm $\|\toolsub{U}\|^2 > 0$ almost everywhere, the removed nuisance-variance term is strictly positive. This theorem intentionally makes a conditional variance claim: the systematic term $\mathbb{V}[w_i (\toolsup{R} + \summarysup{\bar{R}}) \toolsub{U}]$ may differ from $\mathbb{V}[w_i \toolsup{R} \toolsub{U}]$ depending on task structure and covariance, which is why we frame SLCA as a bias--variance trade-off rather than an unconditional variance dominance statement.
\end{proof}

\paragraph{Optimization implication}
For SGD on non-convex $L$-smooth functions, standard convergence bounds scale with the gradient-variance term (e.g., $O(\sigma^2/\epsilon^4)$ to reach an $\epsilon$-stationary point). Theorem~\ref{thm:variance_reduction} shows that SLCA removes a positive conditional variance component from tool-token gradients; when this nuisance component dominates systematic covariance changes, the resulting smoother gradients should translate into faster and more stable optimization.

\paragraph{Part II: Resolution of Advantage Confusion (Directional Fidelity)}

We define a Conflict Regime where the tool segment receives a below-group-mean normalized advantage ($\toolsup{\hat{A}} < 0$) but the summary segment receives an above-group-mean normalized advantage ($\summarysup{\hat{A}} > 0$).

\begin{theorem}[Directional Fidelity: restated from Corollary~\ref{thm:fidelity_main}]
\label{thm:directional_fidelity}
\label{thm:fidelity_main}
Let $\mathbf{d}^* = -
\nabla \log \pi(\toolsub{y})$ be the ideal descent direction to penalize tool failure.
Let normalized scores be $\toolsup{\hat{A}} = -\alpha$ ($\alpha>0$, penalty) and $\summarysup{\hat{A}} = \beta$ ($\beta>0$, reward).
Let $\toolsub{\sigma}$ and $\summarysub{\sigma}$ denote the group standard deviations of the raw segment rewards.
The expectations below are taken over the stated rollout randomness, with the trajectory partition, group standard deviations, and $\alpha,\beta$ held fixed.
\begin{itemize}
    \item \textbf{SLCA Update:} $\mathbb{E}[g_{\text{SLCA}}] \propto -\alpha 
\nabla \log \pi$. Inner product with $\mathbf{d}^*$: $\alpha \|
\nabla\|^2 > 0$. \textbf{(Correct Direction)}
    \item \textbf{Standard GRPO Update:} The trajectory-level raw deviation is proportional to $-\alpha\toolsub{\sigma}+\beta\summarysub{\sigma}$, so $\mathbb{E}[g_{\text{STD}}] \propto (\beta\summarysub{\sigma}-\alpha\toolsub{\sigma})
\nabla \log \pi$. Inner product with $\mathbf{d}^*$ is $(\alpha\toolsub{\sigma}-\beta\summarysub{\sigma})\|
\nabla\|^2$.
\end{itemize}
If $\beta\summarysub{\sigma} > \alpha\toolsub{\sigma}$ (the raw summary deviation dominates the raw tool penalty), Standard GRPO updates in the wrong direction ($<0$), reinforcing the failure. The simpler condition $\beta>\alpha$ is recovered when the two raw reward components have equal group standard deviations. SLCA maintains fidelity with respect to the tool-execution signal.
\end{theorem}

A concrete instantiation of the conflict regime in Theorem~\ref{thm:directional_fidelity} is visualized in Appendix~\ref{app:case_study_viz}: under an imperfect tool trajectory ($\toolsup{R}{=}0.65$, below the group mean) paired with a perfect summary ($\summarysup{R}{=}1.0$, above the group mean), standard GRPO reinforces the inefficient policy, whereas SLCA routes the tool-level penalty into $\toolsup{\hat{A}}$ alone and recovers the correct single parallel invocation.

\paragraph{Bias--Variance Trade-off of the $\summarysup{\hat{A}}$ Estimator}
\label{par:asum_bias}

Within a rollout group, trajectories may diverge after the tool segment and produce distinct post-tool states. SLCA still normalizes $\summarysup{R}$ across the group rather than within subgroups that share an identical post-tool state, so $\summarysup{\hat{A}}$ is not a state-conditional value estimate and can conflate summary quality with post-tool state variation.

This cross-state normalization is not unique to SLCA: standard GRPO normalizes $R_{\text{total}}$ across the same group, and methods such as MT-GRPO~\citep{wei2025mtgrpo} and GRPO-Verif~\citep{wang2025grpoverif} also do not compute exact state-conditional advantages without a learned value function or tree-based rollouts. SLCA \emph{narrows} the scope of this bias:
\begin{itemize}[nosep]
    \item Under standard GRPO, both tool and summary tokens receive $A_{\text{total}}$, the z-score of $(\toolsup{R}+\summarysup{R})$, which includes cross-state variation in $\summarysup{R}$.
    \item Under SLCA, tool tokens receive $\toolsup{\hat{A}}$, the z-score of $\toolsup{R}$ alone, computed across rollouts from the same initial state (shared prompt $x$), without cross-state summary bias. Only summary tokens receive the biased $\summarysup{\hat{A}}$.
\end{itemize}

We frame this as a bias--variance trade-off. The unified estimator avoids grouping bias but exposes tool-token gradients to conditional nuisance variance from the summary-reward channel (Theorem~\ref{thm:variance_main}). SLCA accepts summary-side grouping bias but confines it to summary tokens and removes that direct variance channel. The support-control experiment in Table~\ref{tab:support_control_7b} provides a cleaner decomposition: closing the summary-to-tool $S\!\to\!T$ support path has a 5.51\,pp main effect on $\tau^2$-Bench, whereas the matched w/o SLCA comparison changes normalization and token support together.

Explicit subgroup partitioning, which would normalize $\summarysup{R}$ only within trajectories sharing the same tool outcome, reduces the effective group size to $G/C$ per subgroup, where $C$ is the number of distinct tool outcomes. For multi-tool-call settings with $K$ calls, the effective size decays as $G/C^K$, making explicit partitioning impractical without increasing $G$. This is a design trade-off; soft-grouping or hierarchical alternatives remain open.

\subsection{SLCA Practical Details}
\label{app:slca_details}

This section details the numerical stability mechanisms and edge-case handling strategies implemented in \slcaname.

\paragraph{Presence-Based Filtering (Numerical Stability)}
Tool usage behavior is naturally sparse and dynamic; some rollouts may lack a tool segment (i.e., $|\textcolor{ArxivBlue}{\mathcal{T}_{i,\mathrm{tool}}}|=0$) due to omission or early termination.
Naively normalizing rewards within a group where only one sample contains a tool segment leads to a variance collapse (denominator becomes 0 or unstable), or yields an unnormalized advantage equal to the raw score.
To resolve this, we strictly normalize each segment type using only the subset of samples where that segment is present.

Let $L_i^{s} \triangleq |\mathcal{T}_{i,s}|$ denote the length of segment $s \in \{\toolsegment, \summarysegment\}$ in rollout $i$. We define the valid indices set for group $g$ as:
\begin{equation}
    \mathcal{I}_g^{s} \triangleq \{j \mid g(j)=g \text{ and } L_j^{s} > 0\}.
\end{equation}
We compute the segment-wise advantage $\hat{A}_i^{s}$ using a conditional normalization scheme. With a small stability constant $\epsilon_{\text{norm}}$, the formulation is:
\begin{equation}
    \hat{A}_i^{s} =
    \begin{cases}
        0 & \text{if } L_i^{s}=0 \text{ or } |\mathcal{I}_{g(i)}^{s}| < 2, \\[8pt]
        \dfrac{R_i^{s} - \mu_{g(i)}^{s}}{\sigma_{g(i)}^{s} + \epsilon_{\text{norm}}} & \text{otherwise},
    \end{cases}
    \label{eq:present_filter_app}
\end{equation}
where the mean $\mu_g^{s}$ and standard deviation $\sigma_g^{s}$ are computed exclusively over the valid set $\mathcal{I}_g^{s}$:
\begin{equation}
    \mu_g^{s} = \frac{1}{|\mathcal{I}_g^{s}|} \sum_{j \in \mathcal{I}_g^{s}} R_j^{s}, \quad
    \sigma_g^{s} = \operatorname{Std}\left(\{R_j^{s}\}_{j \in \mathcal{I}_g^{s}}\right).
\end{equation}
Here $\epsilon_{\text{norm}}$ denotes a numerical floor for stable division, not a tuned reward or optimization hyperparameter. When the effective group size $|\mathcal{I}_g^{s}| < 2$, we explicitly zero out the advantage ($\hat{A}_i^{s} = 0$) to disable updates, rather than falling back to standard statistics (e.g., $\mu=0, \sigma=1$), which would amplify raw scores as advantages.

\paragraph{Post-Normalization Weighting}
The segment weights $\toolsub{\lambda}$ and $\summarysub{\lambda}$ must be applied \emph{after} the normalization step.
Applying linear scaling \emph{before} z-score normalization cancels the scale in the idealized case without a numerical floor:
\begin{equation}
    \operatorname{Norm}(\lambda \cdot X) = \frac{\lambda X - \operatorname{Mean}(\lambda X)}{\operatorname{Std}(\lambda X)} = \frac{\lambda(X - \mu)}{\lambda \sigma} = \operatorname{Norm}(X).
\end{equation}
With the fixed $\epsilon_{\text{norm}}$ used in our implementation, this identity holds up to the numerical floor. We construct the final routed advantage $\hat{A}^{\text{SLCA}}_{i,t}$ by multiplying the normalized scalar $\hat{A}_i^{s}$ by $\lambda_s$ during the token-level routing phase (as shown in Eq.~\ref{eq:slca_routing}).

\paragraph{Anti-Hacking Omission Penalty Guard}
We strictly route the omission penalty $R_{\text{pen}} < 0$ to the Summary Segment:
\begin{equation}
    \toolsup{R_i} = 0, \qquad \summarysup{R_i} = R_{\text{pen}}.
\end{equation}
In our implementation, $R_{\text{pen}}=-0.5$. When the ground truth requires tool use but the rollout contains no \texttt{<tool\_call>} opening tag, this value directly \emph{overrides} the weighted trajectory score rather than being subtracted as an additive penalty; the normalized success score is then clipped to zero. By this design, the negative signal propagates only through the summary generation tokens. This effectively penalizes the decision to "answer immediately without calling," correctly aligning the gradient direction to suppress premature summarization.
In a representative SLCA trace, no-call episodes start at roughly 10--15\% and approach zero by update~50. In the uniform sensitivity setting, the parallelism weight is $S_{\mathrm{par}}=0.20$; its BFCL and $\tau^2$-Bench results remain above matched GRPO (App.~\ref{app:hierr_sensitivity}).

\subsection{HierR: Process Reward Definitions and Edge Cases}
\label{app:hierr_reward_details}

HierR provides dense, structure-preserving shaping for the tool segment and a terminal preference score for the summary.
We define a process score as a weighted sum of five components:
\begin{equation}
S_{\text{process}} =
0.10\, S_{\text{fmt}} +
0.25\, S_{\text{name}} +
0.15\, S_{\text{key}} +
0.20\, S_{\text{val}} +
0.30\, S_{\text{par}}.
\label{eq:hierr_process}
\end{equation}
Let $\mathcal{Y}$ be the predicted tool-call multiset and $\mathcal{G}$ be the ground-truth multiset.

\paragraph{(1) Format adherence ($S_{\text{fmt}}$)}
We measure syntactic validity and schema-parseability:
\begin{equation}
S_{\text{fmt}}=\frac{N_{\text{valid}}}{\max(N_{\text{open}},N_{\text{close}})}\in[0,1],
\end{equation}
where $N_{\text{valid}}$ is the number of call blocks that pass JSON parsing and schema validation, and
$N_{\text{open}},N_{\text{close}}$ count \texttt{<tool\_call>} and \texttt{</tool\_call>} tags.

\paragraph{(2) Tool-name match ($S_{\text{name}}$)}
We compute an F1-style score over tool-name multisets $\mathcal{N}_Y,\mathcal{N}_G$:
\begin{equation}
S_{\text{name}}=\frac{2|\mathcal{N}_Y\cap \mathcal{N}_G|}{|\mathcal{N}_Y|+|\mathcal{N}_G|}\in[0,1],
\end{equation}
where $\cap$ denotes multiset intersection.

\paragraph{(3) Argument-key match ($S_{\text{key}}$) and value match ($S_{\text{val}}$)}
For each gold call $g\in\mathcal{G}$, match it to the best predicted call $y^*\in\mathcal{Y}$ with the same tool name.
Let $\mathcal{K}_y,\mathcal{K}_g$ be argument-key sets. We define key match as Jaccard:
\begin{equation}
S_{\text{key}}(y,g)=\frac{|\mathcal{K}_y\cap \mathcal{K}_g|}{|\mathcal{K}_y\cup \mathcal{K}_g|}.
\end{equation}
Value match averages a loose indicator over matched keys:
\begin{equation}
S_{\text{val}}(y,g)=\frac{1}{|\mathcal{K}_g|}\sum_{k\in\mathcal{K}_y\cap \mathcal{K}_g}
\mathbb{I}_{\text{loose}}\big(y[k],g[k]\big),
\end{equation}
where $\mathbb{I}_{\text{loose}}$ allows simple type casting and string normalization.
We aggregate over all $g$ by averaging the best-match scores. Extra predicted calls are not directly averaged into $S_{\text{key}}$ or $S_{\text{val}}$ once all gold calls have been matched; they are penalized through $S_{\text{name}}$, $S_{\text{par}}$, and, when invalid, $S_{\text{fmt}}$.

\paragraph{(4) Parallelism constraint ($S_{\text{par}}$)}
To prevent redundancy or ``parallelism collapse'' at the initial tool step, we impose a hard cardinality check:
\begin{equation}
S_{\text{par}}=\mathbb{I}\big(|Y_{\text{init}}|=|G_{\text{init}}|\big)\in\{0,1\},
\label{eq:hierr_parallel}
\end{equation}
where $Y_{\text{init}}$ and $G_{\text{init}}$ denote the predicted and gold call lists in the first call step.

\subsection{Summary score and omission edge case}
\label{app:summary_score}
We derive the summary preference score $S_{\summaryword} \in [0, 1]$ using an LLM-as-a-judge approach \citep{wataoka2025selfpreference,li2025preferenceleakage}.
Specifically, we employ a structured evaluation prompt to assess the utility and faithfulness of the final response, normalizing the raw ratings to the unit interval.
We query the judge for every rollout, including cases in which mandatory tool use is omitted. For a rollout that triggers the omission guard, the policy reward is overridden by $\summarysub{R}=R_{\text{pen}}=-0.5$; the normalized guard advantage is retained for the policy update.

\paragraph{Judge Model and Deployment}
The judge is GPT-OSS-120B (OpenAI, Apache 2.0), a 120B-parameter open-weight language model served via a vLLM OpenAI-compatible API endpoint.
It is a frozen reward producer: its parameters are not optimized by policy training, although its scalar $\summarysub{R}$ enters the policy loss. The role can be served by a local model or an API endpoint and is separate from the response mocker used by SGLS.
The deployment and decoding parameters are listed in Table~\ref{tab:judge_deployment}.
The deterministic setting (\texttt{temperature}$=0.0$) makes the judge output deterministic for a fixed input; rollout and summary sampling can still vary across runs.

\begin{table}[H]
\centering
\small
\caption{Summary Judge: Deployment and Decoding Parameters.}
\label{tab:judge_deployment}
\begin{tabular}{lc}
\toprule
\rowcolor{ArxivTableHead}
\textbf{Parameter} & \textbf{Value} \\
\midrule
\multicolumn{2}{l}{\textit{Model Specification}} \\
Model & GPT-OSS-120B \\
Parameters & 120B \\
Serving Framework & vLLM (OpenAI-compatible API) \\
\midrule
\multicolumn{2}{l}{\textit{Decoding}} \\
Temperature & 0.0 (fully greedy / deterministic) \\
Max Tokens & 8{,}192 \\
top\_p & 1.0 (vLLM default; no nucleus filtering) \\
top\_k & 0 (disabled; no top-$k$ filtering) \\
\midrule
\multicolumn{2}{l}{\textit{Reliability}} \\
Timeout & 540s (including queue wait) \\
Max Retries & 3 (exponential backoff) \\
\bottomrule
\end{tabular}
\end{table}

\paragraph{Prompt Structure and Input Variables}
The full evaluation prompt is provided in the Prompts section.
Three variables are injected into the template:
\texttt{\{question\_content\}} (the user's original question),
\texttt{\{tool\_responses\}} (tool return values only, excluding \texttt{<tool\_call>} content), and
\texttt{\{final\_response\}} (the model's final summary).
The judge evaluates along three dimensions (\emph{Result Interpretation}, \emph{Information Completeness}, and \emph{Expression Quality}) and produces a categorical rating on a 5-point scale.

\paragraph{Factual vs.\ Stylistic Dimensions}
The three evaluation dimensions differ in nature: \emph{Result Interpretation} and \emph{Information Completeness} are factual; the latter is causally dependent on tool execution quality, since incomplete tool results necessarily limit response completeness. \emph{Expression Quality} is partly stylistic (clarity, organization). SLCA's decoupling guarantee (Proposition~\ref{prop:slca_theory}) holds regardless of what $\summarysup{R}$ measures: it prevents any component of $S_{\summaryword}$, whether factual or stylistic, from contaminating tool-token advantages. The legitimate causal link (correct tools $\to$ better summaries) is preserved via the systematic component $\summarysup{\bar{R}}(\toolsub{y})$ in Assumption~\ref{assump:noise_main}; the direct summary-to-tool pathway is blocked. BFCL and $\tau^2$-Bench use judge-free evaluation metrics, providing checks beyond the judge-based metric.

\paragraph{Score Normalization}
The categorical rating is mapped to $[0,1]$ via linear normalization $(r - 1)/4$, where $r$ is the raw integer score.
Table~\ref{tab:judge_scoring} lists the complete mapping.

\begin{table}[H]
\centering
\small
\caption{Summary Judge: Rating-to-Score Mapping.}
\label{tab:judge_scoring}
\begin{tabular}{lcc}
\toprule
\rowcolor{ArxivTableHead}
\textbf{Rating} & \textbf{Raw Score ($r$)} & \textbf{Normalized $S_{\summaryword}$} \\
\midrule
very poor   & 1 & 0.00 \\
poor        & 2 & 0.25 \\
acceptable  & 3 & 0.50 \\
good        & 4 & 0.75 \\
excellent   & 5 & 1.00 \\
\bottomrule
\end{tabular}
\end{table}

\paragraph{Output Parsing and Failure Handling}
The judge output is parsed via the regex pattern:
\begin{center}
\small\texttt{<response\_quality>.*?<rating>(.*?)</rating>.*?</response\_quality>}
\end{center}
Only the five predefined categorical values listed in Table~\ref{tab:judge_scoring} are accepted.
Any other output (e.g., numeric scores, free-text descriptions, or malformed XML) is treated as a parse failure and assigned $\texttt{JUDGE\_FAILURE\_SCORE} = 0.0$.
This conservative default prevents noisy or malformed judge outputs from injecting spurious reward signals into training.

\paragraph{Omission Guard}
The guard overrides the policy reward when a required tool call is absent, so a summary cannot receive a positive policy reward for answering without tool support. The judge call is still made for the rollout; its output does not replace the omission reward. This keeps the reward protocol and group normalization statistics fixed across the compared conditions.

\paragraph{Safeguards Against Judge Artifacts}
We employ the following safeguards to mitigate known failure modes of LLM-as-a-judge:
\begin{enumerate}
    \item \textbf{Verbosity/style bias:} The structured rubric evaluates three factual dimensions rather than holistic preference, constraining the judge to content-level assessment. The judge receives only tool responses and the final summary (no intermediate reasoning, chain-of-thought, or \texttt{<tool\_call>} content), limiting stylistic influence. The discrete 5-point categorical scale further constrains output granularity.
    \item \textbf{Reward hacking:} Two safeguards are used: (a)~the omission guard overrides the policy reward with $R_{\text{pen}}$ and routes its normalized signal to the available Summary tokens when mandatory tool use is absent; (b)~SLCA's advantage routing (Eq.~\ref{eq:slca_routing}) isolates $S_{\summaryword}$ from tool-token gradients.
    \item \textbf{Preference leakage:} The same judge model, prompt template, and decoding parameters are used identically across all compared methods. The judge weights are frozen throughout training. For cross-domain benchmarks, $S_{\summaryword}$ is not used: BFCL evaluates tool accuracy directly, and $\tau^2$-Bench uses the official $\text{Pass}^1$ metric.
\end{enumerate}

\subsection{SGLS Scope and Evaluation}
\label{app:sgls_scope}

SGLS is a schema-guided simulator: it preserves tool names, required fields, and response structure while allowing concrete observations to differ from those of a live API. We evaluate this structural alignment empirically through the \textit{w/o SGLS} ablation, the cross-domain benchmarks, and the alignment examples in App.~\ref{app:sgls_impl}.

\paragraph{Scope of the semantic argument}
Raw simulated and real observations can have disjoint supports: random hashes, timestamps, and instance IDs can make $D_{\mathrm{TV}}(P_{\mathrm{real}},P_{\mathrm{sim}})=1$, so a standard simulation lemma on raw observations is vacuous. The useful alignment is at the schema and semantic-interaction level: tool names, required fields, call structure, and response roles remain aligned even when concrete values differ. A Wasserstein-style metric on a semantic representation would be more meaningful than total variation, but we do not define such a representation or claim a bound here. SGLS treats the schema as a control plane, and its contribution is assessed through the \textit{w/o SGLS} ablation, the cross-domain benchmarks, and the alignment examples reported in the rest of this appendix.

\subsection{Algorithm Details}
\label{app:full_algo}

Algorithm~\ref{alg:slca_grpo} summarizes the complete \slcaname training procedure with SGLS rollouts and HierR segment rewards.

\begin{algorithm}[ht]
\caption{\textsc{\slcaname} training with SGLS and HierR}
\label{alg:slca_grpo}
\begin{algorithmic}[1]
\STATE \textbf{Input:} dataset $\mathcal{D}$, schemas $\Sigma(x)$, group size $G$, PPO clip $\epsilon_{\text{clip}}$
\FOR{each iteration}
\STATE sample $x\sim\mathcal{D}$
\FOR{$i=1$ to $G$}
\STATE rollout $y_i \sim \pi_{\theta_{\mathrm{old}}}(\cdot\mid x)$ by interacting with SGLS (schema validation + tool responses)
\STATE record mask $m_{i,t}$ and segment sets $\textcolor{ArxivBlue}{\mathcal{T}_{i,\mathrm{tool}}},\textcolor{ArxivGreen}{\mathcal{T}_{i,\mathrm{sum}}}$
\STATE compute HierR segment returns $(\toolsup{S_i},\summarysup{S_i})$ with omission-penalty guard if needed
\ENDFOR
\STATE compute $\toolsup{\hat A_i}=\mathrm{Norm}(\toolsup{S_i})$ and $\summarysup{\hat A_i}=\mathrm{Norm}(\summarysup{S_i})$
with presence-based filtering
\STATE route advantages to tokens via Eq.~\ref{eq:slca_routing} to obtain $\hat A^{\text{SLCA}}_{i,t}$
\STATE update $\theta$ by maximizing Eq.~\ref{eq:slca_grpo_obj}
\ENDFOR
\end{algorithmic}
\end{algorithm}

\section{Empirical Analysis of Optimization Stability}
\label{app:gradient_analysis}

As an optimization diagnostic, we track the $L_2$ norm of the gradient vector $\|
\nabla_\theta\|_2$ throughout training across all model scales. A lower and smoother gradient norm is one descriptive indicator of update variability.

\paragraph{Visual Analysis.}
Figure~\ref{fig:grad_norm_dynamics} illustrates the training dynamics.
\begin{itemize}
    \item \textbf{SFT+GRPO (w/o SLCA) (Orange):} Exhibits frequent high-magnitude spikes and broader fluctuations. These patterns are compatible with updates in which summary and tool signals point to different changes.
    \item \textbf{\slcaname (Blue):} Maintains a generally smoother trajectory in these runs. Its segment-specific support removes the direct summary-to-tool contribution to tool-token updates.
\end{itemize}

\begin{figure}[H]
    \centering
    \includegraphics[width=1.0\textwidth]{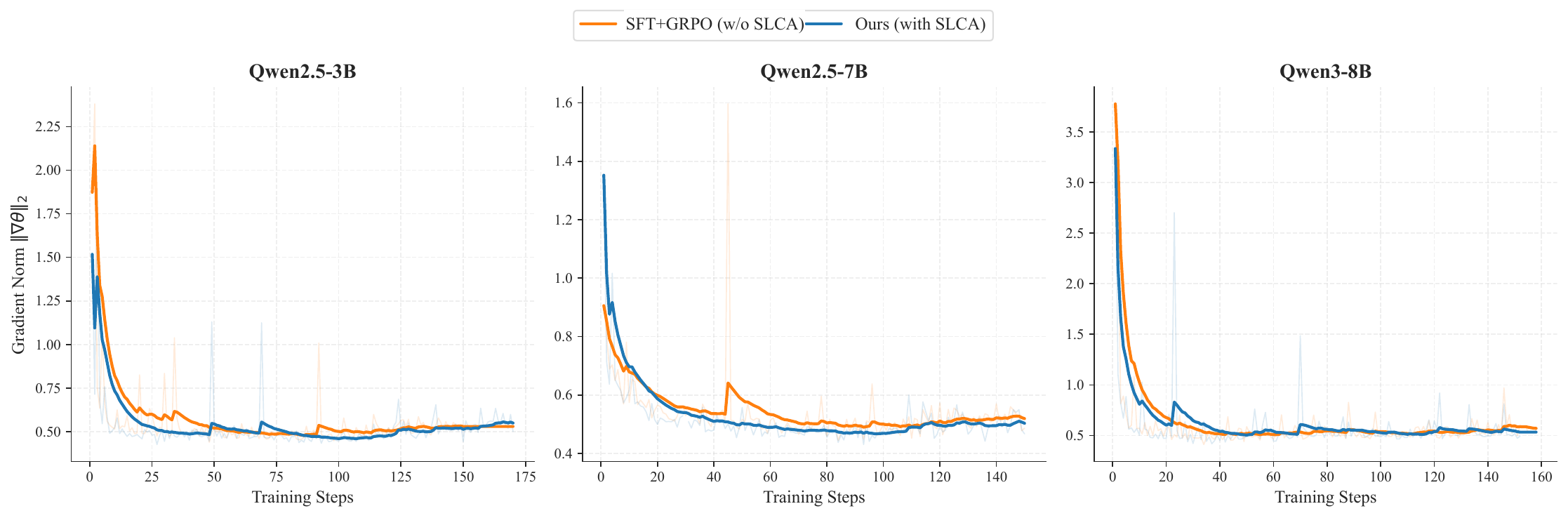}
    \vspace{-0.2cm}
    \caption{
    Gradient Norm Dynamics ($\|
\nabla_\theta\|_2$) during Training.
    We compare the gradient norms of SFT+GRPO (w/o SLCA) (Orange) and \slcaname (Blue) across Qwen2.5-3B, 7B, and Qwen3-8B. These are representative single-run traces and provide a descriptive view of update variability under the two estimators.
    }
    \label{fig:grad_norm_dynamics}
\end{figure}

\paragraph{Quantitative Gradient Stability}
We quantify this diagnostic by computing the standard deviation (Std) of the gradient norm over the training steps. As reported in Table~\ref{tab:grad_std}, \textsc{\slcaname} has a lower observed standard deviation across all model scales in these runs.

On \textbf{Qwen2.5-3B-Instruct}, gradient-norm volatility is reduced by \textbf{25.0\%}. On \textbf{Qwen3-8B-Base}, which has the highest baseline Std in this table, the reduction is 2.5\%.

\paragraph{Cross-Scale Analysis: Gradient Stability vs.\ Task Performance}
The relative reduction decreases across the three backbones (25.0\% $\to$ 13.1\% $\to$ 2.5\%). Gradient-norm standard deviation captures one aggregate variability measure, so these values do not identify the contribution of each noise source.

The matched $\tau^2$-Bench gaps are +1.09\,pp (3B), +9.15\,pp (7B), and +10.03\,pp (8B), all based on three-run means. These aggregate results do not establish a scale trend for gradient stability.
On the simpler Toucan single-turn evaluation, the three-run Success gaps range from +2.05 to +2.53\,pp. Separately, we measured the gradient cosine similarity $\cos(
\textcolor{ArxivBlue}{\nabla_{\mathrm{tool}}}, 
\textcolor{ArxivGreen}{\nabla_{\mathrm{sum}}})$ on the 3B backbone and found it near zero ($[-0.03, 0.08]$) throughout training for \emph{both} SLCA and standard GRPO. In this diagnostic, advantage magnitudes, rather than gradient direction, are the more visible source of cross-segment mismatch.
Because the measured tool/summary gradient cosine is near zero, we refer to SLCA as an ``advantage-level firewall'' rather than a ``gradient firewall'': the diagnostic concerns the magnitude of the assigned advantages, not a measured directional conflict between the two gradient components.

\begin{table}[H]
\centering
\small
\caption{
\textbf{Quantification of Gradient Stability.}
    We report the Standard Deviation (Std) of the gradient norm. \textsc{\slcaname} has lower observed volatility in these representative runs.
}
\label{tab:grad_std}
\setlength{\tabcolsep}{10pt} 
\begin{tabular}{lccc}
\toprule
\rowcolor{ArxivTableHead}
\textbf{Model Backbone} & \textbf{Ours Std} ($\sigma_{\text{SLCA}}$) & \textbf{Baseline Std} ($\sigma_{\text{Base}}$) & \textbf{Improvement} ($\downarrow$) \\
\midrule
\textbf{Qwen2.5-3B-Instruct} & \textbf{0.156} & 0.209 & \textbf{25.0\%} \\
\textbf{Qwen2.5-7B-Instruct} & \textbf{0.094} & 0.108 & \textbf{13.1\%} \\
\textbf{Qwen3-8B-Base}       & \textbf{0.312} & 0.320 & \textbf{2.5\%} \\
\bottomrule
\end{tabular}
\end{table}

\section{Datasets and Training Details}
\label{app:data_details}

\subsection{Data Processing Pipeline}
\label{app:data_pipeline}
As illustrated in Figure~\ref{fig:data_process}, this pipeline is designed to strictly prevent test-set leakage and optimize data composition for different learning stages.

The pipeline consists of four distinct phases:

\paragraph{Phase 1: Noise Filtering and Schema Validation}
Starting with 119,279 Toucan raw trajectories, we first remove approximately 40,000 samples identified as irrelevant or containing low-quality formatting. The remaining 79,279 candidates undergo strict filtering, which removes 1,038 samples containing broken JSON, non-standard tool tokens, or hallucinations to ensure simulator stability. This results in a high-quality valid pool of \textbf{78,241} trajectories.

\paragraph{Phase 2: Evaluation Isolation.}
To ensure rigorous zero-shot evaluation, we randomly sample a held-out Toucan-Test set of \textbf{4,000} samples from the valid pool before any training splits are made. This set comprises:
\begin{itemize}
    \item \textbf{3,000 Native Single-turn samples:} Testing atomic tool-use capabilities.
    \item \textbf{1,000 Decomposed Multi-turn samples:} Extracted from distinct multi-turn dialogues to test context retention and state tracking.
\end{itemize}
The remaining 74,241 samples constitute the Total Training Pool.

\paragraph{Phase 3: Trajectory Decomposition Strategy}
For Reinforcement Learning, training on long-horizon dialogs from scratch is unstable. We decompose multi-turn dialogs into atomic \textit{Context-Action} pairs.
We exclude the first tool-call turn during decomposition and extract samples starting from the second tool invocation (Turn $\ge$ 2). Each decomposed sample provided to the RL agent thus contains historical context from previous user-agent interactions.
The Toucan trajectories use real MCP teacher traces. The segment mask is derived deterministically from the structured trajectory trace and the environment boundary; no annotated segment or token labels and no learned segmenter are used.

\begin{figure*}[t]
    \centering
    \includegraphics[width=1.0\textwidth]{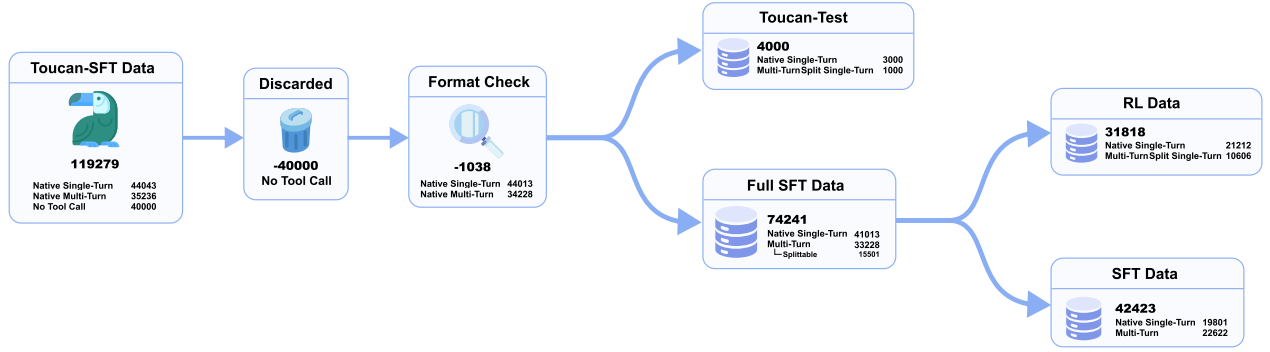}
    \caption{\textbf{Data Processing Pipeline.} The workflow applies strict relevance filtering and format check before isolating the evaluation set. The remaining training pool is split into SFT and RL partitions; the RL set combines native single-turn and decomposed multi-turn trajectories formatted for the SGLS.}
    \label{fig:data_process}
\end{figure*}

\paragraph{Phase 4: SFT vs. RL Partitioning}
We split the training pool (74,241 samples) into two partitions:
\begin{itemize}
    \item \textbf{SFT Partition (42,423 samples):} Representing approximately $57\%$ of the data, this set prioritizes diversity in format (Human/GPT styles) and includes the majority of native multi-turn data to establish a robust instruction-following prior.
    \item \textbf{RL Partition (31,818 samples):} Optimized for exploration, this set is constructed with a specific compositional ratio:
    \begin{equation*}
        \text{RL Composition} \approx \frac{2}{3} \times \text{Native Single-turn} + \frac{1}{3} \times \text{Decomposed Multi-turn}
    \end{equation*}
    Specifically, it contains \textbf{21,212} native single-turn samples and \textbf{10,606} decomposed multi-turn samples. This stratified mix balances the learning of simple function execution with complex, history-dependent reasoning.
\end{itemize}

\paragraph{Total Supervision Visibility}
As noted in the experimental setup, our SFT baseline is trained on the union of both partitions (Total 74,241 samples), so the comparison uses matched data exposure.

\subsection{Training Implementation Details}
\label{app:training_details}

\paragraph{Implementation Environment and Hyperparameters}
We fine-tune the model on the Toucan supervision trajectories using the standard cross-entropy objective. The SFT training is conducted on 8 NVIDIA H20 GPUs with \textit{bfloat16} precision. We employ the AdamW optimizer~\citep{kingma2015adam,loshchilov2019adamw} with a cosine learning rate scheduler.
For the RL stage, training is scaled across 32 NVIDIA H20 GPUs (4 nodes). The maximum sequence lengths for prompts and generations are extended to 16k and 8k tokens, respectively, to accommodate long-context reasoning. Detailed hyperparameters for both stages are listed in Table~\ref{tab:hyperparameters}.

\paragraph{Baseline Configurations and Controlled Variables}
The matched set comprises SFT+GRPO, SFT+\textsc{\slcaname}, w/o SLCA, w/o SGLS, and w/o HierR. Within each backbone, all five use that backbone's SFT checkpoint, the same data split, SGLS endpoint, mocker configuration, decoding settings, evaluation protocol, rollout group size $G=16$, one RL epoch, and the same run set. GRPO, \slcaname, w/o SLCA, and w/o SGLS use the same HierR definition and subweights; w/o HierR removes HierR, and w/o SGLS removes schema conditioning and deterministic validation:
\begin{itemize}
    \item \textbf{SFT (Data-Equivalent Baseline):} By exposing the model to the full dataset in a supervised manner, this baseline serves as a data-equivalent benchmark to determine whether RL optimization yields gains beyond simply scaling up demonstration data.
    \item \textbf{SFT+GRPO (Standard Baseline):} This baseline uses the same SGLS environment and HierR reward definitions as our method. Its estimator aggregates the two rewards into an unweighted sum ($R_{\text{total}} = 1.0 \cdot \toolsub{R} + 1.0 \cdot \summarysub{R}$) and broadcasts the resulting advantage to all learnable tokens.
    \item \textbf{SFT+\slcaname (Ours):} Following the same data split as SFT+GRPO, it computes independent advantages for execution and summarization. The matched comparison uses the same 1:1 weighting ($\toolsub{\lambda}=\summarysub{\lambda}=1$) and routes signals via Eq.~\ref{eq:slca_routing}.
    \item \textbf{Ablation definitions:} w/o SLCA uses the unified estimator, w/o SGLS removes schema conditioning and deterministic validation while retaining the mocker endpoint, and w/o HierR removes the hierarchical reward.
\end{itemize}

\begin{table}[H]
\centering
\small
\caption{Detailed Hyperparameters and Environment Settings. We list the configurations for both the supervised fine-tuning (SFT) and the subsequent RL training stages.}
\label{tab:hyperparameters}
\begin{tabular}{lclc}
\toprule
\rowcolor{ArxivTableHead}
\multicolumn{2}{c}{\textbf{Fine-Tuning Hyperparameters}} & \multicolumn{2}{c}{\textbf{RL Hyperparameters}} \\
\cmidrule{1-4}
\multicolumn{1}{l}{Parameter} & Value & \multicolumn{1}{l}{Parameter} & Value \\
\midrule
Hardware Resources & 8 $\times$ H20 GPUs & Hardware Resources & 32 $\times$ H20 (4 Nodes) \\
Optimizer & AdamW & Precision & bfloat16 \\
Total Batch Size & 32 & Total Batch Size & 128 \\
Per-Device Batch & 1 & Mini-batch Size & 32 \\
Grad Accumulation & 4 & Rollout Group Size ($G$) & 16 \\
Learning Rate & $5\times 10^{-6}$ & Learning Rate & $1\times 10^{-6}$ \\
Number of Epochs & 2 & Number of Epochs & 1 \\
LR Scheduler & Cosine & Max Response Len & 8192 \\
Warmup Ratio & 0.1 & KL Coefficient ($\beta_{\text{KL}}$) & 0.001 \\
Max Context Len & 16384 & Max Prompt Len & 16384 \\
\multicolumn{2}{c}{--} & PPO Clip ($\epsilon_{\text{clip}}$) & 0.2 (symmetric) \\
\multicolumn{2}{c}{--} & Dual-Clip Constant ($c$) & 3.0 \\
\multicolumn{2}{c}{--} & KL Estimator & low-var KL \\
\multicolumn{2}{c}{--} & Loss Aggregation & token-mean \\
\multicolumn{2}{c}{--} & No-Call Penalty ($R_{\text{pen}}$) & $-0.5$ (override) \\
\bottomrule
\end{tabular}%
\end{table}

The PPO clip is symmetric: $\epsilon_{\text{clip}}=\epsilon_{\text{low}}=\epsilon_{\text{high}}=0.2$, without asymmetric DAPO-style clipping. The dual-clip constant $c=3.0$ applies only to negative-advantage samples when the likelihood ratio is large. The no-call value $R_{\text{pen}}=-0.5$ is used when tool use is required but no tool call is emitted; it overrides the weighted score rather than subtracting from it.

\paragraph{Training-step convention}
All RL curves report policy-update steps, not PPO inner mini-batch passes. The main runs use rollout group size $G=16$ and one RL epoch over the post-filtered rollout stream; this corresponds to approximately 100 logged policy-update steps in the default 7B setting. We use this convention consistently across all training-dynamics figures.

\subsection{SGLS Implementation Details}
\label{app:sgls_impl}

\paragraph{Frozen LLM Specification}
The frozen LLM used in SGLS is \textbf{Qwen3-235B-A22B-Instruct} (July 2025 checkpoint), a 235B-parameter Mixture-of-Experts (MoE) model with 22B active parameters per forward pass.
This model belongs to a \emph{different architectural family} from the policy backbone (Qwen2.5-7B-Instruct, dense architecture, September 2024). The cross-family, cross-generational design limits implicit leakage and same-backbone favoritism: (i) the representation spaces are structurally distinct (MoE routing vs.\ dense layers), (ii) the models use different tokenizers and pre-training corpora, and (iii) the frozen LLM's role is limited to template-constrained mocking of valid tool calls; it never provides planning guidance or reward signals.
The Toucan teacher trajectories used to construct the training data are real MCP trajectories curated for tool coverage and quality. They provide the tool-reference source for $\toolsub{R}$ and are not an input to the SLCA estimator.

\paragraph{Deployment Configuration}
The frozen LLM is served via vLLM~0.11.0 with the deployment parameters listed in Table~\ref{tab:sgls_deployment}.

\begin{table}[H]
\centering
\small
\caption{SGLS Frozen LLM: Deployment and Decoding Parameters.}
\label{tab:sgls_deployment}
\begin{tabular}{lc}
\toprule
\rowcolor{ArxivTableHead}
\textbf{Parameter} & \textbf{Value} \\
\midrule
\multicolumn{2}{l}{\textit{vLLM Deployment}} \\
Tensor-Parallel Size (TP) & 8 \\
Max Model Length & 32{,}768 \\
GPU Memory Utilization & 0.86 \\
Max Concurrent Sequences & 300 \\
Precision (dtype) & bfloat16 \\
Prefix Caching & Enabled \\
Expert Parallel (MoE) & Enabled \\
\midrule
\multicolumn{2}{l}{\textit{Decoding (API Call Side)}} \\
Temperature & 0.6 \\
Max Tokens & 4{,}096 \\
top\_p & 1.0 (vLLM default; no nucleus filtering) \\
top\_k & 0 (disabled; no top-$k$ filtering) \\
Stop Tokens & Model EOS (vLLM built-in) \\
\bottomrule
\end{tabular}
\end{table}

\paragraph{Two-Stage Pipeline}
SGLS operates as a two-stage pipeline to minimize unnecessary LLM invocations:
\begin{enumerate}
    \item \textbf{Stage~1: Deterministic Schema Validation (no LLM):}
    Checks tool existence against the current instance's available tool set; if the tool name is not found, returns a hard-coded error string. Then verifies required parameters against the schema definition; missing parameters trigger a deterministic error message listing the absent fields.
    \item \textbf{Stage~2: Template-Constrained LLM Mocking (valid calls only):}
    For calls that pass Stage~1, the tool definition JSON and the parsed \texttt{<tool\_call>} payload are injected into a fixed prompt template. The frozen LLM generates a mock \texttt{<tool\_response>} output. Post-processing extracts the response content via tag matching and handles truncation edge cases (e.g., responses with opening but no closing tags).
\end{enumerate}
This hybrid architecture gives schema violations in early-training rollouts deterministic feedback, while the LLM is invoked only for semantically plausible mock responses to correctly formed calls.

\paragraph{\textit{w/o SGLS} fallback environment}
The \textit{w/o SGLS} ablation keeps the same frozen mocker and rollout interface, and removes only schema conditioning and deterministic validation.
Both variants call the same external vLLM \texttt{/v1/chat/completions} endpoint with the same Qwen3-235B-A22B model, temperature (0.6), max-token budget (4{,}096), and timeout (540s).
Thus the ablation is \emph{prompt-unconstrained}, not model-unconstrained: the API-server role instruction, \texttt{<tool\_response>} output contract, prohibitions against explanations or \texttt{<tool\_call>} emission, few-shot examples, incoming-request format, and rollout limits are unchanged.
The only prompt-side removal is the \texttt{Tool Definition} JSON-schema block; the corresponding code path also removes per-instance schema caching and tool-definition parsing.
As a result, unknown tools, missing required parameters, and type-invalid arguments no longer trigger deterministic error strings; every call is forwarded to the mocker, which generates a \texttt{<tool\_response>} shell whose content is determined by the LLM.
Post-processing only strips/extracts tool-response tags and does not validate JSON fields.

\begin{figure}[H]
    \centering
    \includegraphics[width=\linewidth]{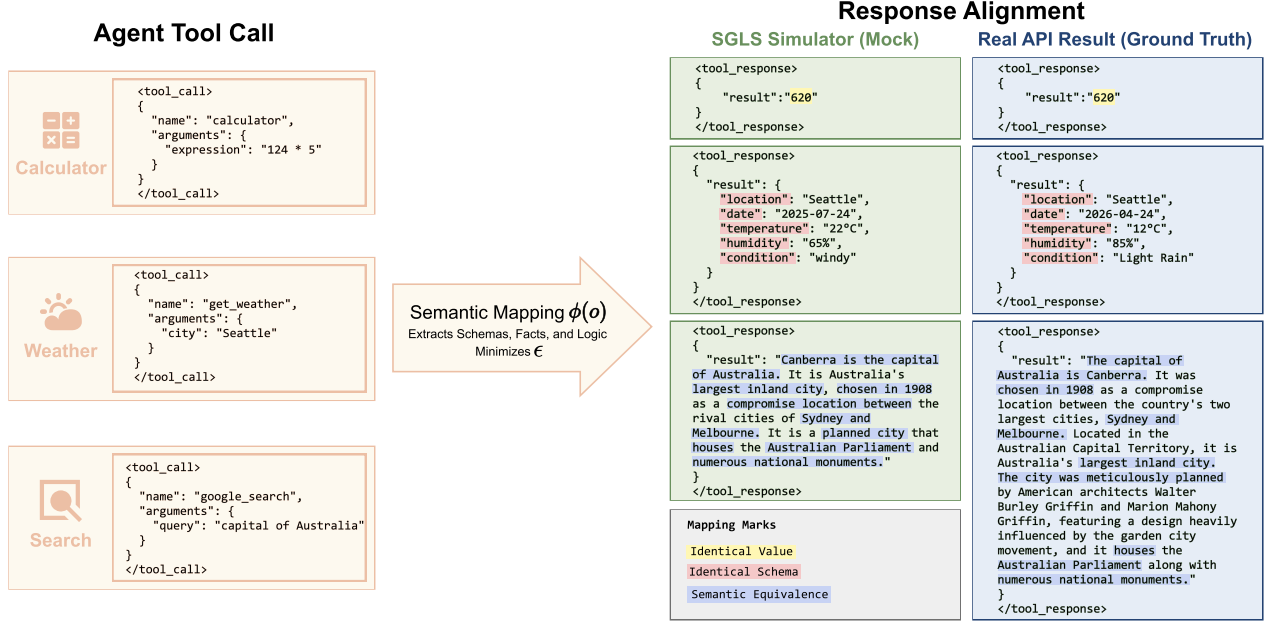}
    \caption{\textbf{SGLS mock responses align with real-API outputs across three representative tool domains.}
    The Semantic Mapping $\phi(o)$ extracts schemas, facts, and logic, realizing three levels of alignment (Mapping Marks):
    \emph{Identical Value} for deterministic tools (\emph{calculator}),
    \emph{Identical Schema} for query-style tools (\emph{weather}; individual values need not match a specific real-API snapshot but remain within plausible ranges),
    and \emph{Semantic Equivalence} for open-ended tools (\emph{search}; factually consistent content differing only in phrasing).}
    \label{fig:sgls_alignment}
\end{figure}

\paragraph{Response-mocker robustness}
An alternate response-mocker check replaces the frozen response mocker with GPT-OSS-120B while keeping the policy initialization, data, validators, rewards, rollout budget, and evaluation protocol fixed. The response mocker generates tool observations; it is separate from the summary judge that produces $\summarysub{R}$. In this check the mocker and judge are intentionally from the same GPT-OSS-120B family; BFCL and $\tau^2$-Bench do not use the judge reward, so their transfer measurements are not affected by that shared-family choice. The summary-judge prompt and decoding remain unchanged. Table~\ref{tab:mocker_swap} reports the matched results.

\begin{table}[H]
\centering
\small
\caption{Response-mocker check on Qwen2.5-7B-Instruct. GPT-OSS-120B replaces the response mocker; entries are mean $\pm$ standard deviation under the matched policy setup.}
\label{tab:mocker_swap}
\setlength{\tabcolsep}{4pt}
\begin{tabular}{lccc}
\toprule
\rowcolor{ArxivTableHead}
\textbf{Method} & \textbf{Toucan Succ. (\%)} & \textbf{BFCL Acc. (\%)} & \textbf{$\tau^2$-Bench Pass$^1$ (\%)} \\
\midrule
SFT+GRPO (w/o SLCA) & 75.74 $\pm$ 1.47 & 68.08 $\pm$ 0.36 & 32.74 $\pm$ 2.66 \\
\rowcolor{oursTint}\textbf{\slcaname} & \textbf{78.57 $\pm$ 1.15} & \textbf{69.32 $\pm$ 0.56} & \textbf{39.90 $\pm$ 1.46} \\
$\Delta$ & \textbf{+2.83} & \textbf{+1.24} & \textbf{+7.16} \\
\bottomrule
\end{tabular}
\end{table}

The SLCA advantage remains positive on all three benchmarks after the response-mocker replacement.

\subsection{Baseline Comparison: RLTR and ToolPO Implementation}
\label{app:baseline_comparison}

To provide context for the comparisons in \S\ref{sec:main_results}, we re-implement RLTR~\citep{li2025rltr} and ToolPO~\citep{deepagent2025} within the same verl training framework. The matched GRPO and \slcaname runs share the controls listed in Table~\ref{tab:hyperparameters}; RLTR and ToolPO retain their method-specific initialization and reward/protocol structure. The ToolPO rows are our implementations under the stated outcome-reward choices, not reproductions of the original ToolPO environment.

\paragraph{SFT Initialization}
For each backbone, the matched methods start from that backbone's SFT checkpoint (e.g., \texttt{qwen2.5-7B-sft}); the RLTR planner checkpoint follows its method-specific construction.
RLTR requires a dedicated \emph{Planner SFT}: we remove summary tokens from the SFT data and append an \texttt{<answer>} marker to delineate the planner's output boundary, following the original paper's design.

\paragraph{RLTR Implementation}
We re-implement the Planner-Summarizer pipeline:
\begin{itemize}
    \item \textbf{Reward:} $R = R_{\text{comp}} + R_{\text{repeat}} + R_{\text{error}}$, where $R_{\text{comp}}$ is an LLM-based Completeness Checker (Qwen3-30B-A3B, $N{=}3$ samples averaged), $R_{\text{repeat}} = -0.1 \times \text{repeat\_count}$, and $R_{\text{error}} = -0.2 \times \text{error\_count}$. Format violations yield $R = -1.0$.
    \item \textbf{Advantage:} Only process (tool) tokens receive non-zero advantages; summary token advantages are set to zero.
    \item \textbf{Evaluation:} A 2-stage pipeline: the Planner generates the tool trajectory, then a frozen SFT Summarizer produces the final answer. For BFCL, only the Planner checkpoint is used.
\end{itemize}

\paragraph{ToolPO Implementation}
We report a reference-rule ToolPO condition and an LLM-judge condition as a protocol-sensitivity diagnostic on 7B. The 3B and 8B ToolPO rows in the scale tables use the LLM-judge Solved/Unsolved outcome reward; no reference-rule runs are available for those scales.
\begin{itemize}
    \item \textbf{Outcome reward protocols:} The reference-rule condition uses the available gold tool calls, while the sensitivity condition uses an LLM judge that classifies each rollout as Solved or Unsolved, denoted $R_{\text{succ}}^{\text{judge}}$. Both are evaluated on the 7B backbone below.
    \item \textbf{$C(a_t^{\text{call}})$ adaptation:} Tool-call correctness via strict matching against gold tool calls (exact name match $+$ all gold keys present $+$ lenient value matching). $R_{\text{action}} = \sum C / |\text{gold\_flat}|$, normalized to $[0,1]$.
    \item \textbf{Advantage:} $\toolsub{A} = A_{\text{global}} + \lambda \cdot A_{\text{tool\_local}}$ for tool tokens, $\summarysub{A} = A_{\text{global}}$ for summary tokens, with $\lambda=1.0$ and group-wise normalization. We adopt $\lambda{=}1.0$ following the original ToolPO specification \citep{deepagent2025}, which gives the local action reward parity with the global outcome reward; we did not further tune $\lambda$ to avoid introducing a confound absent from the original design.
\end{itemize}

\paragraph{Reference rule protocol}
We implement the reference rule as an explicit gold-call conjunction using the
same matching primitives as the HierR process reward. Let $\mathcal{N}_Y$ and
$\mathcal{N}_G$ be the predicted and gold tool-name multisets, and let
$y^*:\mathcal{G}\to\mathcal{Y}$ be the one-to-one matching used above, pairing
each gold call with the best unused predicted call of the same name. Let
$Y_{\mathrm{init}}$ and $G_{\mathrm{init}}$ denote first-round calls:
\begin{equation}
\begin{aligned}
R_{\mathrm{succ}}^{\mathrm{rule}}
 ={}&\mathbb{I}\!\big[N_{\mathrm{valid}}=N_{\mathrm{open}}=N_{\mathrm{close}},\;
 \mathcal{N}_Y=\mathcal{N}_G,\;
 |Y_{\mathrm{init}}|=|G_{\mathrm{init}}|,\\[-2pt]
 &\quad \forall g\in\mathcal{G}:\;
 \left[\mathcal{K}^{\mathrm{req}}_g\subseteq\mathcal{K}_{y^*(g)}\ \land\ 
 \forall k\in\mathcal{K}^{\mathrm{req}}_g,\;
 \mathbb{I}_{\mathrm{loose}}\!\left(y^*(g)[k],g[k]\right)=1\right]\big].
\end{aligned}
\label{eq:toolpo_reference_rule}
\end{equation}
Here $N_{\mathrm{valid}}=N_{\mathrm{open}}=N_{\mathrm{close}}$ requires every
parsed call to be valid with balanced opening and closing tags. The name
multiset equality rejects missing or extra calls and incorrect multiplicities;
$\mathcal{K}^{\mathrm{req}}_g$ is the set of required keys
in gold call $g$, and the matched-call key condition requires all of them.
Values use the same case-insensitive, type-coercing matcher as HierR.
The initial cardinality check enforces the parallel-call count. Cross-round
call order is not checked. The conjunction reuses the gold-call checks already
used by $R_{\text{action}}$, so this row is a protocol comparison rather than
an independent reward source.

\begin{table}[H]
\centering
\small
\caption{ToolPO outcome reward protocol comparison on Qwen2.5-7B-Instruct. Entries are mean $\pm$ standard deviation over three matched runs. The reference rule row uses Eq.~\ref{eq:toolpo_reference_rule}; Unified GRPO and \slcaname are shown as contextual matched baselines under their standard reward protocol.}
\label{tab:toolpo_reward_protocol}
\setlength{\tabcolsep}{4pt}
\begin{tabular}{lccc}
\toprule
\rowcolor{ArxivTableHead}
\textbf{Condition} & \textbf{Toucan Succ. (\%)} & \textbf{BFCL Acc. (\%)} & \textbf{$\tau^2$-Bench Pass$^1$ (\%)} \\
\midrule
ToolPO, LLM-judge $R_{\text{succ}}^{\text{judge}}$ & 20.00 $\pm$ 1.36 & 24.59 $\pm$ 0.87 & 30.44 $\pm$ 2.15 \\
ToolPO, reference rule $R_{\text{succ}}^{\text{rule}}$ & 77.35 $\pm$ 1.19 & 68.86 $\pm$ 0.43 & 34.18 $\pm$ 2.42 \\
Unified GRPO & 76.60 $\pm$ 1.27 & 68.41 $\pm$ 0.11 & 31.87 $\pm$ 2.97 \\
\rowcolor{oursTint}\textbf{\slcaname} & \textbf{79.13 $\pm$ 1.05} & \textbf{69.77 $\pm$ 0.47} & \textbf{41.02 $\pm$ 1.01} \\
\bottomrule
\end{tabular}
\end{table}

The two ToolPO rows use the same additive estimator, initialization, data, training budget, and evaluation. Replacing the LLM-judge outcome reward with the reference rule removes the observed 7B collapse in this ToolPO comparison; the rule-based ToolPO remains close to Unified GRPO and below SLCA on $\tau^2$-Bench. This comparison isolates sensitivity to the outcome-reward protocol.

\paragraph{Discussion: ToolPO-7B Format Behavior}
Under the LLM-judge protocol, ToolPO has low format validity on Qwen2.5-7B-Instruct (\texttt{format\_passed}$\,{=}\,$38\%), whereas the 3B and 8B runs retain higher format validity. The matched GRPO and \slcaname rows use the same 7B SFT initialization, SGLS environment, and RL data; ToolPO retains its method-specific initialization, additive estimator, and outcome reward protocol. The 7B result is also reflected in the three-run mean (Toucan Success $0.200{\pm}0.014$). The training trace shows an increase in unmatched opening tags after step~80 (App.~\ref{app:baseline_dynamics}). This is a diagnostic of the LLM-judge condition. The reference-rule comparison above keeps the ToolPO estimator fixed while changing the outcome reward.

\paragraph{Transparency Note}
We document the following necessary adaptations for reproducibility:
\begin{enumerate}
    \item The reference-rule condition uses the available gold tool calls; the LLM-judge condition is retained as a protocol-sensitivity diagnostic.
    \item RLTR's Planner SFT is trained with summary removal and \texttt{<answer>} markers, rather than using the general SFT directly.
    \item RLTR's penalty coefficients ($\lambda_{\text{repeat}}{=}0.1$, $\mu_{\text{error}}{=}0.2$) are set to reasonable values, as the original paper does not specify them.
    \item We do not force HierR into ToolPO's additive update, because doing so would change ToolPO's original reward protocol rather than evaluate the method-faithful baseline.
    \item All methods use vanilla policy loss mode, so the policy-loss implementation is held fixed.
\end{enumerate}

\subsection{Evaluation Protocols}
\label{app:eval_protocols}

For full reproducibility, we document the complete evaluation protocol for each benchmark, including software versions, model configurations, and scoring logic. The matched GRPO and \slcaname rows share the benchmark protocol; the ToolPO and RLTR choices are listed in App.~\ref{app:baseline_comparison}.

\paragraph{$\tau^2$-Bench Protocol}
We use the official $\tau^2$-Bench framework~\citep{barres2025tau2bench} without modifications to prompts, tools, tasks, or scoring logic.
Table~\ref{tab:tau2_protocol} summarizes the configuration.

\begin{table}[H]
\centering
\small
\caption{$\tau^2$-Bench evaluation protocol.}
\label{tab:tau2_protocol}
\setlength{\tabcolsep}{4pt}
\begin{tabular}{ll}
\toprule
\rowcolor{ArxivTableHead}
\textbf{Parameter} & \textbf{Value} \\
\midrule
\texttt{tau2} package & v0.2.1.dev0 (from \texttt{@v0.2.0} tag) \\
Evaluation framework & evalscope 1.3.0 \\
Task split & airline, retail, telecom (all 3 official domains) \\
Dataset source & \texttt{evalscope/tau2-bench-data} (ModelScope) \\
Task filtering & \texttt{LLMGTAgent.check\_valid\_task()} \\
Aggregation & \texttt{mean\_and\_pass\_hat\_k} (built-in) \\
Trials per task & 1 (default \texttt{repeats=1}; Pass$^1$) \\
\midrule
\multicolumn{2}{l}{\textit{User Simulator LLM}} \\
\midrule
Model & DeepSeek-V3.2 \\
Temperature & 0.7 \\
Max tokens & 1024 \\
Timeout & 200s \\
Max retries & 8 \\
\bottomrule
\end{tabular}
\end{table}

\paragraph{BFCL V3 Protocol}
We evaluate on BFCL V3~\citep{pmlr-v267-patil25a} using the function-calling protocol.
Table~\ref{tab:bfcl_protocol} lists the configuration.

\begin{table}[H]
\centering
\small
\caption{BFCL V3 evaluation protocol.}
\label{tab:bfcl_protocol}
\setlength{\tabcolsep}{4pt}
\begin{tabular}{ll}
\toprule
\rowcolor{ArxivTableHead}
\textbf{Parameter} & \textbf{Value} \\
\midrule
Evaluation framework & evalscope 1.3.0 \\
Evaluation mode & \texttt{is\_fc\_model=True} \\
Temperature & 0 (greedy) \\
Max tokens & 8192 \\
\texttt{parallel\_tool\_calls} & True \\
\texttt{underscore\_to\_dot} & True \\
Batch size & 256 \\
\bottomrule
\end{tabular}
\end{table}

\noindent\textit{Chat template alignment.}
The vLLM serving endpoint uses a Toucan-specific chat template that injects a fixed system prompt and skips upstream system messages.
This is a \textbf{chat template alignment} choice ensuring inference consistency with training, not a modification of the BFCL benchmark protocol.
The skipped upstream system messages contain only generic formatting instructions; skipping them is a compatibility step and does not change instance-specific content or evaluation semantics.
BFCL tool definitions are injected through the standard \texttt{tools} parameter in the function-calling interface, which the template processes correctly.
The \texttt{underscore\_to\_dot=True} flag reconciles naming-convention differences between training-time schemas and BFCL's schema format.
For the main BFCL V3 comparison, we exclude the relevance-detection subset and restrict evaluation to single-turn samples. These choices are held fixed for the matched GRPO and \slcaname rows; the Multi-Turn extension is reported separately below.

\paragraph{Toucan In-Domain Protocol}
The Toucan evaluation uses a custom asynchronous evaluation system (independent of evalscope).
Table~\ref{tab:toucan_protocol} lists the agent and environment configuration.

\begin{table}[H]
\centering
\small
\caption{Toucan in-domain evaluation protocol.}
\label{tab:toucan_protocol}
\setlength{\tabcolsep}{4pt}
\begin{tabular}{ll}
\toprule
\rowcolor{ArxivTableHead}
\textbf{Parameter} & \textbf{Value} \\
\midrule
Dataset & \texttt{tool\_call\_rl\_40k\_v1.parquet} (4{,}000 held-out) \\
Agent loop max turns & 10 \\
Temperature & 0.0 (greedy) \\
Max tokens & 12{,}288 \\
Top-$p$ & 1.0 \\
Concurrency & 256 \\
Retries & 3 \\
\midrule
\multicolumn{2}{l}{\textit{Tool Simulator (same as SGLS)}} \\
\midrule
Model & Qwen3-235B-A22B \\
Temperature & 0.6 \\
Max tokens & 4{,}096 \\
Timeout & 540s \\
\midrule
\multicolumn{2}{l}{\textit{LLM Judge}} \\
\midrule
Model & GPT-OSS-120B \\
Temperature & 0.0 (deterministic) \\
Max tokens & 8{,}192 \\
Scoring & 5-point categorical $\to$ $[0,1]$ via $(r{-}1)/4$ \\
 Summary-judge threshold & rating $\in$ \{good, excellent\} ($\geq 0.75$) \\
\bottomrule
\end{tabular}
\end{table}

The primary Success metric in Table~\ref{tab:toucan_main} is $\mathbb{I}[S_{\text{process}}\geq 0.9]$; it is distinct from the summary-judge threshold above.

\noindent\textit{Process Score sub-item weights.}
The $S_{\text{process}}$ metric (Eq.~\ref{eq:hierr_process}) decomposes into five sub-items:
format ($w{=}0.10$), name ($w{=}0.25$), key ($w{=}0.15$), value ($w{=}0.20$), and parallel ($w{=}0.30$).
Name matching uses multiset F1 over predicted and gold tool-name occurrences; key matching uses gold-call greedy Jaccard similarity; value matching uses lenient comparison (case-insensitive, automatic numeric type conversion).

\subsection{Details on Tool Space Scalability}
\label{app:toolspace_details}

We evaluate tool-space scalability via a controlled candidate-size sweep (K-sweep) under two distractor sampling strategies (Hard vs.\ Random). This appendix documents the dataset statistics, the K-sweep candidate-set construction procedure, the hard-negative mining mechanism, and the context-length profiling.

\paragraph{Dataset and Global Tool Pool}
All K-sweep datasets are derived from the Toucan evaluation set containing 4,000 instances. The global tool pool comprises 1,568 tools in total, exhibiting a long-tailed distribution of schema lengths. In the original data, the allowed-tool list contains 5.44 tools on average, and the gold-tool set size is 1.54 on average. These statistics motivate a controlled augmentation protocol to probe robustness under larger candidate sets.

\paragraph{Controlled Candidate-Set Construction (K-sweep)}
For each instance $i$, let $\mathcal{A}_i$ denote the original allowed tool set and $\mathcal{G}_i \subseteq \mathcal{A}_i$ denote the gold tool set. Given a target candidate size $K$, we construct a controlled candidate set $\tilde{\mathcal{A}}_i$ as follows:
\begin{itemize}
    \item \textbf{Gold preservation:} Always retain all gold tools $\mathcal{G}_i$ in $\tilde{\mathcal{A}}_i$.
    \item \textbf{Original baseline ($K=0$):} If $K=0$, keep the instance unchanged, i.e., $\tilde{\mathcal{A}}_i=\mathcal{A}_i$.
    \item \textbf{Expansion ($|\mathcal{A}_i|<K$):} Add $K-|\mathcal{A}_i|$ distractor tools sampled from the global tool pool, excluding tools already present in $\mathcal{A}_i$.
    \item \textbf{Contraction ($|\mathcal{A}_i|>K$):} Keep all gold tools and uniformly subsample $K-|\mathcal{G}_i|$ tools from $\mathcal{A}_i \setminus \mathcal{G}_i$.
\end{itemize}
This construction enforces a precise candidate-set size while guaranteeing that the target tool(s) remain available, thereby isolating the effect of tool-space size on tool selection.

\paragraph{Distractor Sampling Strategies.}
We implement two strategies for selecting distractors in the expansion step:
\begin{itemize}
    \item \textbf{Random Sampling (Control):} Distractors are drawn uniformly at random (without replacement) from the global tool pool after excluding $\mathcal{A}_i$.
    \item \textbf{Hard-Negative Sampling (Semantic Distractors):} We mine semantically similar distractors based on tool-description similarity. Specifically, we encode each tool's description text (tool name concatenated with its functionality docstring) using the sentence-transformer \texttt{sentence-transformers/all-MiniLM-L6-v2}, and precompute a cosine similarity matrix over the global pool. For each instance, we retrieve the most similar non-gold tools to the gold tool(s) (excluding tools already in $\mathcal{A}_i$) and take the top-scoring tools as distractors until the required budget is met. This procedure produces functionally related distractors, maximizing the probability of semantic interference.
\end{itemize}

\paragraph{Context Length Profiling and Skip Policy}
All token statistics below are computed using the Qwen2.5-7B-Instruct tokenizer with a maximum context length of 32,768 tokens. During K-sweep dataset generation, if an instance exceeds the context limit after tool-schema injection, it is automatically skipped. Tables~\ref{tab:token_profile_random} and~\ref{tab:token_profile_hard} report empirical system-prompt length statistics and skip rates for Random and Hard strategies, respectively.

\begin{table}[H]
  \centering
  \small
  \setlength{\tabcolsep}{4pt}
  \caption{System prompt length statistics vs.\ candidate set size ($K$) for \textbf{Random} distractors (Qwen2.5 tokenizer).}
  \label{tab:token_profile_random}
  \begin{tabular}{lccccc}
  \toprule
\rowcolor{ArxivTableHead}
  \textbf{Candidate Size ($K$)} & \textbf{Original} & \textbf{20} & \textbf{50} & \textbf{100} & \textbf{150} \\
  \midrule
  \textbf{Valid Samples} & 3,991 & 3,980 & 3,882 & 2,607 & 443 \\
  \textbf{Avg. Tokens} & 943 & 4,614 & 11,064 & 17,177 & 23,635 \\
  \textbf{Std. Dev.} & 1,114 & 3,679 & 5,256 & 1,550 & 733 \\
  \textbf{Skip Rate} & 0.2\% & 0.5\% & 2.9\% & 34.8\% & 88.9\% \\
  \bottomrule
  \end{tabular}
\end{table}

\begin{table}[H]
  \centering
  \small
  \setlength{\tabcolsep}{4pt}
  \caption{System prompt length statistics vs.\ candidate set size ($K$) for \textbf{Hard} distractors (Qwen2.5 tokenizer).}
  \label{tab:token_profile_hard}
  \begin{tabular}{lccccc}
  \toprule
\rowcolor{ArxivTableHead}
  \textbf{Candidate Size ($K$)} & \textbf{Original} & \textbf{20} & \textbf{50} & \textbf{100} & \textbf{150} \\
  \midrule
  \textbf{Valid Samples} & 3,991 & 3,959 & 3,793 & 3,023 & 2,034 \\
  \textbf{Avg. Tokens} & 943 & 3,693 & 8,522 & 14,483 & 20,609 \\
  \textbf{Std. Dev.} & 1,114 & 2,735 & 3,806 & 2,595 & 1,719 \\
  \textbf{Skip Rate} & 0.2\% & 1.0\% & 5.2\% & 24.4\% & 49.2\% \\
  \bottomrule
  \end{tabular}
\end{table}

\noindent
Hard sampling tends to yield shorter prompts and lower skip rates at fixed $K$ than Random sampling, since semantically similar tools are more likely to match the length profile of tools already present in the original candidate set. We report main-paper scalability results up to $K=100$ and treat $K=150$ as a stress-test regime with high context-limit attrition.

\section{Additional Experimental Results} 
\label{app:full_results}

\subsection{Toucan-Test: full HierR decomposition}
Table~\ref{tab:toucan_full} expands Table~\ref{tab:toucan_main} by reporting two additional components:
(i) \emph{Parallel}, which captures the parallelism/cardinality constraints in HierR, and
(ii) \emph{Summary}, an LLM-judge score over the user-facing segment $\summarysub{y}$.
We include RLTR and ToolPO as prior credit-assignment baselines. Reporting these terms separately is useful here: RLTR does not optimize the summary segment inside RL, while ToolPO can attain a non-trivial summary score even when tool execution degrades.

\begin{table}[ht] 
\centering
\small 
\caption{
Comprehensive Results on Toucan-Test across Scales.
This table expands upon the main results by including all sub-metrics.
In addition to Name F1, ArgMatch, Process, and Success (defined in Table~\ref{tab:toucan_main}), we report:
Parallel: The parallelism/cardinality constraint score defined in Eq.~\ref{eq:hierr_parallel};
Summary: The user-facing response quality score evaluated by the LLM-judge.
\textsc{\slcaname} has the highest mean Success and Process values across
scales. RLTR and ToolPO
are included for context; their method-specific protocols are described in
App.~\ref{app:baseline_comparison}. 
RLTR's Summary column is marked ``--'' as its frozen summarizer
is not optimized during RL. Trained rows report mean$\pm$std over three runs;
Original rows are point evaluations.}
\label{tab:toucan_full}

\setlength{\tabcolsep}{5pt}
\begin{tabular}{lcccccc}
\toprule
\rowcolor{ArxivTableHead}
\textbf{Method} & \textbf{Name F1} & \textbf{ArgMatch} & \textbf{Parallel} & \textbf{Process} & \textbf{Summary} & \textbf{Success} \\
\midrule

\rowcolor{scaleThreeTint}\textit{\textbf{Qwen2.5-3B-Instruct}} &  &  &  &  &  &  \\
\midrule
Original       & 0.5207 & 0.5073 & 0.5412 & 0.5174 & 0.3477 & 0.3816 \\
SFT                 & 0.7685{\scriptsize$\pm$.0069} & 0.7132{\scriptsize$\pm$.0239} & 0.7772{\scriptsize$\pm$.0071} & 0.7401{\scriptsize$\pm$.0050} & 0.5512{\scriptsize$\pm$.0047} & 0.6500{\scriptsize$\pm$.0089} \\
SFT+GRPO            & 0.8910{\scriptsize$\pm$.0087} & 0.8072{\scriptsize$\pm$.0154} & 0.9176{\scriptsize$\pm$.0080} & 0.8577{\scriptsize$\pm$.0057} & 0.7842{\scriptsize$\pm$.0122} & 0.7412{\scriptsize$\pm$.0115} \\
RLTR       & 0.7598{\scriptsize$\pm$.0149} & 0.6353{\scriptsize$\pm$.0229} & 0.6729{\scriptsize$\pm$.0098} & 0.6924{\scriptsize$\pm$.0162} & --     & 0.6627{\scriptsize$\pm$.0151} \\
ToolPO                  & 0.8050{\scriptsize$\pm$.0142} & 0.7769{\scriptsize$\pm$.0160} & 0.8946{\scriptsize$\pm$.0137} & 0.8368{\scriptsize$\pm$.0141} & 0.6189{\scriptsize$\pm$.0117} & 0.5128{\scriptsize$\pm$.0145} \\
\rowcolor{oursTint}\textbf{\textsc{Ours}} & \textbf{0.9006}{\scriptsize$\pm$.0087} & \textbf{0.8092}{\scriptsize$\pm$.0112} & \textbf{0.9214}{\scriptsize$\pm$.0071} & \textbf{0.8625}{\scriptsize$\pm$.0049} & 0.7635{\scriptsize$\pm$.0089} & \textbf{0.7647}{\scriptsize$\pm$.0093} \\

\midrule

\rowcolor{scaleSevenTint}\textit{\textbf{Qwen2.5-7B-Instruct}} &  &  &  &  &  &  \\
\midrule
Original                & 0.8008 & 0.7388 & 0.8003 & 0.7687 & 0.5882 & 0.6224 \\
SFT                     & 0.8390{\scriptsize$\pm$.0075} & 0.7531{\scriptsize$\pm$.0224} & 0.8397{\scriptsize$\pm$.0083} & 0.8094{\scriptsize$\pm$.0082} & 0.6346{\scriptsize$\pm$.0054} & 0.7214{\scriptsize$\pm$.0087} \\
SFT+GRPO                & 0.8971{\scriptsize$\pm$.0093} & 0.8350{\scriptsize$\pm$.0148} & 0.9091{\scriptsize$\pm$.0079} & 0.8667{\scriptsize$\pm$.0061} & 0.8241{\scriptsize$\pm$.0108} & 0.7660{\scriptsize$\pm$.0127} \\
RLTR      & 0.7673{\scriptsize$\pm$.0132} & 0.6346{\scriptsize$\pm$.0205} & 0.6928{\scriptsize$\pm$.0119} & 0.7022{\scriptsize$\pm$.0143} & --     & 0.6545{\scriptsize$\pm$.0168} \\
ToolPO                  & 0.3042{\scriptsize$\pm$.0168} & 0.3552{\scriptsize$\pm$.0192} & 0.3438{\scriptsize$\pm$.0143} & 0.3408{\scriptsize$\pm$.0157} & 0.6054{\scriptsize$\pm$.0128} & 0.2000{\scriptsize$\pm$.0136} \\
\rowcolor{oursTint}\textbf{\textsc{Ours}}  & \textbf{0.9164}{\scriptsize$\pm$.0075} & \textbf{0.8353}{\scriptsize$\pm$.0132} & \textbf{0.9251}{\scriptsize$\pm$.0068} & \textbf{0.8766}{\scriptsize$\pm$.0045} & \textbf{0.8450}{\scriptsize$\pm$.0087} & \textbf{0.7913}{\scriptsize$\pm$.0105} \\

\midrule

\rowcolor{scaleEightTint}\textit{\textbf{Qwen3-8B-Base}} &  &  &  &  &  &  \\
\midrule
Original                & 0.0887 & 0.1158 & 0.1208 & 0.1089 & 0.0322 & 0.0471 \\
SFT                     & 0.8058{\scriptsize$\pm$.0082} & 0.7437{\scriptsize$\pm$.0260} & 0.8133{\scriptsize$\pm$.0071} & 0.7754{\scriptsize$\pm$.0060} & 0.6261{\scriptsize$\pm$.0058} & 0.6925{\scriptsize$\pm$.0098} \\
SFT+GRPO                & 0.9110{\scriptsize$\pm$.0092} & 0.8334{\scriptsize$\pm$.0135} & 0.9287{\scriptsize$\pm$.0065} & 0.8767{\scriptsize$\pm$.0050} & 0.8339{\scriptsize$\pm$.0110} & 0.7697{\scriptsize$\pm$.0141} \\
RLTR      & 0.8780{\scriptsize$\pm$.0125} & 0.7468{\scriptsize$\pm$.0219} & 0.7523{\scriptsize$\pm$.0097} & 0.7941{\scriptsize$\pm$.0147} & --     & 0.7454{\scriptsize$\pm$.0167} \\
ToolPO                  & 0.6854{\scriptsize$\pm$.0168} & 0.8180{\scriptsize$\pm$.0177} & 0.6642{\scriptsize$\pm$.0135} & 0.7510{\scriptsize$\pm$.0165} & 0.6789{\scriptsize$\pm$.0137} & 0.3199{\scriptsize$\pm$.0156} \\
\rowcolor{oursTint}\textbf{\textsc{Ours}}  & \textbf{0.9177}{\scriptsize$\pm$.0075} & 0.8320{\scriptsize$\pm$.0140} & \textbf{0.9303}{\scriptsize$\pm$.0064} & \textbf{0.8786}{\scriptsize$\pm$.0046} & 0.8239{\scriptsize$\pm$.0079} & \textbf{0.7902}{\scriptsize$\pm$.0105} \\
\bottomrule
\end{tabular}
\end{table}

\paragraph{Process and summary scores.}
Across scales, \textsc{\slcaname} has higher mean \emph{Process} and
strict \emph{Success} scores than matched GRPO, and higher mean
\emph{Parallel} scores. These metrics capture structural constraints beyond
tool-name matching.
On Qwen2.5-3B and Qwen3-8B-Base, standard GRPO attains
a higher \emph{Summary} score but a lower \emph{Success} rate than
\textsc{\slcaname}.
This pattern is compatible with Cross-Segment Credit Misattribution:
a plausible summary can coexist with erroneous tool behavior when a mixed
advantage is broadcast to all tokens.
Under the LLM-judge diagnostic, ToolPO shows the same separation on Qwen2.5-7B:
it reaches \emph{Summary} 0.6054 while its \emph{Process} score is 0.3408.
On Qwen2.5-7B, \textsc{\slcaname} has higher reported means for both
execution and \emph{Summary}; the two scores move in the same direction
in this comparison.

\subsection{Summary Advantage Ablation}
\label{app:summary_advantage_ablation}

We isolate the learned summary-reward branch while keeping the omission safeguard
unchanged. Let $\mathcal{O}$ denote rollouts that trigger the omission guard.
After computing the ordinary group-normalized summary advantages, we set
$\hat A_i^{\mathrm{sum}}=0$ for $i
otin\mathcal{O}$ and retain the computed
guard advantage for $i\in\mathcal{O}$. The summary tokens remain in the loss
mask and under the per-token KL constraint. We query the judge for every
rollout, including omission cases; its parameters remain frozen. For a guard
rollout, the policy reward is overridden by
$R_{\mathrm{sum}}=R_{\mathrm{pen}}=-0.5$.

\begin{table}[H]
\centering
\scriptsize
\caption{Summary advantage ablation on Qwen2.5-7B-Instruct. Process and Summary are unit-interval scores; benchmark columns are percentages. The no-call column is the end-of-training rate. Entries in the score columns are mean $\pm$ standard deviation over three matched runs.}
\label{tab:summary_advantage_ablation}
\setlength{\tabcolsep}{1.5pt}
\resizebox{\linewidth}{!}{%
\begin{tabular}{lcccccc}
\toprule
\rowcolor{ArxivTableHead}
\textbf{Condition} & \textbf{No-call (\%)} & \textbf{Process} & \textbf{Toucan (\%)} & \textbf{Summary} & \textbf{BFCL (\%)} & \textbf{$\tau^2$ (\%)} \\
\midrule
SFT & -- & 0.8094 $\pm$ 0.0082 & 72.14 $\pm$ 0.87 & 0.6346 $\pm$ 0.0054 & 67.89 $\pm$ 0.28 & 32.54 $\pm$ 1.92 \\
Unified GRPO & 0.6 & 0.8667 $\pm$ 0.0061 & 76.60 $\pm$ 1.27 & 0.8241 $\pm$ 0.0108 & 68.41 $\pm$ 0.11 & 31.87 $\pm$ 2.97 \\
\shortstack[l]{$\hat A^{\mathrm{sum}}=0$\\(except omission penalty episodes)} & 0.5 & 0.8748 $\pm$ 0.0061 & 78.72 $\pm$ 1.22 & 0.6512 $\pm$ 0.0208 & 69.44 $\pm$ 0.55 & 36.24 $\pm$ 1.58 \\
\shortstack[l]{$\hat A^{\mathrm{sum}}=0$\\(including omission penalty episodes)} & 14.2 & 0.7506 $\pm$ 0.0227 & 67.54 $\pm$ 1.94 & 0.5587 $\pm$ 0.0248 & 61.53 $\pm$ 1.74 & 29.44 $\pm$ 2.28 \\
\rowcolor{oursTint}\textbf{\slcaname} & 0.4 & \textbf{0.8766 $\pm$ 0.0045} & \textbf{79.13 $\pm$ 1.05} & \textbf{0.8450 $\pm$ 0.0087} & \textbf{69.77 $\pm$ 0.47} & \textbf{41.02 $\pm$ 1.01} \\
RLTR & -- & 0.7022 $\pm$ 0.0143 & 65.45 $\pm$ 1.68 & -- & 63.11 $\pm$ 0.62 & 33.72 $\pm$ 2.38 \\
\bottomrule
\end{tabular}}
\end{table}

The omission-exception condition remains close to SLCA on Process, Toucan,
and BFCL, while its $\tau^2$-Bench score is lower by 4.78 pp and its Summary
score is close to the SFT value. The full-zero condition is reported as a
guard diagnostic: removing both the summary advantage and the guard raises the
no-call rate to 14.2\% and lowers all listed scores. Its no-call rollouts
receive no advantage-weighted policy gradient, while the KL term remains active.
The two rows are not interpreted as a comparison of the summary branch alone.

\subsection{Training Dynamics against RLTR and ToolPO}
\label{app:baseline_dynamics}

Static benchmark numbers can obscure how proxy objectives and execution quality move during training. We analyze the full training traces across all three model scales and observe different behaviors for ToolPO's additive advantage design.

\paragraph{ToolPO training trace}

\begin{figure}[H]
    \centering
    \includegraphics[width=0.88\linewidth]{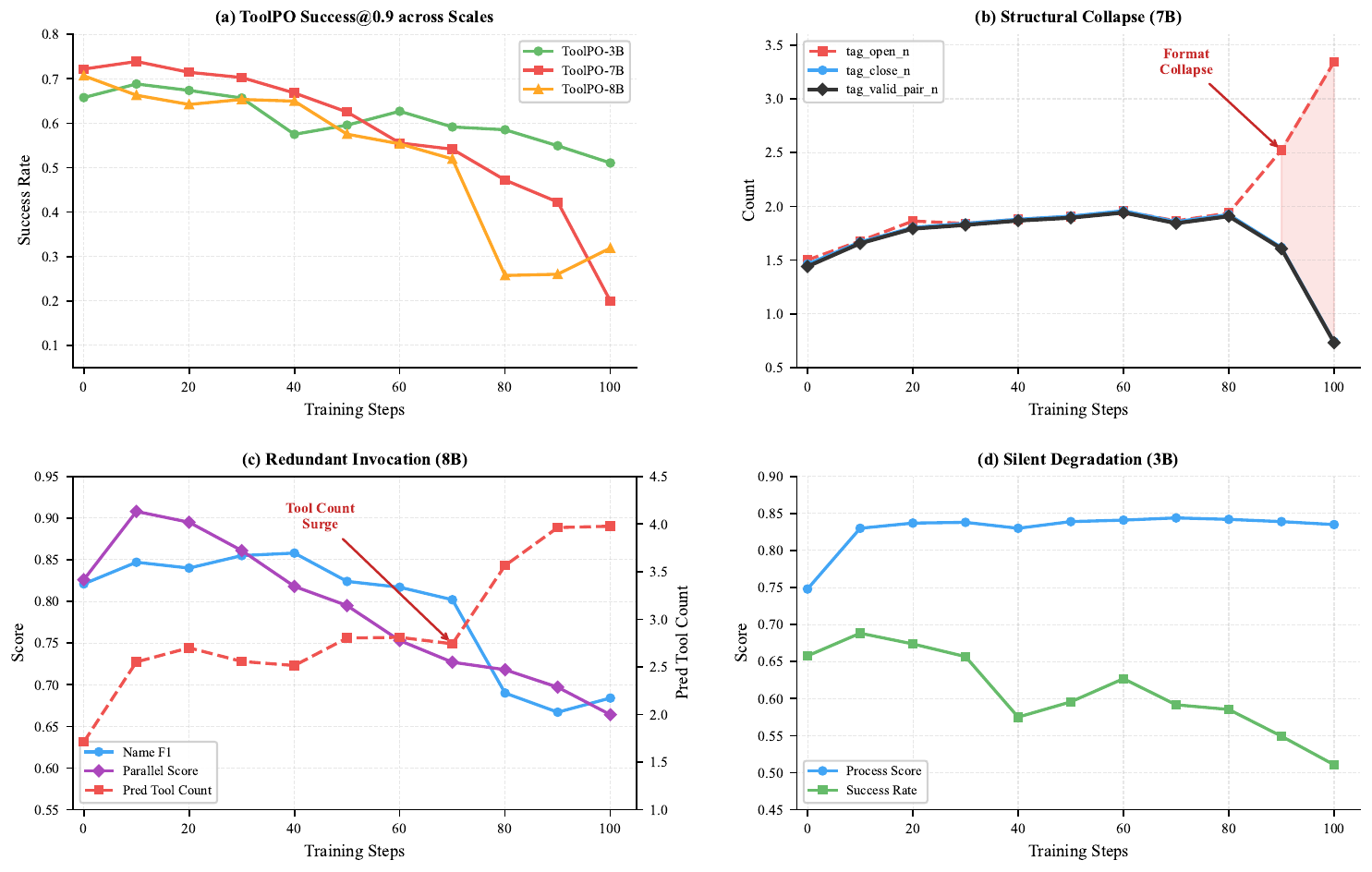}
    \vspace{-0.2cm}
    \caption{\textbf{ToolPO training behavior under the LLM-judge diagnostic (representative single-run traces).}
    \textbf{(a)} Success{@}0.9 across three scales: a decline on 3B ($0.65{\to}0.51$), a sharper decline on 7B ($0.72{\to}0.20$), and an overall decline with a brief mid-training rebound on 8B, ending near $0.32$.
    \textbf{(b)} 7B: \texttt{tag\_open\_n} diverges from \texttt{tag\_close\_n} after step~80, reducing the valid parsed-call count from 1.91 to 0.73.
    \textbf{(c)} 8B: \texttt{pred\_tool\_count} surges $1.71{\to}3.98$ (gold ${\approx}2.0$), driving Name~F1 and Parallel~Score down.
    \textbf{(d)} 3B: Process~Score plateaus at ${\sim}0.84$, below the $S_{\text{process}} \geq 0.9$ threshold, explaining why proxy metrics improve while strict Success declines.
    }
    \label{fig:toolpo_reward_hacking}
\end{figure}

Here \texttt{tag\_valid\_pair\_n} denotes the number of parsed, schema-valid
call blocks. It is not a count formed by pairing the opening and closing tag
totals.

The traces show different failure patterns across model scales (Figure~\ref{fig:toolpo_reward_hacking}a):

\begin{itemize}[nosep,leftmargin=*]
    \item \textbf{3B: Silent degradation (Figure~\ref{fig:toolpo_reward_hacking}d).} The proxy metrics increase ($R_{\text{succ}}$: $0.24{\to}0.50$; summary: $0.56{\to}0.62$), and \texttt{process\_score} rises from 0.75 to a plateau of ${\sim}0.84$. Strict Success{@}0.9 declines ($0.65{\to}0.51$). The observed gap is consistent with the threshold $S_{\text{process}} \geq 0.9$: the Process~Score remains near ${\sim}0.84$. Format integrity is preserved (\texttt{format\_passed} $\approx 0.98$, \texttt{tag\_valid\_pair\_n} $\approx 2.0$), so the decline is not accompanied by a structural-format change in this trace.
    \item \textbf{7B: Structural format decline (Figure~\ref{fig:toolpo_reward_hacking}b).} The trace shows more opening than closing tool-call tags (\texttt{tag\_open\_n}: $1.50{\to}3.34$; \texttt{tag\_close\_n}: $1.46{\to}0.74$), while the summary score remains non-zero. The valid parsed-call count drops from 1.91 to 0.73, and Success{@}0.9 drops from 0.72 to 0.20. The change is concentrated between steps 80--100, when the gap between \texttt{tag\_open\_n} and \texttt{tag\_close\_n} widens to 2.60.
    \item \textbf{8B: Redundant tool invocation (Figure~\ref{fig:toolpo_reward_hacking}c).} Qwen3-8B-Base keeps \texttt{tag\_open\_n} and \texttt{tag\_close\_n} close in this trace. The trace instead shows more predicted tool calls (\texttt{pred\_tool\_count}: $1.71{\to}3.98$, gold ${\approx}2.0$). Success{@}0.9 declines overall, with a brief mid-training rebound, and ends near ${\sim}0.32$ as redundant calls accumulate. At steps~70--80, \texttt{pred\_tool\_count} exceeds 3.5 and Name~F1 falls from $0.80$ to $0.69$.
\end{itemize}

These observations are specific to the LLM-judge diagnostic and do not imply that ToolPO always fails. They motivate the separate matched GRPO comparison, where the summary-to-tool support path is blocked by SLCA's routing.

\paragraph{RLTR: Coarse-Grained Reward and Training Stagnation}


As visualized in Figure~\ref{fig:baseline_dynamics_7b}, RLTR avoids direct summary-to-tool leakage via pipeline separation, but its planning reward remains sequence-level and coarse. In our runs, \texttt{rltr\_comp} oscillates between 0.49 and 0.64, while \emph{Process} and \emph{Parallel} improve much more slowly than under \textsc{\slcaname}. On Qwen3-8B-Base, $R_{\text{comp}}$ doubles (0.22$\to$0.43) while \texttt{error\_count} increases from 0.39 to 0.59, suggesting that the completeness reward does not penalize structurally invalid calls (cf.\ the frozen-summarizer limitation noted in Table~\ref{tab:toucan_full}).

\subsection{Visualizing Credit Misattribution}
\label{app:case_study_viz}

\Cref{fig:case_study_viz} presents a diagnostic case on a parallel
information-retrieval task (comparing the current stock prices of Apple
and Microsoft), illustrating how Cross-Segment Credit Misattribution
manifests as \emph{tool-use inefficiency masked by a perfect summary}.
The reference trajectory consists of a single parallel
\texttt{get\_stock\_price} invocation covering both tickers, and the
tool-level reward is decomposed as $\toolsub{R} = 0.10\,R_{\text{format}}
+ 0.25\,R_{\text{name}} + 0.15\,R_{\text{key}} + 0.20\,R_{\text{value}}
+ 0.30\,R_{\text{parallel}}$.

\begin{figure}[H]
    \centering
    \includegraphics[width=0.95\textwidth]{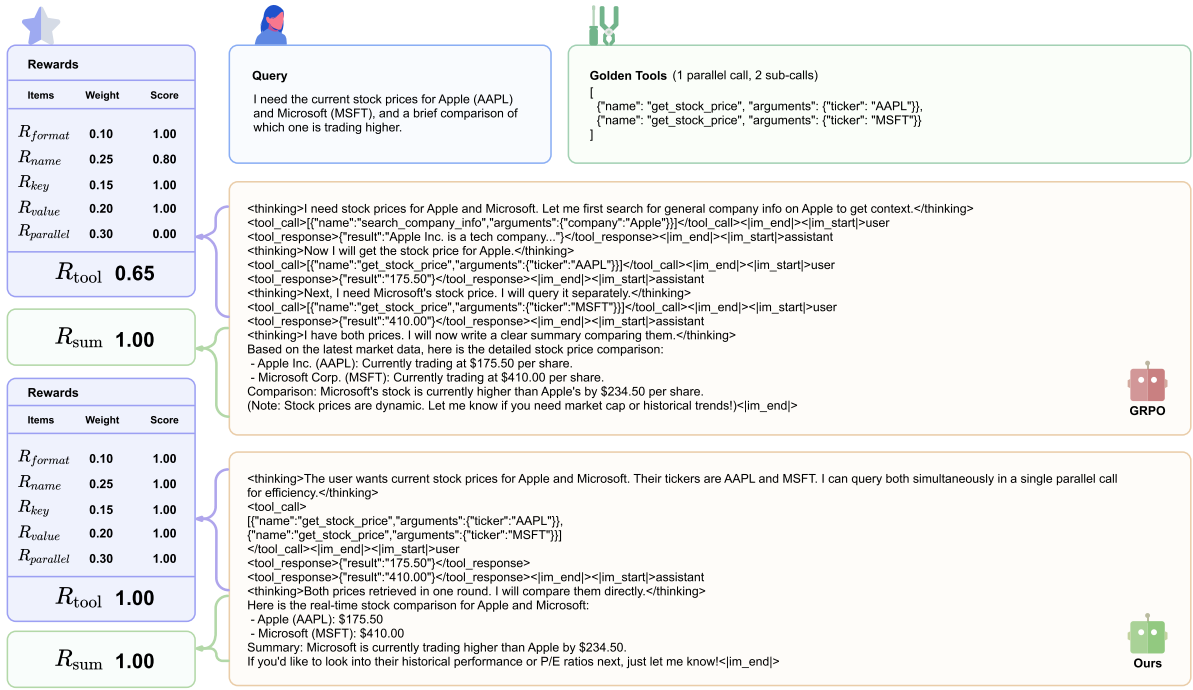}
    \caption{
    \textbf{Diagnostic Case Study: Resolving Cross-Segment Credit
    Misattribution on a parallel tool-use task.}
    \textbf{(Top) Reference trajectory and rewards.}
    \textbf{(Middle) Standard GRPO.} The agent adopts a redundant,
    trial-and-error trajectory (first issuing an unnecessary
    \texttt{search\_company\_info} call, then retrieving the two stock
    prices through \emph{serial} \texttt{get\_stock\_price} invocations
    in separate turns), yet still produces a fluent, factually correct
    comparison ($\summarysub{R}{=}1.00$). The tool-level decomposition
    exposes where the trajectory fails: the \texttt{format} reward is
    unaffected ($1.00$), the \texttt{key} and \texttt{value} rewards
    remain high on the individual calls that do match the schema
    ($1.00$ each), the \texttt{name} reward is
    diluted by the redundant \texttt{search\_company\_info} call
    ($0.80$), and the \texttt{parallel} reward is
    $\mathbf{0.00}$ because the two required stock-price queries are
    issued as separate sequential turns (further preceded by a redundant
    exploration call), rather than batched into a single parallel
    invocation. Aggregated under the fixed weights above, this yields
    $\toolsub{R}{=}0.65$.
    Because standard GRPO merges $\toolsub{R}$ and $\summarysub{R}$
    into a single trajectory-level advantage, the high summary reward
    can offset the structural tool penalty in the unified advantage, allowing
    the inefficient pattern to be reinforced.
    \textbf{(Bottom) \slcaname (Ours).} The same model, trained under
    our advantage-isolation objective, produces a trajectory in which both tickers are
    known a priori and issues a \emph{single parallel tool call}
    covering AAPL and MSFT simultaneously. All five tool sub-rewards
    saturate at $1.00$, giving $\toolsub{R}{=}1.00$ while
    preserving $\summarysub{R}{=}1.00$. By routing $\toolsub{A}$
    and $\summarysub{A}$ through separate optimization pathways, segment
    routing keeps the summary signal from entering the tool-token update in
    this example. The tool-execution score increases by $+0.35$ while the
    summary score remains $1.00$. The trajectory is qualitatively similar to
    the representative single-run \texttt{pred\_tool\_count} increase observed
    under ToolPO on 8B (\Cref{fig:toolpo_reward_hacking}c); the two diagnostics
    use different protocols.
    }
    \label{fig:case_study_viz}
\end{figure}

\subsection{BFCL: Robustness in Multi-Turn Scenarios}
\label{app:bfcl_multiturn}

While the main paper (\Cref{fig:bfcl_scales}) focuses on atomic
schema generalization, \Cref{tab:bfcl_full} additionally reports
Multi-Turn accuracy, now including RLTR and ToolPO.
On Qwen3-8B-Base, standard GRPO has a lower Multi-Turn mean than SFT
(0.0822$\pm$0.0035 vs.\ 0.1544$\pm$0.0102), while
\textsc{\slcaname} reaches 0.1489$\pm$0.0025. The SLCA mean is close to
the SFT value in this three-run comparison; the table does not isolate the
source of the Multi-Turn difference.
On Qwen2.5-3B-Instruct, \slcaname{} also remains below SFT on Multi-Turn (0.0622$\pm$0.0025 vs.\ 0.0867$\pm$0.0100).
RLTR remains competitive on AST Non-Live, suggesting that
planner-only RL can learn basic tool selection under unseen
schemas, but it lags on Multi-Turn where end-to-end coordination
matters.
Under the reported ToolPO evaluation, the Multi-Turn means are
0.0961$\pm$0.0067 on 3B and 0.0083$\pm$0.0067 on 8B. This spread is
reported as context rather than as a claim about the original method across
model scales.

\begin{table}[H]
\centering
\small
\caption{
Comprehensive Generalization Results on BFCL V3 across Scales.
This table details BFCL generalization performance.
In addition to Overall Acc, AST Live, and AST Non-Live (reported in \Cref{fig:bfcl_scales}), we include Multi-Turn accuracy to assess state tracking under schema generalization.
Trained rows report mean$\pm$std over three runs; Original rows are point evaluations. ToolPO and RLTR retain their method-specific protocols.
}
\label{tab:bfcl_full}
\setlength{\tabcolsep}{5pt}
\begin{tabular}{lcccc}
\toprule
\rowcolor{ArxivTableHead}
\textbf{Method} & \textbf{Overall Acc} & \textbf{AST Live} & \textbf{AST Non-Live} & \textbf{Multi-Turn} \\
\midrule

\rowcolor{scaleThreeTint}\textit{\textbf{Qwen2.5-3B-Instruct}} &  &  &  &  \\
\midrule
Original            & 0.4225 & 0.5174 & 0.5139 & 0.0271 \\
SFT                 & 0.6163{\scriptsize$\pm$.0026} & 0.6862{\scriptsize$\pm$.0022} & 0.8104{\scriptsize$\pm$.0043} & 0.0867{\scriptsize$\pm$.0100} \\
SFT+GRPO            & 0.6321{\scriptsize$\pm$.0032} & 0.6936{\scriptsize$\pm$.0030} & 0.8583{\scriptsize$\pm$.0087} & 0.0600{\scriptsize$\pm$.0033} \\
RLTR                & 0.5838{\scriptsize$\pm$.0065} & 0.6136{\scriptsize$\pm$.0074} & 0.8275{\scriptsize$\pm$.0044} & 0.0494{\scriptsize$\pm$.0100} \\
ToolPO              & 0.6095{\scriptsize$\pm$.0086} & 0.6477{\scriptsize$\pm$.0108} & 0.8325{\scriptsize$\pm$.0139} & 0.0961{\scriptsize$\pm$.0067} \\
\rowcolor{oursTint}\textbf{\textsc{Ours}} & \textbf{0.6361}{\scriptsize$\pm$.0052} & \textbf{0.7027}{\scriptsize$\pm$.0070} & 0.8574{\scriptsize$\pm$.0083} & 0.0622{\scriptsize$\pm$.0025} \\

\midrule

\rowcolor{scaleSevenTint}\textit{\textbf{Qwen2.5-7B-Instruct}} &  &  &  &  \\
\midrule
Original            & 0.6490 & 0.7446 & 0.8130 & 0.1172 \\
SFT                 & 0.6789{\scriptsize$\pm$.0028} & 0.7595{\scriptsize$\pm$.0053} & 0.8635{\scriptsize$\pm$.0012} & 0.1435{\scriptsize$\pm$.0120} \\
SFT+GRPO            & 0.6841{\scriptsize$\pm$.0011} & 0.7783{\scriptsize$\pm$.0026} & 0.8548{\scriptsize$\pm$.0074} & 0.1443{\scriptsize$\pm$.0034} \\
RLTR                & 0.6311{\scriptsize$\pm$.0062} & 0.6840{\scriptsize$\pm$.0078} & 0.8591{\scriptsize$\pm$.0053} & 0.0750{\scriptsize$\pm$.0094} \\
ToolPO              & 0.2459{\scriptsize$\pm$.0087} & 0.2043{\scriptsize$\pm$.0103} & 0.4096{\scriptsize$\pm$.0118} & 0.0251{\scriptsize$\pm$.0056} \\
\rowcolor{oursTint}\textbf{\textsc{Ours}} & \textbf{0.6977}{\scriptsize$\pm$.0047} & \textbf{0.7820}{\scriptsize$\pm$.0068} & \textbf{0.8739}{\scriptsize$\pm$.0086} & \textbf{0.1550}{\scriptsize$\pm$.0024} \\

\midrule

\rowcolor{scaleEightTint}\textit{\textbf{Qwen3-8B-Base}} &  &  &  &  \\
\midrule
Original            & 0.2403 & 0.2435 & 0.3617 & 0.0000 \\
SFT                 & 0.6850{\scriptsize$\pm$.0035} & 0.7779{\scriptsize$\pm$.0015} & 0.8528{\scriptsize$\pm$.0031} & 0.1544{\scriptsize$\pm$.0102} \\
SFT+GRPO            & 0.6696{\scriptsize$\pm$.0032} & 0.7750{\scriptsize$\pm$.0027} & 0.8522{\scriptsize$\pm$.0071} & 0.0822{\scriptsize$\pm$.0035} \\
RLTR                & 0.6610{\scriptsize$\pm$.0065} & 0.7197{\scriptsize$\pm$.0078} & 0.8919{\scriptsize$\pm$.0053} & 0.0861{\scriptsize$\pm$.0102} \\
ToolPO              & 0.6241{\scriptsize$\pm$.0090} & 0.6827{\scriptsize$\pm$.0096} & 0.8765{\scriptsize$\pm$.0104} & 0.0083{\scriptsize$\pm$.0067} \\
\rowcolor{oursTint}\textbf{\textsc{Ours}} & \textbf{0.7031}{\scriptsize$\pm$.0050} & \textbf{0.7856}{\scriptsize$\pm$.0082} & \textbf{0.8954}{\scriptsize$\pm$.0083} & 0.1489{\scriptsize$\pm$.0025} \\

\bottomrule
\end{tabular}
\end{table}

\subsection{Additional Analysis on \texorpdfstring{$\tau^2$}{tau2}-Bench Results}
\label{app:tau2_details}

\paragraph{Baseline Configurations and Omissions}
The Qwen3-8B-Base \textit{Original} baseline has a 0.00 Pass$^1$ rate in every domain; it is shown explicitly in \Cref{tab:tau2_full}. The SFT baseline serves as the starting point for the 8B scale comparisons.

\begin{table}[H]
\centering
\small
\caption{
Comprehensive Collaboration Results on $\tau^2$-Bench across Scales.
We report Pass$^1$ rates for dual-control scenarios across three distinct domains: Airline (User Retention), Retail (Order Modification), and Telecom (Plan Management). Pass$^1$ is the proportion of successful tasks in the finite official task set; because each run evaluates a finite set, the reported values are discrete proportions and their standard deviations reflect variation of those proportions across runs. The Qwen3-8B-Base Original row is included and is zero in all four reported metrics. Trained rows are mean$\pm$std over three runs. ToolPO and RLTR retain their method-specific protocols.
}
\label{tab:tau2_full}
\setlength{\tabcolsep}{5pt}
\begin{tabular}{lcccc}
\toprule
\rowcolor{ArxivTableHead}
\textbf{Method} & \textbf{Overall Pass$^1$} & \textbf{Airline} & \textbf{Retail} & \textbf{Telecom} \\
\midrule

\rowcolor{scaleThreeTint}\textit{\textbf{Qwen2.5-3B-Instruct}} &  &  &  &  \\
\midrule
Original            & 0.2324 & 0.3810 & 0.1481 & 0.2692 \\
SFT                 & 0.2705{\scriptsize$\pm$.0236} & 0.1364{\scriptsize$\pm$.0227} & 0.3966{\scriptsize$\pm$.0259} & 0.1954{\scriptsize$\pm$.0217} \\
SFT+GRPO            & 0.3454{\scriptsize$\pm$.0257} & 0.2576{\scriptsize$\pm$.0131} & 0.4598{\scriptsize$\pm$.0217} & 0.2644{\scriptsize$\pm$.0348} \\
RLTR                & 0.2850{\scriptsize$\pm$.0291} & 0.2727{\scriptsize$\pm$.0227} & 0.3649{\scriptsize$\pm$.0303} & 0.2098{\scriptsize$\pm$.0303} \\
ToolPO              & 0.2355{\scriptsize$\pm$.0226} & 0.1288{\scriptsize$\pm$.0262} & 0.3678{\scriptsize$\pm$.0263} & 0.1437{\scriptsize$\pm$.0179} \\
\rowcolor{oursTint}\textbf{\textsc{Ours}} & \textbf{0.3563}{\scriptsize$\pm$.0091} & \textbf{0.2576}{\scriptsize$\pm$.0347} & 0.4569{\scriptsize$\pm$.0172} & \textbf{0.2931}{\scriptsize$\pm$.0172} \\

\midrule

\rowcolor{scaleSevenTint}\textit{\textbf{Qwen2.5-7B-Instruct}} &  &  &  &  \\
\midrule
Original            & 0.3023 & 0.3023 & 0.4054 & 0.2018 \\
SFT                 & 0.3254{\scriptsize$\pm$.0192} & 0.2381{\scriptsize$\pm$.0227} & 0.4167{\scriptsize$\pm$.0253} & 0.2647{\scriptsize$\pm$.0208} \\
SFT+GRPO            & 0.3187{\scriptsize$\pm$.0297} & 0.3375{\scriptsize$\pm$.0177} & 0.4221{\scriptsize$\pm$.0260} & 0.2124{\scriptsize$\pm$.0375} \\
RLTR                & 0.3372{\scriptsize$\pm$.0238} & 0.2195{\scriptsize$\pm$.0271} & 0.4001{\scriptsize$\pm$.0302} & 0.3182{\scriptsize$\pm$.0264} \\
ToolPO              & 0.3044{\scriptsize$\pm$.0215} & 0.2000{\scriptsize$\pm$.0243} & 0.4286{\scriptsize$\pm$.0278} & 0.2208{\scriptsize$\pm$.0231} \\
\rowcolor{oursTint}\textbf{\textsc{Ours}} & \textbf{0.4102}{\scriptsize$\pm$.0101} & 0.2827{\scriptsize$\pm$.0380} & \textbf{0.5323}{\scriptsize$\pm$.0203} & \textbf{0.3380}{\scriptsize$\pm$.0190} \\

\midrule

\rowcolor{scaleEightTint}\textit{\textbf{Qwen3-8B-Base}} &  &  &  &  \\
\midrule
Original            & 0.0000 & 0.0000 & 0.0000 & 0.0000 \\
SFT                 & 0.3104{\scriptsize$\pm$.0200} & 0.2273{\scriptsize$\pm$.0227} & 0.4540{\scriptsize$\pm$.0303} & 0.1983{\scriptsize$\pm$.0259} \\
SFT+GRPO            & 0.3466{\scriptsize$\pm$.0236} & 0.2652{\scriptsize$\pm$.0131} & 0.4425{\scriptsize$\pm$.0303} & 0.2816{\scriptsize$\pm$.0303} \\
RLTR                & 0.3490{\scriptsize$\pm$.0218} & 0.2273{\scriptsize$\pm$.0227} & 0.4282{\scriptsize$\pm$.0217} & 0.3161{\scriptsize$\pm$.0217} \\
ToolPO              & 0.4432{\scriptsize$\pm$.0186} & 0.2197{\scriptsize$\pm$.0262} & 0.5891{\scriptsize$\pm$.0303} & 0.3822{\scriptsize$\pm$.0217} \\
\rowcolor{oursTint}\textbf{\textsc{Ours}} & \textbf{0.4469}{\scriptsize$\pm$.0084} & \textbf{0.3561}{\scriptsize$\pm$.0473} & 0.5374{\scriptsize$\pm$.0217} & \textbf{0.3908}{\scriptsize$\pm$.0217} \\

\bottomrule
\end{tabular}
\end{table}

\paragraph{Domain-Specific Performance Analysis}
The \slcaname mean is higher than the matched GRPO mean on Overall Pass$^1$ for all three backbones. Domain-level orderings vary by scale: on Qwen2.5-3B-Instruct, \slcaname ties GRPO on Airline, trails it on Retail, and leads on Telecom; RLTR has the highest Airline mean. On Qwen2.5-7B-Instruct, the Airline mean is below GRPO while Retail and Telecom are higher. On Qwen3-8B-Base, \slcaname leads on Airline and Telecom, while ToolPO has the highest Retail mean.

The inclusion of RLTR and ToolPO reveals additional patterns.
RLTR trails SLCA on 7B (0.3372 vs.\ 0.4102) and 3B (0.2850 vs.\
0.3563), constrained by its frozen summarizer in tasks requiring
full conversational agent loops.
On Qwen3-8B-Base, the reported ToolPO and \slcaname Overall means
are 44.32$\pm$1.86\% and 44.69$\pm$0.84\%, a difference of
0.37\,pp. ToolPO has the higher Retail mean (58.91$\pm$3.03\% vs.\
53.74$\pm$2.17\%), while \slcaname is higher on Airline and Telecom.
These results show that the relative ordering depends on the domain and
protocol.

\subsection{HierR Weight Sensitivity Analysis}
\label{app:hierr_sensitivity}

To examine sensitivity to HierR weight choices, we evaluate three $S_{\text{process}}$ weight configurations on Qwen2.5-7B-Instruct (three-run mean$\pm$std), summarized in Table~\ref{tab:hierr_sensitivity}.

\begin{table}[H]
\centering
\small
\caption{HierR weight sensitivity on Qwen2.5-7B-Instruct. Entries are mean$\pm$std over three runs.}
\label{tab:hierr_sensitivity}
\setlength{\tabcolsep}{3pt}
\begin{tabular}{lccccccc}
\toprule
\rowcolor{ArxivTableHead}
& \multicolumn{5}{c}{\textbf{$S_{\text{process}}$ Weights}} & & \\
\cmidrule{2-6}
\textbf{Config} & $S_{\text{fmt}}$ & $S_{\text{name}}$ & $S_{\text{key}}$ & $S_{\text{val}}$ & $S_{\text{par}}$ & \textbf{BFCL Acc (\%)} & \textbf{$\tau^2$ Pass$^1$ (\%)} \\
\midrule
\rowcolor{oursTint}\textbf{Default (SLCA)} & 0.10 & 0.25 & 0.15 & 0.20 & 0.30 & \textbf{69.77$\pm$0.47} & \textbf{41.02$\pm$1.01} \\
Uniform (SLCA) & 0.20 & 0.20 & 0.20 & 0.20 & 0.20 & 69.18$\pm$0.55 & 38.62$\pm$1.86 \\
Value-heavy (SLCA) & 0.05 & 0.20 & 0.20 & 0.35 & 0.20 & 69.44$\pm$0.48 & 39.35$\pm$1.73 \\
\midrule
w/o HierR & \multicolumn{5}{c}{---} & 69.07$\pm$0.46 & 34.50$\pm$1.44 \\
GRPO (unified adv.) & \multicolumn{5}{c}{same as Default} & 68.41$\pm$0.11 & 31.87$\pm$2.97 \\
\bottomrule
\end{tabular}
\end{table}

The Uniform configuration sets $S_{\text{par}}=0.20$ and remains above the matched GRPO baseline on BFCL and $\tau^2$-Bench (+0.77\,pp and +6.75\,pp, respectively). The Default configuration has the highest values among the tested settings.

\subsection{Unified Reward Ratio Sensitivity}
\label{app:unified_ratio_sweep}

We vary the relative weight of the two raw rewards in the unified GRPO estimator,
\begin{equation}
\hat A_i^{\mathrm{uni}}(\alpha)
=
\operatorname{Norm}\!\left(
\alpha R_i^{\mathrm{tool}}+R_i^{\mathrm{sum}}
\right),
\qquad
\alpha=\frac{\lambda_{\mathrm{tool}}}{\lambda_{\mathrm{sum}}},
\end{equation}
while keeping the SFT initialization, data split, SGLS, HierR components, rollout budget, and evaluation protocol fixed. The displayed ratios include the existing $1{:}1$ condition and the three additional tested ratios.

\begin{table}[H]
\centering
\small
\caption{Unified reward ratio sensitivity on Qwen2.5-7B-Instruct. Entries are mean $\pm$ standard deviation over three matched runs. The ratio is applied before the unified group normalization; the scale-invariance identity holds up to the numerical floor described in App.~\ref{app:slca_details}.}
\label{tab:unified_ratio_sweep}
\setlength{\tabcolsep}{4pt}
\begin{tabular}{lccc}
\toprule
\rowcolor{ArxivTableHead}
\textbf{Reward ratio} & \textbf{Toucan Success@0.9 (\%)} & \textbf{BFCL Acc. (\%)} & \textbf{$\tau^2$-Bench Pass$^1$ (\%)} \\
\midrule
Unified $1{:}1$ & 76.60 $\pm$ 1.27 & 68.41 $\pm$ 0.11 & 31.87 $\pm$ 2.97 \\
Unified $2{:}1$ & 77.54 $\pm$ 1.24 & 68.83 $\pm$ 0.31 & 34.62 $\pm$ 2.71 \\
\textbf{Unified $3{:}1$} & \textbf{77.88 $\pm$ 1.30} & \textbf{68.95 $\pm$ 0.36} & \textbf{35.24 $\pm$ 2.58} \\
Unified $5{:}1$ & 77.02 $\pm$ 1.41 & 68.60 $\pm$ 0.44 & 32.90 $\pm$ 2.88 \\
\midrule
\rowcolor{oursTint}\textbf{\slcaname} & \textbf{79.13 $\pm$ 1.05} & \textbf{69.77 $\pm$ 0.47} & \textbf{41.02 $\pm$ 1.01} \\
\bottomrule
\end{tabular}
\end{table}

The unified reward ratio sweep recovers part of the SLCA gap within the tested ratios. The best tested unified ratio remains below SLCA on BFCL and $\tau^2$-Bench, while increasing the tool weight to $5{:}1$ reduces the $\tau^2$-Bench score relative to the best tested ratio. The pattern indicates a trade-off between tool and summary rewards in the unified estimator rather than a single uniformly favorable weight.

\subsection{Support Control on Qwen2.5-7B-Instruct}
\label{app:support_control_7b}

We evaluate four support conditions at the primary 7B scale. Let
$T=\lambda_{\mathrm{tool}}\operatorname{Norm}(R_{\mathrm{tool}})$ and
$S=\lambda_{\mathrm{sum}}\operatorname{Norm}(R_{\mathrm{sum}})$, with both weights set to one. The four cells differ only in which normalized components reach each token segment; additive sums are left unrescaled, as in the tested estimator. All cells use the same initialization, data, rollout budget, and evaluation protocol. Unified is shown as a joint-normalization reference and is not part of the factorial effects.

\begin{table}[H]
\centering
\scriptsize
\caption{Support control on Qwen2.5-7B-Instruct. The four support cells use separately normalized components. Unified is a joint-normalization reference. Entries are mean $\pm$ standard deviation over three matched runs.}
\label{tab:support_control_7b}
\setlength{\tabcolsep}{1.5pt}
\begin{tabular}{lccccc}
\toprule
\rowcolor{ArxivTableHead}
\textbf{Condition} & \textbf{Tool tokens} & \textbf{Summary tokens} & \textbf{Toucan (\%)} & \textbf{BFCL (\%)} & \textbf{$\tau^2$ (\%)} \\
\midrule
\rowcolor{oursTint}\textbf{Both closed (SLCA)} & $T$ & $S$ & 79.13 $\pm$ 1.05 & 69.77 $\pm$ 0.47 & 41.02 $\pm$ 1.01 \\
$S\mathbin{\to}T$ open & $T+S$ & $S$ & 77.21 $\pm$ 1.24 & 68.72 $\pm$ 0.29 & 33.95 $\pm$ 2.71 \\
$T\mathbin{\to}S$ open & $T$ & $T+S$ & 79.02 $\pm$ 1.12 & 69.84 $\pm$ 0.44 & 41.36 $\pm$ 1.24 \\
Both open & $T+S$ & $T+S$ & 78.30 $\pm$ 1.09 & 69.30 $\pm$ 0.41 & 37.42 $\pm$ 2.10 \\
\midrule
Unified (joint normalization) & joint normalized reward & joint normalized reward & 76.60 $\pm$ 1.27 & 68.41 $\pm$ 0.11 & 31.87 $\pm$ 2.97 \\
\bottomrule
\end{tabular}
\end{table}

The factorial effects are computed from the four support cells rather than from a diagonal contrast:
\begin{table}[H]
\centering
\small
\caption{Factor effects from the 7B support control. Values are computed from the cell means in Table~\ref{tab:support_control_7b} and are reported in percentage points. Main effects average the two simple effects over the other factor; interaction entries use the unhalved difference-in-differences contrast.}
\label{tab:support_factor_effects}
\setlength{\tabcolsep}{5pt}
\begin{tabular}{lccc}
\toprule
\rowcolor{ArxivTableHead}
\textbf{Effect} & \textbf{Toucan} & \textbf{BFCL} & \textbf{$\tau^2$-Bench} \\
\midrule
Close $S\mathbin{\to}T$ & 1.32 & 0.80 & 5.51 \\
Open $T\mathbin{\to}S$ & 0.49 & 0.33 & 1.91 \\
Interaction & 1.20 & 0.51 & 3.13 \\
\bottomrule
\end{tabular}
\end{table}

Closing the $S\mathbin{\to}T$ path gives the larger estimated effect on all three benchmarks in this control. Opening the $T\mathbin{\to}S$ path has a smaller effect, and the difference between that condition and SLCA is $-0.11/+0.07/+0.34$ pp on Toucan, BFCL, and $\tau^2$-Bench. These results are support control effects under the tested additive estimator; they are not presented as a routing-only causal estimate.

\subsection{Full Ablation Studies Across Scales}
\label{app:full_ablation}

\paragraph{Component contributions across scales.}
Table~\ref{tab:full_ablation} reports ablation results across the three backbones.
The w/o SLCA rows have lower reported Toucan Success means by 2.05 to 2.53\,pp across the three backbones. The matched BFCL and $\tau^2$-Bench means are also lower in the reported blocks; the 3B BFCL gap is 0.40\,pp.
The w/o HierR rows show the largest Toucan Success drops in each block and lower means on most other Toucan metrics. Removing deterministic schema validation and schema-conditioned tool-response prompting has a smaller effect on in-domain Success; it also changes BFCL and $\tau^2$-Bench, with the largest BFCL drop at 8B.

\begin{table}[ht]
\centering
\small
\caption{Comprehensive Ablation Studies Across Model Families. Comparison of Full \slcaname against variants removing key components. Metrics include Toucan-Test (Name F1, ArgMatch, Process, Success), BFCL (Accuracy), and $\tau^2$-Bench (Pass$^1$). Process is computed as a post-hoc evaluation metric for all settings, including w/o HierR. Trained rows in all three blocks report mean$\pm$std over three runs; Original evaluations are point values where shown. The w/o SLCA condition is the SFT+GRPO estimator.}
\label{tab:full_ablation}
\setlength{\tabcolsep}{3pt}
\begin{tabular}{lcccccc}
\toprule
\rowcolor{ArxivTableHead}
\multirow{2}{*}{\textbf{Method}} & \multicolumn{4}{c}{\textbf{Toucan-Test}} & \textbf{BFCL} & \textbf{$\tau^2$-Bench} \\
\cmidrule{2-7}
 & \textbf{Name F1} & \textbf{ArgMatch} & \textbf{Process} & \textbf{Success} & \textbf{Acc} & \textbf{Pass$^1$} \\
\midrule
\rowcolor{scaleThreeTint}\textit{\textbf{Qwen2.5-3B-Instruct}} &  &  &  &  &  &  \\
\midrule
\rowcolor{oursTint}\textbf{\textsc{\slcaname}} & \textbf{0.9006}{\scriptsize$\pm$.0087} & \textbf{0.8092}{\scriptsize$\pm$.0112} & \textbf{0.8625}{\scriptsize$\pm$.0049} & \textbf{0.7647}{\scriptsize$\pm$.0093} & \textbf{0.6361}{\scriptsize$\pm$.0052} & \textbf{0.3563}{\scriptsize$\pm$.0091} \\
w/o SLCA & 0.8910{\scriptsize$\pm$.0087} & 0.8072{\scriptsize$\pm$.0154} & 0.8577{\scriptsize$\pm$.0057} & 0.7412{\scriptsize$\pm$.0115} & 0.6321{\scriptsize$\pm$.0032} & 0.3454{\scriptsize$\pm$.0257} \\
w/o SGLS & 0.8927{\scriptsize$\pm$.0075} & 0.7982{\scriptsize$\pm$.0109} & 0.8559{\scriptsize$\pm$.0050} & 0.7584{\scriptsize$\pm$.0101} & 0.6304{\scriptsize$\pm$.0048} & 0.2995{\scriptsize$\pm$.0182} \\
w/o HierR & 0.8399{\scriptsize$\pm$.0084} & 0.7850{\scriptsize$\pm$.0161} & 0.8264{\scriptsize$\pm$.0075} & 0.6552{\scriptsize$\pm$.0107} & 0.5993{\scriptsize$\pm$.0046} & 0.3152{\scriptsize$\pm$.0145} \\
\midrule
\rowcolor{scaleSevenTint}\textit{\textbf{Qwen2.5-7B-Instruct}} (three-run mean$\pm$std) &  &  &  &  &  &  \\
\midrule
\rowcolor{oursTint}\textbf{\textsc{\slcaname}} & \textbf{0.9164}{\scriptsize$\pm$.0075} & \textbf{0.8353}{\scriptsize$\pm$.0132} & 0.8766{\scriptsize$\pm$.0045} & \textbf{0.7913}{\scriptsize$\pm$.0105} & \textbf{0.6977}{\scriptsize$\pm$.0047} & \textbf{0.4102}{\scriptsize$\pm$.0101} \\
w/o SLCA & 0.8971{\scriptsize$\pm$.0093} & 0.8350{\scriptsize$\pm$.0148} & 0.8667{\scriptsize$\pm$.0061} & 0.7660{\scriptsize$\pm$.0127} & 0.6841{\scriptsize$\pm$.0011} & 0.3187{\scriptsize$\pm$.0297} \\
w/o SGLS & 0.9163{\scriptsize$\pm$.0081} & 0.8298{\scriptsize$\pm$.0126} & 0.8787{\scriptsize$\pm$.0051} & 0.7802{\scriptsize$\pm$.0112} & 0.6875{\scriptsize$\pm$.0048} & 0.3619{\scriptsize$\pm$.0195} \\
w/o HierR & 0.9167{\scriptsize$\pm$.0089} & 0.8284{\scriptsize$\pm$.0138} & 0.8521{\scriptsize$\pm$.0083} & 0.7344{\scriptsize$\pm$.0121} & 0.6907{\scriptsize$\pm$.0046} & 0.3450{\scriptsize$\pm$.0144} \\
\midrule
\rowcolor{scaleEightTint}\textit{\textbf{Qwen3-8B-Base}} &  &  &  &  &  &  \\
\midrule
\rowcolor{oursTint}\textbf{\textsc{\slcaname}} & \textbf{0.9177}{\scriptsize$\pm$.0075} & 0.8320{\scriptsize$\pm$.0140} & \textbf{0.8786}{\scriptsize$\pm$.0046}  & \textbf{0.7902}{\scriptsize$\pm$.0105} & \textbf{0.7031}{\scriptsize$\pm$.0050} & \textbf{0.4469}{\scriptsize$\pm$.0084} \\
w/o SLCA & 0.9110{\scriptsize$\pm$.0092} & 0.8334{\scriptsize$\pm$.0135} & 0.8767{\scriptsize$\pm$.0050}   & 0.7697{\scriptsize$\pm$.0141} & 0.6696{\scriptsize$\pm$.0032} & 0.3466{\scriptsize$\pm$.0236} \\
w/o SGLS & 0.9123{\scriptsize$\pm$.0084} & 0.8276{\scriptsize$\pm$.0110} & 0.8751{\scriptsize$\pm$.0058} & 0.7832{\scriptsize$\pm$.0127} & 0.6704{\scriptsize$\pm$.0047} & 0.3853{\scriptsize$\pm$.0146} \\
w/o HierR & 0.8985{\scriptsize$\pm$.0105} & 0.8215{\scriptsize$\pm$.0121} & 0.8648{\scriptsize$\pm$.0075} & 0.7496{\scriptsize$\pm$.0139} & 0.6603{\scriptsize$\pm$.0045} & 0.3647{\scriptsize$\pm$.0146} \\
\bottomrule
\end{tabular}
\end{table}

\subsection{Scaling to Large Tool Spaces Across Families}
\label{app:full_toolspace}

Table~\ref{tab:full_toolspace} extends the scalability analysis across all three backbones, including an extreme stress test at $K=150$. The table compares the two sampling strategies at each candidate-set size:
\begin{itemize}
    \item \textbf{Semantic interference:} Across all models, performance drops under Hard-Negative sampling compared to Random sampling (e.g., Qwen3-8B drops from $\sim$60\% to $\sim$36\% at $K=100$). The stress test makes semantically similar tools harder to distinguish than random candidates at the same set size.
    \item \textbf{Effect of context limits:} At $K=150$ for Qwen2.5-7B, the Random-sampling rows are near-tied, while the Hard-sampling gap is +2.50 pp. These results show how candidate-set construction and context length interact with the estimator.
    \item \textbf{Across tested families:} \slcaname is higher than GRPO in the reported Hard-sampling rows (e.g., +2.25 pp for 3B at $K=150$).
\end{itemize}

\begin{table}[H]
\centering
\small
\caption{
Comprehensive Scalability Results Across Model Families.
We report \textbf{Success$@$0.9} as the candidate set size $K$ increases.
Comparison is made between Hard-Negative (Semantic) and Random sampling to isolate the impact of semantic interference. The \textit{Original} column reports the multi-run mean at $K=0$, matching the main results; the $K>0$ columns retain the stress-test evaluations.
For $K\ge 100$, especially $K=150$, these stress-test values are computed on context-valid survivors after skip filtering.
}
\label{tab:full_toolspace}
\setlength{\tabcolsep}{4pt}
\begin{tabular}{lccccc}
\toprule
\rowcolor{ArxivTableHead}
& \multicolumn{5}{c}{\textbf{Candidate Tool Set Size ($K$)}} \\
\cmidrule{2-6}
\textbf{Method} & \textbf{Original} & \textbf{20} & \textbf{50} & \textbf{100} & \textbf{150} \\
\midrule

\multicolumn{6}{l}{\textbf{Hard-Negative Sampling} \hfill \textit{\textbf{Qwen2.5-3B-Instruct}}} \\
\midrule
SFT+GRPO & 0.7412 & 0.3400 & 0.2921 & 0.2662 & 0.2285 \\
\rowcolor{oursTint}\textbf{\textsc{Ours}} & \textbf{0.7647} & \textbf{0.3584} & \textbf{0.3209} & \textbf{0.2865} & \textbf{0.2510} \\
\midrule
\addlinespace[2pt]
\multicolumn{6}{l}{\textbf{Random Sampling} \hfill \textit{\textbf{Qwen2.5-3B-Instruct}}} \\
\midrule
SFT+GRPO & 0.7412 & 0.6399 & 0.5701 & 0.4950 & 0.4360 \\
\rowcolor{oursTint}\textbf{\textsc{Ours}} & \textbf{0.7647} & \textbf{0.6637} & \textbf{0.5854} & \textbf{0.5050} & \textbf{0.4413} \\

\midrule

\multicolumn{6}{l}{\textbf{Hard-Negative Sampling} \hfill \textit{\textbf{Qwen2.5-7B-Instruct}}} \\
\midrule
SFT+GRPO & 0.7660 & 0.3650 & 0.3325 & 0.3111 & 0.2570 \\
\rowcolor{oursTint}\textbf{\textsc{Ours}} & \textbf{0.7913} & \textbf{0.3862} & \textbf{0.3452} & \textbf{0.3304} & \textbf{0.2820} \\
\midrule
\addlinespace[2pt]
\multicolumn{6}{l}{\textbf{Random Sampling} \hfill \textit{\textbf{Qwen2.5-7B-Instruct}}} \\
\midrule
SFT+GRPO & 0.7660 & 0.6749 & 0.6060 & 0.5477 & 0.5155 \\
\rowcolor{oursTint}\textbf{\textsc{Ours}} & \textbf{0.7913} & \textbf{0.7030} & \textbf{0.6377} & \textbf{0.5776} & 0.5129 \\
\midrule

\multicolumn{6}{l}{\textbf{Hard-Negative Sampling} \hfill \textit{\textbf{Qwen3-8B-Base}}} \\
\midrule
SFT+GRPO & 0.7697 & 0.3976 & 0.3615 & 0.3508 & 0.3080 \\
\rowcolor{oursTint}\textbf{\textsc{Ours}} & \textbf{0.7902} & \textbf{0.4143} & \textbf{0.3662} & \textbf{0.3601} & \textbf{0.3135} \\
\midrule
\addlinespace[2pt]
\multicolumn{6}{l}{\textbf{Random Sampling} \hfill \textit{\textbf{Qwen3-8B-Base}}} \\
\midrule
SFT+GRPO & 0.7697 & 0.7044 & 0.6468 & 0.5775 & 0.5729 \\
\rowcolor{oursTint}\textbf{\textsc{Ours}} & \textbf{0.7902} & \textbf{0.7312} & \textbf{0.6753} & \textbf{0.6029} & \textbf{0.5844} \\

\bottomrule
\end{tabular}
\end{table}

\subsection{Execution-Based Reward Experiment}
\label{app:exec_reward}

To test whether SLCA's gains extend beyond gold-trajectory matching rewards, we replace $\toolsub{R} = S_{\text{process}}$ with execution-based $R_{\text{succ}}$. $\summarysub{R} = S_{\summaryword}$ remains unchanged.
Both methods use Qwen2.5-7B-Instruct with identical data, SGLS, and compute.

For this test, the binary $R_{\text{succ}}$ judge labels each rollout as Solved or Unsolved without per-example gold-call matching. Table~\ref{tab:exec_reward_main} reports the matched comparison.

\begin{table}[H]
\centering
\small
\caption{Execution-based $R_{\text{succ}}$ comparison on Qwen2.5-7B-Instruct. Entries are mean $\pm$ standard deviation over three runs; data, initialization, SGLS, and compute are matched.}
\label{tab:exec_reward_main}
\setlength{\tabcolsep}{4pt}
\begin{tabular}{lccc}
\toprule
\rowcolor{ArxivTableHead}
\textbf{Method} & \textbf{Toucan Succ. (\%)} & \textbf{BFCL Acc. (\%)} & \textbf{$\tau^2$-Bench Pass$^1$ (\%)} \\
\midrule
SFT+GRPO (w/o SLCA) & 73.51 $\pm$ 1.45 & 66.28 $\pm$ 0.32 & 25.54 $\pm$ 2.13 \\
\rowcolor{oursTint}\textbf{\slcaname} & \textbf{76.43 $\pm$ 1.21} & \textbf{67.61 $\pm$ 0.51} & \textbf{30.18 $\pm$ 1.76} \\
$\Delta$ & \textbf{+2.92} & \textbf{+1.33} & \textbf{+4.64} \\
\bottomrule
\end{tabular}
\end{table}

The SLCA advantage persists under the execution-based reward. Absolute scores are lower than under dense HierR rewards because the execution signal is sparser.

\paragraph{Training Dynamics under $R_{\text{succ}}$}

\begin{figure}[H]
    \centering
    \includegraphics[width=0.85\linewidth]{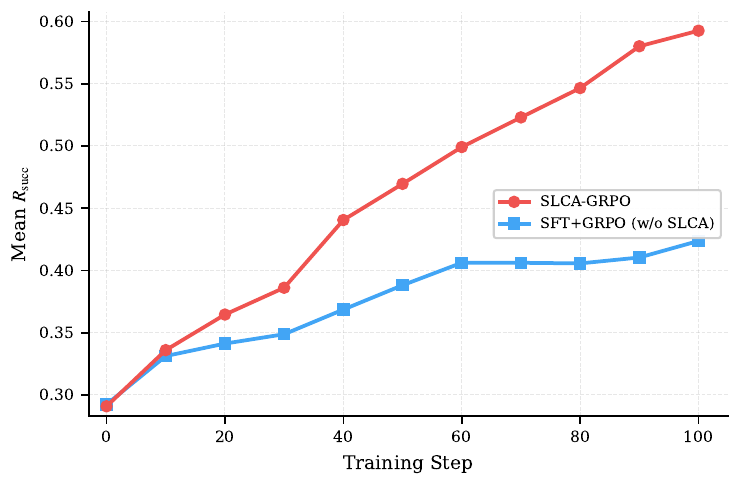}
    \caption{\textbf{Training dynamics under execution-based $R_{\text{succ}}$ (7B).}
    The curves are single-run checkpoint evaluations on the same fixed set of 3,991 valid Toucan validation examples. The plotted $R_{\text{succ}}$ is distinct from Toucan Success@0.9. At step~100, \slcaname reaches $R_{\text{succ}}{=}0.593$, compared with $0.424$ for SFT+GRPO (w/o SLCA).}
    \label{fig:exec_reward_dynamics}
\end{figure}

The $R_{\text{succ}}$ learning curve (Figure~\ref{fig:exec_reward_dynamics}) reveals a clear divergence: at step~0, both methods start at $R_{\text{succ}} \approx 0.29$; by step~100, \slcaname reaches 0.592583 (2365/3991), whereas the SFT+GRPO (w/o SLCA) baseline reaches 0.423703 (1691/3991), a +16.9\,pp gap. The logged validation series is \path{val-aux/toucan_eval_v4/reward/tool_call/r_succ/mean@1}; each checkpoint reports the exact empirical proportion on the same fixed set of 3,991 examples from one run per method, rather than an across-run mean. This execution-based metric is separate from Toucan's strict Success@0.9 measure.

\section{Prompts}
\label{app:prompt_details}
\subsection{Summary Judge Prompts}
\label{app:summary_prompt}
\begin{lstlisting}[style=promptstyle, label={lst:eval_prompt}]
## Task
Evaluate the quality of an AI assistant's final response based on the tool results it received.

## Assessment Criteria
### Response Quality (1-5 scale)
- 1 (very poor): Response is completely wrong, irrelevant, or ignores tool results
- 2 (poor): Seriously misinterprets tool results or misses critical information
- 3 (acceptable): Basically correct but has minor errors or incomplete coverage
- 4 (good): Correctly uses tool results, response is clear and complete
- 5 (excellent): Perfect interpretation of results, professional and helpful response

## User Question
```
{question_content}
```

## Tool Results
```
{tool_responses}
```

## Model's Final Response
```
{final_response}
```

## Evaluation Points
1. **Result Interpretation**: Does the model correctly understand and explain the data returned by tools?
2. **Information Completeness**: Does the response address all parts of the user's question based on the tool results?
3. **Expression Quality**: Is the response clear, well-organized, and helpful to the user?

## Output Format
<response>
  <response_quality>
    <reasoning>Brief analysis of the response quality based on the three evaluation points above.</reasoning>
    <rating><!-- one of: very poor, poor, acceptable, good, excellent --></rating>
  </response_quality>
</response>
\end{lstlisting}

\subsection{SGLS Simulator Prompts}
\label{app:sgls_prompt}
\begin{lstlisting}[style=promptstyle, label={lst:sgls_prompt}]
# CRITICAL ROLE: You are an API SERVER, NOT an AI assistant or agent.

You are role-playing as a **backend API server** that executes function calls and returns results.
You are NOT making tool calls - you are RESPONDING to them as if you were the actual API endpoint.

## YOUR ONLY OUTPUT FORMAT:
<tool_response>{"key": "value", ...}</tool_response>

## ABSOLUTE PROHIBITIONS:
- NEVER output <tool_call> tags (you are the SERVER, not the client)
- NEVER output markdown code blocks
- NEVER output explanations or conversational text
- ONLY output the <tool_response> tag with JSON inside

## Examples of CORRECT server responses:

Input: calculator(expression="124 * 5")
<tool_response>{"result": 620}</tool_response>

Input: get_weather(city="Seattle")
<tool_response>{"temp": "12$^\circ$C", "condition": "Rainy"}</tool_response>

Input: google_search(query="capital of Australia")
<tool_response>{"snippets": ["Canberra is the capital city of Australia."]}</tool_response>

---
## NOW EXECUTE THIS API CALL:

Tool Definition:
%s

Incoming Request:
%s

Return the execution result in <tool_response>...</tool_response> format:
\end{lstlisting}


\end{document}